\documentclass[11pt]{amsart}
\usepackage[margin=1in]{geometry}
\usepackage{graphicx}
\usepackage{bm}
\usepackage{multirow}
\usepackage{mathtools}
\usepackage{amsmath,amssymb,amsfonts}
\usepackage{algorithm, algpseudocode}
\usepackage{xcolor}
\usepackage{subfigure}
\usepackage{url}
\usepackage{tikz}
\usepackage{stanli}
\usetikzlibrary{decorations.pathreplacing}

\usepackage[normalem]{ulem}

\definecolor{exact}{RGB}{166, 97, 26}
\definecolor{no}{RGB}{223, 194, 125}
\definecolor{old}{RGB}{128, 205, 193}
\definecolor{new}{RGB}{1, 133, 113}

\usepackage{multirow}
\usepackage{pdflscape}
\usepackage{siunitx}
\usepackage{adjustbox}
\newtheorem{theorem}{Theorem}[section]

\newtheorem{proposition}{Proposition}[section]

\theoremstyle{definition}

\newtheorem{remark}[theorem]{Remark}
\newtheorem{assumption}[theorem]{Assumption}

\newcommand{\RN}[1]{%
	\textup{\uppercase\expandafter{\romannumeral#1}}%
}
\DeclareMathOperator*{\argmin}{\arg\!\min}
\DeclareMathOperator{\vect}{vec}
\DeclareMathOperator{\baroplus}{\bar{\oplus}}
\DeclareMathOperator{\diag}{Diag}
\DeclareMathOperator{\sech}{sech}

\newcommand{\bb}[1]{\mathbf{#1}}
\newcommand{\xx}{\bb{x}}

\newcommand{\xt}{\tilde{\bb{x}}}
\newcommand{\xh}{\hat{\bb{x}}}

\newcommand{\mmu}{\bm{\mu}}
\newcommand{\vp}{\varphi}

\newcommand{\lr}[1]{\left(#1\right)}
\newcommand{\norm}[1]{\left\lVert#1\right\rVert}
\newcommand{\nn}[1]{\left|#1\right|}

\newcommand{\uu}{\bb{U}}
\newcommand{\ut}{\bb{U}^\intercal}

\newcommand{\pp}{\bb{p}}

\newcommand{\qq}{\bb{q}}
\newcommand{\XX}{\bb{X}}

\newcommand{\Xh}{\hat{\XX}}

\newcommand{\Aa}{\bb{A}}
\newcommand{\DD}{\bb{D}}
\newcommand{\Ah}{\hat{\bb{A}}}
\newcommand{\Bh}{\hat{\bb{B}}}
\newcommand{\Ch}{\hat{\bb{C}}}
\newcommand{\Dh}{\hat{\bb{D}}}

\newcommand{\JJ}{\bb{J}}
\newcommand{\Jh}{\hat{\bb{J}}}
\newcommand{\Ih}{\hat{\bb{I}}}
\newcommand{\LL}{\bb{L}}
\newcommand{\Lh}{\hat{\bb{L}}}

\newcommand{\Th}{\hat{\bb{T}}}

\newcommand{\vv}{\bb{v}}
\newcommand{\IP}[2]{\left\langle#1,#2\right\rangle}

\title{Canonical and Noncanonical Hamiltonian Operator Inference}
\author{Anthony Gruber$^{1,*}$}
\author{Irina Tezaur$^2$}
\thanks{$^*$Corresponding author: (Anthony Gruber)  adgrube@sandia.gov}

\newcommand{\rev}[1]{\textcolor{black}{#1}}

\begin{document}

\begin{abstract}
A method for the nonintrusive and structure-preserving model reduction of canonical and noncanonical Hamiltonian systems is presented.    Based on the idea of operator inference, this technique is provably convergent and reduces to a straightforward linear solve given snapshot data and gray-box knowledge of the system Hamiltonian.  Examples involving several hyperbolic partial differential equations show that the proposed method yields reduced models which, in addition to being accurate and stable with respect to the addition of basis modes, preserve conserved quantities well outside the range of their training data.  
% \IKTcomment{Does it make sense to call the solid mechanics problem advection-dominated?  There isn't really advection here, though the problem is forced by an initial excitation to the velocity.  I can ask Alejandro if one uses this sort of lingo (``advection-dominated") in solid mechanics.}

\vspace{0.5pc}

\end{abstract}

% Model reduction is important for reducing the degrees of freedom in semi-discretized problems.  It is important for model reduction to preserve the structure which governs stability

\maketitle

\section{Introduction}

In recent years, Digital Twins (DTs) have emerged as a new paradigm in the field of modeling and simulation.  A DT is a computational model of a physical asset, such as a component, system or process, that evolves continuously in real or near-real time, so as to persistently represent the ever-changing structure and behavior of the underlying physical asset.  In order for DTs to achieve their full potential as enablers of beyond-forward analyses such as optimal experimental design (OED), control and uncertainty quantification (UQ), it is essential that these computational models are: (1) capable of incorporating real-time data as it becomes available, (2) computationally efficient enough to provide predictions in real or near-real time, and (3) equipped with rigorous mathematical convergence, stability and accuracy guarantees.

Particularly helpful in establishing the above criteria is making appropriate use of well-studied mathematical structure 
inherent in the underlying partial differential equations (PDEs) when such structure is available. In the case that the system modeled obeys a variational principle, there are centuries of knowledge regarding dynamical properties (e.g., conservation laws) which can be leveraged to produce accurate and realistic simulations.  Of particular interest at the present time are Hamiltonian systems, which form compact models of reversible, potentially chaotic dynamics.  Since many common systems relevant to digital twins have a Hamiltonian form (e.g, H{\'e}non-Heiles, $n$-body motion, idealized MHD, solid dynamics), it is becoming increasingly necessary to have useful ways of building relatively cheap Hamiltonian surrogates which can be used to inform a high-quality digital representation.

Projection-based model order reduction (PMOR) is a  promising strategy for reducing the computational cost of high-fidelity simulations, making projection-based reduced order models (ROMs) ideal candidates for constructing DTs.  The key idea in PMOR is to learn a low-dimensional trial subspace by performing a data compression on a set of snapshots collected from a high-fidelity simulation or physical experiment, and to restrict the state variables to reside in this subspace.  This effectively projects the high-fidelity dynamics into a much smaller function space, which must be carefully imbued with sufficient information for accurate reconstruction of the high-fidelity solutions.  Traditionally, affine (or linear) approaches have been employed for constructing the low-dimensional trial subspace in which the ROM solution is sought, e.g., Proper Orthogonal Decomposition (POD) \cite{sirovich,holmes},
Dynamic Mode Decomposition (DMD) \cite{rowley1,Schmid_JFM_2010}\rev{,} balanced POD
(BPOD) \cite{rowley, willcox}, balanced truncation \cite{gugercin, moore}, and
the reduced basis method (RBM)  \cite{rozza, veroy}. 
While all such methods have their own strengths and weaknesses, without loss of generality, this work restricts attention to the POD approach for calculating reduced bases due to its prevalence, flexibility, and simplicity.  Beyond linear techniques, it is interesting to note that, in recent years, the idea of employing trial subspaces defined by nonlinear manifolds has started to be explored by a growing number of authors; see, e.g., \cite{lee2020model, Fresca:2022, Kim:2022, gruber2022a, Romor:2023, Sharma:2023, BarnettFarhatMaday:2023} and references therein for nonlinear manifold approaches based on convolutional autoencoders, and \cite{Barnett:2022, Geelen:2023} for quadratic manifold approaches.  Nonlinear approximation approaches have the advantage of mitigating the so-called Kolmogorov $n$-width barrier \cite{Pinkus:1985}, which reduces the efficacy and efficiency of linear manifold ROMs for convection-dominated problems\footnote{As discussed in Section \ref{sec:concl}, extending the approach proposed herein to nonlinear manifold bases will be the subject of future work.}.  However, they are often more difficult to train and can exhibit poor convergence behavior when compared with their linear counterparts \cite{gruber2022a,lee2020model}.
% \IKTcomment{I think the statement in the previous sentence is rather strong and could offend some people... Are you referring to poor convergence w.r.t. the basis size?   Do you have refs of the methods having poor convergence behavior?  If so, those should be added.   I tried to tone it down by inserting ``can" but refs would still be useful.}

Once a low-dimensional trial subspace has been constructed, the mathematical operators defining a ROM are obtained through a projection of the corresponding full order model (FOM) operators onto the reduced subspace.  Performing this projection step is in general a very intrusive process, as it requires access to the FOM code used to generate the snapshot data.  This intrusive nature of the projection step in PMOR limits the class of problems to which the approach can be applied, precluding the application of PMOR to FOMs that are given as a black-box.  A promising approach for overcoming this limitation is data-driven Operator Inference (OpInf) (e.g., \cite{Peherstorfer:2016, Ghattas:2021, Benner:2020}), which aims to construct projection-based ROMs in a nonintrusive way.  OpInf is motivated by the observation that projection preserves algebraic structure, that is, if the semi-discretized FOM has polynomial nonlinearities, a projection-based ROM for this system will also have polynomial nonlinearities of the same degree.  Once the functional, algebraic structure of the FOM (and hence the ROM) is determined, OpInf works by replacing the intrusive projection step that is typically used to determine the ROM operators with a least-squares problem that infers these operators directly in a black-box fashion using available snapshot data (c.f. Section~\ref{subsec:genopinf}).  
%In recent years, OpInf has been applied to a wide range of applications, including fluid mechanics and combustion \cite{McQuarrie:2021, Qian:2020, Swischuk:2020, Zastrow:2023}, additive manufacturing \cite{Khodabakhshi:2022}, magnetohydrodynamics (MHD) \cite{Issan:2023}, and solid mechanics \cite{Filanova:2022}. 

% \IKTcomment{Again, if I'm missing any important refs, please add.   I don't think it's necessary (or possible) to cite everything.  Note that I don't describe lifting - decided not to go there since we don't go there and this description is high level.  If you think it's important, I can add it.  Lastly: the paper \cite{Filanova:2022} is one I was not aware of that handles second order in time systems for solid mechanics - we should look at this for our planned future work.} 

\rev{It is well known} that \rev{projection-based} ROMs constructed using either intrusive or non-intrusive techniques will generally \rev{not automatically} inherit key mathematical properties of the PDEs from which they are derived.  Since these properties are often well-understood to be responsible for the involved physics, this is a major defect which can harm the predictive performance of ROMs, limiting their utility in practical cases of interest.  To remedy this difficulty, a variety of methodologies have been proposed which focus around preserving different mathematical structures often seen in application settings. Here, we summarize the literature on this subject for several common properties whose numerical preservation is critical to a wide range of applications: energy-/entropy- stability, conservation law preservation, and variational structure preservation (most notably involving Hamiltonian or Lagrangian structure, and including the focus of this paper).

It is worth noting that the majority of structure-preserving PMOR
approaches in the literature focus on intrusive ROMs rather than
non-intrusive OpInf models.  The present work is a step towards
filling this gap for the specific case of Hamiltonian systems.
\rev{In order to distinguish our approach from other related work,
we provide a succinct overview of existing OpInf methods below,
after our overview
of commonly-preserved structures/properties.}\\

%\IKTcomment{Comment to Anthony: the text from this point on is WIP in the introduction...  Please don't read / edit yet.  I will let you know when I'm done.} \\

\noindent {\bf Energy- and entropy-stability.}  The bulk of the literature on energy- and entropy-stability preserving PMOR approaches focuses on the specific case of compressible flow.  It is well-known that projection-based ROMs for compressible flow constructed via Galerkin projection in the $L^2$ inner product lack an \textit{a priori} stability guarantee \cite{BuiTanh:2007, Barone:2009, Kalashnikova:2014}.  This problem can be circumvented for traditional intrusive ROMs through a variable transformation or by changing the inner product in which the projection is done, yielding energy-stable \cite{Rowley:2004, Barone:2009, Kalashnikova:2014, Serre:2012} or entropy-stable \cite{Kalashnikova:2011, Chan:2020, Parish:2022} approaches. Alternate approaches for mitigating the problem which are less intrusive and possible to apply in conjunction with OpInf model reduction include subspace rotation \cite{Balajewicz:2016} and eigenvalue reassignment \cite{Kalashnikova2:2014, Rezaian:2021}.  
\rev{An interesting and very recent pre-print that considers incompressible flow is the work of Klein and Sanderse \cite{Klein:2023}, which develops a novel kinetic energy and momentum conserving hyper-reduction method for projection-based ROMs for the incompressible Navier-Stokes equations.} \\

\noindent {\bf Conservation law preservation.} A second problem arising in PMOR for fluid mechanics applications, and, more broadly, conservative systems of PDEs, is lack of conservation: ROMs constructed from conservative models are not guaranteed to maintain the underlying model's conservation laws. The following three references for mitigating this problem for intrusive projection-based ROMs are noteworthy.  In  \cite{Carlberg_SISC_2018}, Carlberg \textit{et al.} present a methodology for constructing conservative compressible flow ROMs by modifying the minimization problem defining the Least Squares Petrov-Galerkin (LSPG) \cite{CarlbergGappy} to include local or global conservation law constraints.  The formulation in \cite{Carlberg_SISC_2018} is extended to the case of incompressible flow in \cite{Rosenberger:2022}, yielding a method that is both mass- and kinetic energy-conserving, and thus nonlinearly stable.  An alternate way to create an incompressible flow ROM with mass and energy conservation is presented in \cite{Mohebujjaman_IJNMF_2018}.  In this work, the authors demonstrate that these properties can be attained at the ROM level through a careful selection of the boundary condition treatment and finite element space underlying the ROM.  The preservation of conservation laws in OpInf remains an open problem, although progress has been made on problems with a variational form through \cite{sharma2022hamiltonian, sharma2022preserving}, and this work. \\

\noindent {\bf Variational structure preservation.} Another mathematical property mentioned previously and exhibited by a wide range of physical systems (e.g., solid dynamics, the shallow water equations, etc.), including the ones considered in the present work, is variational structure.  This includes systems amenable to the standard Hamiltonian and Lagrangian formalisms, as well as the more general formalisms of, e.g., Euler-Poincar\'e, Lie-Poisson and metriplecticity.  The advantages of biasing toward this structure are clear to see; for example, since the Hamiltonian can be considered a representation of the energy of a system, a Hamiltonian structure-preserving discretization will automatically obey at least one conservation law.  Several Hamiltonian (or Lagrangian) structure-preserving approaches have been developed in recent years for the specific case of solid dynamics.  In \cite{Lall_PhysicaD_2003}, it is shown that performing a Galerkin projection of the second-order-in-time Euler-Lagrange equations defining a canonical solid dynamics problem preserves Lagrangian structure, provided no hyper-reduction is employed.  As discussed in \cite{Carlberg_SISC_2015}, traditional hyper-reduction approaches such as collocation, the Discrete Empirical Interpolation Method (DEIM) \cite{Chaturantabut:2009} and gappy POD \cite{Everson:1995} destroy the Lagrangian structure of the ROM.  In \cite{Carlberg_SISC_2015}, two Lagrangian structure-preserving approaches for performing hyper-reduction on these systems, termed reduced basis sparsification (RBS) and matrix gappy POD, are proposed.  Both approaches are of the ``approximate-then-project'' flavor, meaning they apply hyper-reduction to the nonlinear terms in the governing equations prior to projecting these terms onto a reduced basis.  An alternate ``project-then-approximate'' approach for preserving Lagrangian structure in ROMs for nonlinear solid dynamics applications is the Energy-Conserving Sampling and Weighting (ECSW) method of Farhat \textit{et al.} \cite{Farhat_IJNME_2015}.  In this method, the nonlinear projected function is approximated using a set of points and weights, the latter set of which are obtained by solving a non-negative least-squares optimization problem.  Interestingly, there has recently been some headway into OpInf techniques for PMOR on solid mechanical systems as well. In \cite{sharma2022preserving}, the authors develop a gray-box method for learning the linear parts of Lagrangian systems in a way that respects the symmetric positive definite nature of the governing operators.

A broader class of symplecticity-preserving PMOR methods focus on directly reducing the Hamiltonian first-order-in-time system \eqref{eq:hamilFOM}. 
As discussed in \cite{Lall_PhysicaD_2003}, performing a Galerkin projection of these equations onto a set of reduced basis vectors will generally not preserve the Hamiltonian/symplectic structure of the system.  Several works, e.g., \cite{Peng_JSC_2016, Hesthaven_ARXIV_2021}, propose to remedy this through Proper Symplectic Decomposition (PSD) and symplectic Galerkin projection.  In \cite{Peng_JSC_2016}, Peng \textit{et al.} propose three algorithms for calculating the PSD, based on the cotangent lift, complex SVD and nonlinear programming.  These algorithms effectively generate reduced bases %for %$z$ defined in \eqref{eq:z} 
such that projection onto the subspaces spanned by these bases will maintain symplecticity.  Further,  a version of DEIM for Hamiltonian systems, termed Symplectic DEIM (SDEIM), is developed for maintaining skew-symmetry (but not necessarily symplecticity or Hamiltonian structure) when performing hyper-reduction.  An approach based on a globally optimal symplectic reduced basis in the sense of the PSD is derived in \cite{Buchfink:2022}.  
% \IKTcomment{Again, we want to be careful not to offend potential reviewers.  I toned down the previous sentence by removing ``unfortunately".  
% I also have a question related to this: our approach also has some approximation in it, as discussed in Remark 3.3, no?
% }
Here, it is shown the POD of a canonizable Hamiltonian system is automatically symplectic, from which the authors deduce optimality of the PSD.  
% Other authors that consider PSD and SDEIM are Afkham, Hesthaven \textit{et al.} \cite{Afkham_SISC_2017, Afkham_SISC_2019, Hesthaven_ARXIV_2021}.
In \cite{Afkham_SISC_2017}, PSD is extended to create a greedy approach for symplectic basis generation.  The approach is advertised as more cost-effective than traditional POD and PSD, and exhibits exponentially-fast convergence.  
%Numerical results are presented on a parametric linear wave equation and a nonlinear Schrodinger equation.  
The follow-on work \cite{Afkham_SISC_2019} presents a reduced dissipative Hamiltonian (RDH) method as a structure-preserving model reduction approach for Hamiltonian systems with dissipation.  Unlike other approaches, the proposed approach enables the reduced system to be integrated using a symplectic integrator.  %The method is evaluated and compared to PSD on a dissipative heat equation, the Sine-Gordon equation and a port-Hamiltonian linear ladder network system.  
The recent work \cite{Hesthaven_ARXIV_2021}, based on a lot of the same ideas as \cite{Afkham_SISC_2017, Afkham_SISC_2019}, 
%The new contribution seems to be the demonstration 
demonstrates that linear symplectic maps can be used to guarantee that the reduced models inherit the geometric formulation from the full dynamics.  The approach evolves the approximating symplectic reduced space in time along a trajectory locally constrained on the tangent space of the high-dimensional dynamics.  
%Dissipative Hamiltonian systems are discussed in addition to canonical Hamiltonian systems. 
The recent pre-print \cite{pagliantini2022gradientpreserving}  presents 
 a different DEIM-based hyper-reduction method for nonlinear parametric dynamical systems
characterized by gradient fields such as Hamiltonian and port-Hamiltonian systems and gradient flows. 
%The
%gradient structure is associated with conservation of %invariants or with dissipation and hence
%plays a crucial role in the description of the physical %properties of the system. Traditional
%hyper-reduction of nonlinear gradient fields yields efficient %approximations that, however, lack
%the gradient structure. 
The authors %propose to first 
decompose
the nonlinear part of the Hamiltonian
%, mapped into a suitable reduced space, 
into a sum
of $d$ terms, each characterized by a sparse dependence on the system state, and obtain a hyper-
reduced approximation %is obtained 
of the Jacobian
by applying DEIM to the derived %$d$-valued nonlinear 
function. The resulting hyper-reduced model retains the gradient structure, and possesses \textit{a priori} error estimates showing that the hyper-reduced model converges to
the reduced model and the Hamiltonian is asymptotically preserved.  

It is also possible to derive Hamiltonian structure-preserving ROMs using the classical POD reduced basis.  In \cite{Gong_CMAME_2017}, a least-squares system is solved to ensure skew-symmetry of the POD-Galerkin system corresponding to the governing Hamiltonian form.  An \textit{a priori} error estimate for the resulting POD/Galerkin ROM is developed, but hyper-reduction is not considered, rendering the approach inefficient.  In \cite{Sockwell_THESIS_2019}, Sockwell presents a Hamiltonian structure-preserving approach that is most closely related to the approach in \cite{Gong_CMAME_2017} and that possesses similar error estimates; however, the technique in \cite{Sockwell_THESIS_2019} is derived in a Hilbert space and takes advantage of the Hamiltonian framework in order to abstract the technique to a wide variety of weighted inner-product spaces.  This method is shown to preserve linear Casimir invariants, and is demonstrated in the context of the rotating shallow water equations, commonly used in ocean modeling on the sphere.  \rev{In addition to intrusive PMOR approaches, there is recent work in developing non-intrusive OpInf PMOR approaches that preserve Lagrangian \cite{SharmaLagrangian:2022}  as well as Hamiltonian structure  \cite{sharma2022hamiltonian, Sharma:2023},} and ROMs with nonlinear manifold (e.g., convolutional autoencoder) bases \cite{Buchfink_ARXIV_2021}.  Notably, \cite{sharma2022hamiltonian, Sharma:2023} are the only works to the authors' knowledge in which a Hamiltonian structure-preserving nonintrusive OpInf PMOR methodology is developed, although this method is limited to canonical Hamiltonian systems with a block-diagonal gradient structure.

% In addition, there is some recent work in developing PMOR approaches that preserve Hamiltonian structure which are focused on ROMs with nonlinear manifold (e.g., convolutional autoencoder) bases \cite{Buchfink_ARXIV_2021}.    non-intrusive OpInf ROMs \cite{sharma2022hamiltonian}  To the best of our knowledge, at the present time, \cite{sharma2022hamiltonian} is the only work in which a Hamiltonian structure-preserving non-intrusive OpInf PMOR methodology is developed. The approach is limited to systems in the so-called canonical Hamiltonian form. \\

Beyond Hamiltonian systems, it is worth mentioning some current references focusing on structure-preserving model reduction for port-Hamiltonian and metriplectic systems, e.g., \cite{Chaturantabut_SISC_2016, Gruber_CMAME_2023}, which are extensions of the Hamiltonian formalism to systems with dissipation.  Metripletic dynamical systems separate dynamics into terms that are ``energy-preserving'' and ``dissipative'', represented by a noncanonical Poisson structure and a degenerate Riemannian metric structure, respectively.  %The two terms are reversible and %irreversible, respectively. 
To the best of our knowledge, \cite{Gruber_CMAME_2023} is the first paper to develop a structure-preserving (intrusive) ROM for PDEs with metriplectic structure.  Conversely, the work \cite{Chaturantabut_SISC_2016} presents three techniques for constructing reduced bases for port-Hamiltonian systems: one based on POD, one based on $\mathcal{H}_2/\mathcal{H}_{\infty}$-derived optimized bases (which can be calculated without any snapshots), and one that is a mixture of the two.  
%A version of DEIM that preserves port-Hamiltonian structure is also developed.   
Interestingly, the approach in \cite{Chaturantabut_SISC_2016} is based on Petrov-Galerkin projection, rather than Galerkin projection.  \\

\noindent {\bf Operator Inference.}  \rev{Data-driven operator
inference originated in the seminal work of Peherstorfer and Willcox
\cite{Peherstorfer:2016}, which demonstrates that reduced operators
in a projection-based ROM can be inferred non-intrusively (i.e.,
without access to the corresponding FOM operators or code) through
the numerical solution of an optimization problem, given a set of
FOM snapshots.  An acknowledged deficiency of the original OpInf
formulation is that it is only applicable to PDEs that are linear or
contain only low-order polynomial nonlinearities.  As demonstrated
in subsequent works \cite{Qian:2020, Khodabakhshi:2022,
KramerAIAA:2019}, this shortcoming can be circumvented for many
physical systems by using a technique known as ``lifting", which
defines a transformation of the state variables into auxiliary
variables that make the governing PDEs linear or quadratic. The
resulting approach, termed ``Lift and Learn" \cite{Qian:2020} has
been applied to a wide range of problems, including fluid mechanics
and combustion \cite{McQuarrie:2021, Qian:2020, Swischuk:2020,
Zastrow:2023}, additive manufacturing \cite{Khodabakhshi:2022},
magnetohydrodynamics (MHD) \cite{Issan:2023}, and solid mechanics
\cite{Filanova:2022}.}

\rev{During the past 1-2 years, researchers have begun to extend
operator inference in several important directions.  In
\cite{Benner:2020}, non-intrusive operator inference is extended to
problems with non-polynomial nonlinearities given in analytic form,
in a way that does require the definition of a lifting
transformation.  In several recent works, the group of Kramer
\textit{et al.} has developed OpInf methodologies that preserve
Hamiltonian (or symplectic) \cite{Sharma:2023, sharma2022preserving}
and Lagrangian \cite{SharmaLagrangian:2022} structure, to ensure
energy-conserving ROMs.  Note that, in the Hamiltonian case, all OpInf work to date has been restricted to purely canonical systems (c.f. Section~\ref{subsec:hamiltonian}.  A primary contribution of this work is the ability to treat both canonical and noncanonical systems of interest.} 

% \textcolor{red}{Comment from Irina for Anthony: you could consider adding here something about how what we're
% doing is different, e.g., being able to handle non-canonical
% forms.}

\rev{Other recent works have focused on improving the efficiency and
robustness of the optimization problem underlying OpInf.  It is
well-known that this optimization problem generally requires
regularization, and the results can be extremely sensitive to the
choice of regularization parameters.  Several researchers have begun to look at
ways to optimize the choice of these regularization parameters.  In \cite{Guo:2022}, Guo \textit{et al.} present a Bayesian
approach to operator inference, in which the maximum marginal
likelihood provides insight into the selection of the regularization
parameters specified in the OpInf minimization problem.  An alternate remedy known as nested operator inference is being pursued by Aretz \textit{et al.} \cite{Aretz:SIAMCSE2023}.}

\rev{While OpInf originated in the context of ROMs in which the
solution is approximated using an affine POD basis, the method has
recently been extended to balanced truncation \cite{KramerBT:2022}
and quadratic manifold bases \cite{KramerBT:2022, Geelen:2023}.  The
latter work \cite{KramerBT:2022} presents a symplecticity-preserving
method based on quadratic manifold bases.  The advantage of using quadratic manifold bases over linear bases is that it is often possible to represent the reduced solution using fewer basis vector, especially for problems exhibiting a slow decay of the Kolmogorov $n$-width \cite{Pinkus:1985}.  
} \\

\noindent \rev{{\bf Contributions of this manuscript.} The present work extends the literature on structure-preserving OpInf techniques to include the linear operators governing general canonical and noncanonical Hamiltonian systems.  Particularly, we contribute
\begin{itemize}
    \item Non-intrusive methods based on OpInf for learning either: (1) the linear part of the Hamiltonian gradient in the case of canonical Hamiltonian systems, or (2) the constant part of the Poisson matrix in the case of noncanonical Hamiltonian systems.  In contrast to previous work, these methods impose no restriction on the separability of the Hamiltonian or the algorithm used to compute the ROM basis.
    \item Theoretical analysis which guarantees that the learned operators converge to their intrusive counterparts in the limits of increasing basis size and increasing amounts of training data.
    \item Several numerical examples probing the behavior of these Hamiltonian OpInf ROMs in comparison to ROMs based black-box OpInf as well as more intrusive PMOR techniques.
\end{itemize}}

% \textcolor{red}{Comment from Irina: is it worth mentioning that our approach doesn't require regularization, as some of the other approaches, like Boris's?  That could be another contribution.}

The remainder of this paper is organized as follows.  Section~\ref{sec:pre} recalls preliminary information on Hamiltonian systems (including methods for their POD), as well as OpInf and average vector field time integration (c.f. \cite{celledoni2012preserving}).  Section~\ref{sec:hopinf} describes the present methods for canonical and noncanonical Hamiltonian OpInf, as well as their connection to previous work.  Section~\ref{sec:theory} provides analysis showing that the proposed OpInf methods converge to their intrusive counterparts with the addition of snapshot data and basis modes.  Finally, Section~\ref{sec:numerics} provides numerical evidence for the Hamiltonian OpInf approaches in Section~\ref{sec:hopinf} using \rev{five} example problems:  a linear wave equation, \rev{a nonseparable but canonical quadratic Hamiltonian system,} the Korteweg-De Vries (KdV) equation, the Benjamin-Bona-Mahoney (BBM) equation, and a 3D linear elastic cantilever plate.  Finally, some conclusions and future directions are discussed in Section~\ref{sec:concl}.

% \IKTcomment{Fill in...  this is typically how I finish my introductions, but your call if you want to use this format/structure.}

% \IKTcomment{Below is original text from Anthony...}

% Nonintrusive model reduction is critical for resource-intensive applications in which ... 

% \red{Irina: you are invited to start writing this section, either in full or with short notes which I can flesh out later.  Since you put together such a thorough review of related work, I think you will write that (sub)section much better than me.}  \IKTcomment{Done!  Let me know what you think...  feel free to truncate what I've written as needed, as I know it is a lot, perhaps too much...}

\section{Preliminaries}\label{sec:pre}

Here some preliminary information on Hamiltonian systems, as well as intrusive and nonintrusive methods of model reduction for such systems, is summarized.

\subsection{Hamiltonian Systems} \label{subsec:hamiltonian}

The Hamiltonian formalism provides a mechanical framework encompassing a wide variety of conservative dynamical systems which arise from a variational principle.  In particular, it reduces the problem of understanding a near-arbitrarily complicated dynamical system to the simpler problem of understanding a scalar-valued function $H$, called the Hamiltonian, and a skew-symmetric Poisson bracket $\{\cdot,\cdot\}$, which encodes a Lie algebra realization on functions.  More formally, given a state vector $\xx\in\mathbb{R}^N$, \rev{it follows that $\nabla\xx=\bb{I}$, and so} any Hamiltonian system can be written in the form 
\begin{equation}\label{eq:hamilFOM}
    \dot{\xx} = \{\xx,H(\xx)\} = \nabla \xx\cdot \LL(\xx)\nabla H(\xx) = \LL(\xx)\nabla H(\xx),
\end{equation}
for some $H:\mathbb{R}^N\to\mathbb{R}$ and some potentially degenerate Poisson matrix $\LL:\mathbb{R}^N\to\mathbb{R}^{N\times N}$, $\LL^\intercal = -\LL$ which is antisymmetric and satisfies the Jacobi identity,
\[\sum_{\ell=1}^N \lr{L_{il}L_{jk,l}+L_{jl}L_{ki,l}+L_{kl}L_{ij,l}} = 0, \qquad 1\leq i,j,k\leq N,\]
where $L_{ij}$ are the components of $\LL$ and $L_{ij,k}$ denotes the derivative with respect to the $k^{th}$ basis vector $\bb{e}_k\in\mathbb{R}^N$. From this, it is easy to see that the Poisson bracket (generated by $\LL$) is also skew-symmetric, bilinear, and obeys a Leibniz rule.  Moreover, the Hamiltonian $H$ is a conserved quantity, since $\{H,H\}=0$ by antisymmetry.

In the simplest case, Hamiltonian systems are dual (and equivalent) to their Lagrangian counterparts.  To see this, consider a Lagrangian function $L(t, \qq, \dot\qq)$ defined in terms of a position variable $\qq\in\mathbb{R}^N$ and its associated velocity $\dot{\qq}$.  Then, under some regularity conditions (see e.g. \cite{touchette2005legendre}), there is a canonical Legendre transformation 
\begin{equation*}
    H(t, \qq, \pp) = \sup_{\dot{\qq}} \lr{\IP{\pp}{\dot{\qq}}-L(t,\qq,\dot\qq)},
\end{equation*}
which yields the conjugate momentum vector $\pp = L_{\dot{\qq}} := \nabla_{\dot{\qq}}L \in \mathbb{R}^N$.  Substituting $L = \IP{\pp}{\dot{\qq}}-H$ in the usual action integral $S = \int L\, dt$ and computing the first variation now leads immediately to Hamilton's equations for the state $\xx=\begin{pmatrix}\qq &\pp\end{pmatrix}^\intercal\in\mathbb{R}^{2N}$,
\[ \dot{\xx} = \begin{pmatrix} \dot{\qq} \\ \dot{\pp}\end{pmatrix} = \begin{pmatrix}\bb{0} & \bb{I} \\ -\bb{I} & \bb{0}\end{pmatrix} \begin{pmatrix} H_{\qq} \\ H_{\pp}\end{pmatrix} = \JJ\nabla H(\xx),\]
similar to the above.  Notice that $\JJ$ is anti-involutive and (trivially) satisfies the Jacobi identity, which implies that this Hamiltonian system is in canonical form.  Conversely, systems of the form \eqref{eq:hamilFOM} for which $\LL\neq\JJ$ are said to be noncanonical.  Noncanonical Hamiltonian systems are quite flexible and have an important property:  elements in the kernel of $\LL$, called Casimirs, are invariant quantities, meaning that many (but not all) constants of motion in a noncanonical Hamiltonian system can be identified directly from its Poisson structure.  Since Casimir invariants are often directly responsible for the long-time behavior of the system, it is important that they are appropriately respected by model reduction methods.  It will be shown in Section~\ref{sec:numerics} that the particular Hamiltonian OpInf methods developed in Section~\ref{sec:hopinf} attend to this issue at least as well as the current state of the art.  

Although the Hamiltonian and Lagrangian formalisms can often be freely exchanged, many interesting dynamical systems which are readily modeled using the Hamiltonian formalism do not have an unconstrained Lagrangian formulation.  For example, every completely integrable equation, including the KdV equation considered in Section~\ref{sec:numerics}, has a bi-Hamiltonian structure and therefore a singular Legendre transformation.  Therefore, these systems can only be expressed in Lagrangian terms if the argument to the Lagrangian is constrained to be a derivative of the state variable (see, e.g., \cite{nutku1985on,nutku2000lagrangian} for the case of KdV).  This makes working directly with the Hamiltonian formulation of a dynamical system preferable in many cases, and encourages the search for model reduction techniques which are more general than those developed for canonical Hamiltonian systems.  In particular, the nonintrusive methods of Section~\ref{sec:hopinf} are well adapted to noncanonical Hamiltonian systems and do not appeal to Lagrangians or Legendre transforms.

% As mentioned in the Introduction, noncanonical Hamiltonian systems contain important degeneracies and can only be canonized on the complement of these.\red{make sure to mention this}

% it is important that structure-preserving model reduction techniques can respect them appropriately.  The next subsection discusses specific constructions which have been developed for Hamiltonian systems. \red{fix this paragraph}

% although this ceases to be true in the case that the system in question is noncanonical.  

% s imposes strong constraints on their
% dynamics with, for instance the existence of differential invariants, and the possibility of symmetry-
% related conservation laws.

% \IKTcomment{One thing that is missing from the paper is explaining up front what is the Hamiltonian form and the difference of canonical/non-canonical Hamiltonian forms.  I suggest adding a subsection here, and defining the $~J$
% matrix, and then properties of $~L$ for the non-canonical case.  Maybe give examples of some Hamiltonian systems (just list in words).  I would also discuss $~p$ and $~q$, which is mentioned in the next sub-section but never really defined.}{}

\subsection{Proper Orthogonal Decomposition for Hamiltonian systems}\label{subsec:hampod}

Given a large semidiscrete Hamiltonian system \eqref{eq:hamilFOM}, it is often necessary to perform model reduction in order to produce a feasible surrogate.  Typically, this means constructing an informative reduced basis for the solution space to the system onto which the dynamics can be projected.  While there are a variety of linear and nonlinear methods for accomplishing this task (including those in \cite{tu2014on,schmid2022dynamic,lee2020model}, to mention a few), this paper focuses on the linear technique known as Proper Orthogonal Decomposition (POD) which has seen the most widespread success.  POD uses snapshots $\xx\in\mathbb{R}^N$ of the high-fidelity model solution to construct a variance-maximizing subspace in which reduced solutions can be represented.  To explain this more precisely, let $\bb{Y}\in\mathbb{R}^{N\times n_s}$ be a matrix with rank $r\leq\min\{N,n_s\}$ containing $n_s$ snapshots of the high-fidelity solution $\bb{y} = \xx-\xx_0 \rev{\in\mathbb{R}^N}$ shifted by the initial condition $\xx_0:=\xx(0)$. Such snapshots could be collected at, e.g., discrete points in the interval $[0,T]$, where $T\in\mathbb{R}$ represents the final simulation time.  If $\bb{Y} = \tilde{\uu}\bm{\Sigma}\bb{V}^\intercal$ is the singular value decomposition of this mean-centered data matrix, standard computations show that the matrix $\uu\in \mathbb{R}^{N\times n}$ comprised of the first $n<r$ columns of $\tilde{\uu}$ minimizes the $L^2\lr{[0,T]}$ reconstruction error of $\bb{y}$, and that this error is precisely the sum of the remaining squared singular values \cite{kunisch2001galerkin}.  More precisely, it follows that
\[ \norm{\bb{y} - \uu\ut\bb{y}}^2 := \int_0^T \nn{\bb{y} - \uu\ut\bb{y}}^2 dt = \sum_{i=n+1}^r \sigma_i^2, \]
where $\sigma_i$ is the $i^{th}$ singular value of $\bb{Y}$.  This is the basis for the standard Galerkin POD-ROM (G-ROM) procedure, which is applied to the dynamical system governing $\xx$ by making the approximation $\xt = \xx_0 + \uu\xh\rev{\in\mathbb{R}^N}$ for some unknown coefficients $\xh\in\mathbb{R}^n$ and using that $\ut\uu = \bb{I}$ in $\mathbb{R}^n$.  In particular, inserting this approximation into the Hamiltonian system \eqref{eq:hamilFOM} yields the update rule
\[\dot{\xh} = \ut\LL(\xt)\nabla H(\xt),\]
which is low-order but obviously not Hamiltonian since $\ut\LL \neq -\lr{\ut\LL}^\intercal = \LL\uu$.  An effective remedy for this is the strategy developed in \cite{gong2017structure}, which solves the overdetermined least-squares problem $\ut\LL = \Lh\ut$ for $\Lh = \ut\LL\uu$, yielding a skew-symmetric operator which is guaranteed to produce dynamics which preserve the reduced Hamiltonian $\hat{H}=H\circ\xt\rev{:\mathbb{R}^n\to\mathbb{R}}$.  To see this, consider solving the Hamiltonian POD-ROM (H-ROM)
\[\rev{\dot\xh} = \Lh(\xt)\nabla\hat{H}(\xh). \]
Then, it follows that the change in the value of the reduced Hamiltonian along a solution is given by 
\[\dot{\hat{H}} = \dot{\xh}\cdot\nabla \hat{H} = \Lh\nabla\hat{H}\cdot\nabla\hat{H} = -\Lh\nabla\hat{H}\cdot\nabla\hat{H} = 0,\]
so that this quantity is exactly preserved up to time discretization error.

% where we have abused notation by writing $\Lh$ instead of $\Lh\circ\xt$.  \IKTcomment{Pardon my ignorance, but what does $\circ$ mean?  Should define.  Also, to some people, H-ROM means hyper-reduced ROM, but I'm OK overloading the acronym here...} \textcolor{red}{The $\circ$ is just function composition.  Do you think this needs explanation?} 

While noncanonical Hamiltonian systems are (thus far) limited to variants of the ``ordinary'' POD basis construction described above, it turns out that there are several useful ways to construct the POD basis $\uu$ in the case of canonical Hamiltonian systems. In particular, when $N=2M$ for some $M\in\mathbb{N}$ and $\xx=\begin{pmatrix} \qq & \pp \end{pmatrix}^\intercal$ separates nicely into position and momentum variables, it is frequently useful to use a basis built block-wise from sections of the snapshot data $\bb{Y}$. This is particularly true in the presence of scale separation, where the variance in one of $\qq,\pp$ will be dominated by the other if a 
standard POD basis for the full field $\begin{pmatrix} \qq & \pp \end{pmatrix}^\intercal$ is used \cite{Parish:2022, Lindsay:2022}.  In this case, separating the data $\bb{Y} = \begin{pmatrix} \bb{Y}_q & \bb{Y}_p\end{pmatrix}^\intercal$ into $M \times n_s$ blocks and carrying out the POD procedure described before yields separate bases $\uu_q,\uu_p\in\mathbb{R}^{N\times m}$ for position and momentum, which can be combined into the block basis $\uu = \mathrm{Diag}\lr{\uu_q,\uu_p}$ of size $N\times n$ where $n=2m$.  This has the effect of normalizing the importance of $\qq$ and $\pp$ in the dimension reduction, often leading to better performance in the associated ROMs.  As an added benefit, notice that both $\uu_q^\intercal\bb{Y}_q\bb{Y}_q^\intercal\uu_q$ and $\uu_p^\intercal\bb{Y}_p\bb{Y}_p^\intercal\uu_p$ are diagonal under this construction, since each POD basis is drawn from the SVD of the snapshots.

% For example, canonical Hamiltonian systems separate nicely into ODEs for position and momentum variables, so it it reasonable to consider POD bases $\uu_q,\uu_p$ built from  separate snapshots $\XX_q, \XX_p$ corresponding to position resp. momentum.  Assuming there are no significant interactions between position and momentum in the gradient of $H$, it is similarly reasonable to assume the inferred operator $\Dh$ is block-diagonal, in which case it is easy to see that the OpInf problem from before decouples into a pair of problems which independently satisfy Proposition~\ref{prop:trunc}.  

In addition to this, another block basis construction which has been demonstrably useful in the model reduction of canonical Hamiltonian systems is known as the ``cotangent lift'' algorithm from \cite{Peng_JSC_2016}.  This procedure constructs a basis such that $\uu^\intercal\JJ = \JJ_n\ut$ (for $\JJ_n\in\mathbb{R}^{n\times n}$ the canonical Poisson matrix of dimension $n$) by choosing $\uu$ from the left singular vectors of the concatenated snapshot matrix $\begin{pmatrix} \qq & \pp \end{pmatrix}\approx\bar{\uu}\bm{\Sigma}\bb{V}^\intercal\in\mathbb{R}^{M\times 2n_s}$.
% \IKTcomment{Has $\JJ_n$ been defined?  I couldn't find the definition, but maybe I missed it.}
More precisely, if $\bar{\uu}\in\mathbb{R}^{M\times m}$ contains the first $M$ left singular vectors, the basis $\uu = \mathrm{Diag}\lr{\bar{\uu},\bar{\uu}}$ satisfies the required condition.  This is quite a useful construction, as it follows that $\ut\JJ\nabla H = \JJ_n\ut\nabla H = \JJ_n\nabla\hat{H}$ and hence the prototypical G-ROM is converted into an H-ROM.  On the other hand, it is clear that $\bar{\uu}^\intercal\bb{Y}_q\bb{Y}_q^\intercal\bar{\uu}$ is not diagonal (and same for $p$), since $\bb{V}^\intercal = \begin{pmatrix}\bb{V}_1 &\bb{V}_2\end{pmatrix}^\intercal\in\mathbb{R}^{m\times 2n_s}$ and so $\bb{V}_1^\intercal\bb{V}_1\neq \bb{V}_2^\intercal\bb{V}_2\neq \bb{V}_1^\intercal\bb{V}_1 + \bb{V}_2^\intercal\bb{V}_2 = \bb{I}$.

% \[\uu = \begin{pmatrix}\bar{\uu} & \bb{0} \\ \bb{0} & \bar{\uu}\end{pmatrix},\]

% On the other hand, there is another useful basis generation algorithm in the case of canonical Hamiltonian systems which does not lead to operators that obey Proposition~\ref{prop:trunc}.  This is the cotangent lift of \cite{peng2017}, which 

\subsection{Generic Operator Inference}\label{subsec:genopinf}

Consider a dynamical system of the form 
\begin{equation}\label{eq:genFOM}
    \dot{\xx}\lr{t,\mmu} = \bb{F}\lr{t,\mmu,\xx\lr{t,\mmu}},
\end{equation}
where $\xx:\mathbb{R}\times\mathbb{R}^p\to\mathbb{R}^N$ is a time-dependent state variable and $\mmu \in \mathbb{R}^p$ is a vector of parameters.  As mentioned previously, constructing a POD basis $\uu\in\mathbb{R}^{N\times n}$ and making the approximation $\xt = \xx_0 + \uu\xh$ leads to the canonical Galerkin ROM,
\begin{equation}\label{eq:genROM}
\dot{\hat{\xx}}\lr{t,\mmu} = \ut\bb{F}\lr{t,\mu,\xt\lr{t,\mmu}},
\end{equation}
which is an $n$-dimensional dynamical system describing the evolution of the POD basis coefficients.  While this procedure is well studied and often effective, it clearly requires intrusion into the FOM simulation code via access to the operators governing the high-fidelity system \eqref{eq:genFOM}, since it is necessary to assemble $\ut\bb{F}(\xt)$.  In the case that $\xx_0=\bm{0}$ and $\bb{F}(\xx) = \bb{D}\xx$ is linear, this means direct access to $\bb{D}\in\mathbb{R}^{N\times N}$ is needed in order to assemble the reduced operator $\Dh = \ut\bb{D}\uu\in\mathbb{R}^{n\times n}$.  However, it is frequently impossible (or prohibitively expensive) to access this information, due to, e.g., complicated or proprietary legacy codes.  
% \IKTcomment{Our method needs some info too from the FOM code like the mass and stiffness matrices...}{}
This motivates the black-box operator inference technique (OpInf) of \cite{willcox} which is used for the non-intrusive modeling of dynamical systems such as \eqref{eq:genFOM}.  More precisely, OpInf uses snapshot data to learn the tensor coefficients of a polynomial approximation $\bb{D}_0,...,\bb{D}_n$ to the action of $\bb{F}$, so that if $\xt$ satisfies 
\[\dot{\bb{z}} = \bb{D}_0 + \bb{D}_1\bb{z} + \bb{D}_2\lr{\bb{z}\otimes\bb{z}} + ... + \bb{D}_n\lr{\bb{z}\otimes...\otimes\bb{z}},\] 
then $\bb{z}\approx\xx$ remains close to a solution to the original system.  This ansatz is clearly exact in the case that the model in question is polynomial (or differentially polynomial), but has been shown to be useful even outside of this case, see, e.g., \cite{Swischuk:2020,Ghattas:2021} \rev{and lifting transformations \cite{Qian:2020, McQuarrie:2021, Qian:2020, Swischuk:2020,
Zastrow:2023}}.  Moreover, it readily extends to learning the coefficients of a POD-based ROM, since reduced basis projection preserves this polynomial structure.  

% \IKTcomment{Mention idea of lifting, and the fact that projection preserves polynomial structure?  That is missing currently.} \red{I don't think we need to say anything about lifting, because I don't use it anywhere.}  \IKTcomment{That's fine.  I omitted lifting in the intro for the same reason, though I made a note of it there in case you felt it was worth mentioning in the intro.}
  
To see this in detail, consider learning a linear approximation $\dot{\bb{z}}=\bb{D}\bb{z}$ (for simplicity), and suppose an $N\times n_s$ matrix $\XX$ of snapshot data is provided containing (partial) trajectories of the original system \eqref{eq:genFOM}.  Then, an approximation $\XX_t \approx \dot{\XX}$ to the temporal derivative of each snapshot can be formed through, e.g., finite differences, and the matrix least-squares problem, 
\[ \argmin_{\bb{D}\in\mathbb{R}^{N\times N}}R\lr{\bb{D}} = \argmin_{\bb{D}\in\mathbb{R}^{N\times N}}\nn{\XX_t - \bb{D}\XX}^2, \]
can be solved to yield the desired operator $\bb{D}$.  More precisely, exterior differentiation yields 
\[dR\lr{\bb{D}} = -2\IP{\XX_t-\bb{D}\XX}{d\mathbf{D}\,\XX} = -2\IP{\lr{\XX_t-\bb{D}\XX}\XX^\intercal}{d\mathbf{D}}, = \IP{\nabla R\lr{\bb{D}}}{d\bb{D}},\]
so that solving $\nabla R\lr{\bb{D}} = -2\lr{\XX_t-\bb{D}\XX}\XX^\intercal = \bb{0}$ reduces to solving the linear system 
\[\bb{DXX}^\intercal = \XX_t\XX^\intercal.\]
Alternatively, there is the equivalent vectorized system
\[\lr{\XX\XX^\intercal\otimes\bb{I}}\vect\bb{D} = \vect\lr{\XX_t\XX^\intercal},\]
where $\otimes$ denotes Kronecker's matricized tensor product and equivalence follows via the ``vec trick'' (see Appendix~\ref{app:kron} for a review of these ideas).  

% \IKTcomment{Again, I think you should define ``vec", as I am not sure this is common notation.  Perhaps you plan to do this in the appendix?}

% In fact, to accomplish this it suffices (by convexity) to compute a solution to $\nabla R = \bb{0}$.

% Systems like \eqref{eq:genFOM} may arise from, e.g., the spatial discretization of a system of PDEs, and in the event that $N$ is high-dimensional, it is often necessary to  perform model reduction in order to efficiently simulate such dynamics.  A common intrusive method for this purpose is the Proper Orthogonal Decomposition (POD), which is now briefly recalled.  Let 

% As an example, suppose $\bb{F}=\bb{F}(\xx)$ has no explicit time or parameter dependence.  Then, we can write $\bb{F}(\xx) = \bb{Dx} + \bb{F}_{\mathrm{nl}}(\xx)$

% Alternatively, applying the ``vec trick'' (see Appendix) yields an equivalent vectorized system
% \[ \lr{\XX\XX^\intercal\otimes\bb{I}}\vect\bb{D} = \vect\lr{\XX_t\XX^\intercal}, \]
% where $\otimes$ denotes Kronecker's matricized tensor product (see Appendix for a review).  Solving this 

% $\XX\approx\uu\bb{\Sigma V}^\intercal$ with $\bb{U}\in\mathbb{R}^{N\times n}$, then the columns of $\bb{U}$ form a basis for the variance-maximizing subspace of dimension $n$, and any solution $\xx\in\mathbb{R}^N$ to $\dot{\xx}=\bb{F}(\xx)$ can be approximated with the Galerkin approximation

On the other hand, in practical application settings, it is usually undesirable (or even infeasible) to infer the full $N\times N$ operator $\bb{D}$ in this way, as this requires solving a linear system which scales with $N^2$.  Instead, it is more useful to combine OpInf with dimension reduction techniques such as POD, since, if $\xt = \uu\xh$ where $\uu\in\mathbb{R}^{N\times n}$ is a POD basis and $\xh\in\mathbb{R}^n$, then the snapshot data $\XX$ and its approximate time derivative $\XX_t$ can be projected onto this basis before inferring a reduced operator.  In particular, there are the $n\times n_s$ matrices $\Xh = \ut\XX$ and $\Xh_t = \ut\XX_t$, which can be used to infer a lower-dimensional operator $\Dh\in\mathbb{R}^{n\times n}$ governing the non-intrusive reduced dynamical system $\dot{\xh} = \Dh\xh$.  In this case, $\Dh$ is inferred through the reduced OpInf problem of size $n$,
\[\argmin_{\Dh\in\mathbb{R}^{n\times n}} \nn{\Xh_t-\Dh\Xh}^2,\]
which is solved as described above.

% Besides reducing computational costs, inferring the reduced operator $\Dh$ has another added benefit in many cases due to the hierarchical order of the columns of $\bb{U}$.  To describe this, note that if $n'<n$ and $\Dh$ solves the OpInf problem of size $n$, then for any $1\leq i',j' \leq n'$, $n'\leq K\leq n$,
% \[ \lr{\Xh_t\Xh^\intercal}^{i'}_{j'} = \hat{D}^{i'}_k\lr{\Xh\Xh^\intercal}^k_{j'} = \hat{D}^{i'}_{k'}\lr{\Xh\Xh^\intercal}^{k'}_{j'} + \hat{D}^{i'}_{K}\lr{\Xh\Xh^\intercal}^{K}_{j'},\]
% so that the top-left $n'\times n'$ truncation of $\Dh$ is close to the solution $\Dh'$ to the OpInf problem of size $n'$ provided that the entries $\hat{D}^{i'}_{K}\lr{\Xh\Xh^\intercal}^{K}_{j'}$ are small.  This happens if, for example, the subspacae spanned by the last $n-n'$ columns of $\uu$ is nearly invariant under the action of $\XX\XX^\intercal$.

Besides reducing computational costs relative to inference of the full operator $\bb{D}$, inferring the reduced operator $\Dh$ has the following added benefit due to the hierarchical order of the columns of $\bb{U}$. While this result appears to be well known, the lack of a standard reference has motivated the inclusion of a proof in Appendix~\ref{app:proofs}.

% Moreover, due to decreasing order of importance of the columns of $\bb{U}$ to the approximate solution, there is the following result about work re-use:

% we can form the Galerkin approximation $\xt = \uu\xh$ for $\xh\in\mathbb{R}^n$ and search for a reduced operator $\hat{\bb{D}}\in\mathbb{R}^{n\times n}$ satisfying $\dot{\xh}=\hat{\bb{D}}\xh$ which converges to $\bb{D}$ as $n\to N$. 

\begin{proposition}\label{prop:trunc}
Suppose $\bb{U}\in\mathbb{R}^{N\times n}$ is the matrix of left singular vectors of some data matrix $\XX$, and $\bm{\Sigma}\in\mathbb{R}^{n\times n}$ is the corresponding diagonal matrix of (nonzero) singular values $\{\sigma_j\}_{j=1}^n$. Then, the unique solution to the OpInf problem of size $n$ is given by 
% \IKTcomment{I changed $\Dh$ to $\bb{D}$
% in the objective function being minimized, as I think this is more correct.}
\[ \Dh = \argmin_{\bb{D}\in\mathbb{R}^{n\times n}}\nn{\Xh_t - \bb{D}\Xh}^2 = \Xh_t\Xh^\intercal\bm{\Sigma}^{-2}.\]
Moreover, for any $n'<n$, the submatrix $\Dh'\in\mathbb{R}^{n'\times n'}$ formed by extracting the first $n$ rows and columns of $\hat{\bb{D}}$ is the solution to the corresponding OpInf problem of size $n$.
\end{proposition}

\begin{remark}\label{rem:tikhonov}
Note that the conclusion of Proposition~\ref{prop:trunc} continues to hold if the minimization objective is Tikhonov regularized by a multiple of $\Dh$, as can be checked by considering the minimization objective $\nn{\Xh_t-\Dh\Xh}^2 + \eta\nn{\Dh}^2$ for some $\eta>0$ and repeating the arguments above.  Moreover, the conclusion also holds block-wise if $\uu$ is a block basis as discussed in Section~\ref{subsec:hampod} and $\Dh$ is block diagonal, since the relevant problem decouples over the blocks of $\Dh$.
\end{remark}

\subsection{\rev{Previous Hamiltonian Operator Inference}}\label{subsec:prevhopinf}
\rev{The idea of using OpInf in conjunction with Hamiltonian systems has been previously explored in \cite{sharma2022hamiltonian}, where it was used to learn the linear part of the Hamiltonian gradient $\nabla H$ for a sub-class of canonical Hamiltonian systems which are known as separable.  The separability assumption implies that the system Hamiltonian decomposes as $H(\qq,\pp) = T(\qq) + V(\pp)$ for some real-valued functions $T,V$ depending only on $\qq,\pp$, respectively.  A consequence of this is that the gradient $\nabla H$ becomes block-diagonal in the variables $\qq,\pp$, a fact which is preserved at the POD-ROM level as long as a cotangent lift POD basis $\uu\in\mathbb{R}^{N\times n}$ satisfying $\ut\JJ\uu=\JJ_n$ is employed, where $\JJ_n$ is the canonical symplectic matrix of dimension $n=2m$.  With this additional restriction on the reduced basis, it follows that the intrusive H-ROM for the approximation $\xt = \uu\xh\approx\xx$ decouples over $\qq,\pp$, becoming $\dot{\xh} = \JJ_n\nabla\hat{H}(\xt) = \JJ_n\lr{\Ah\xh + \nabla\hat{f}(\xt)}$ where $\Ah = \mathrm{diag}\lr{\Ah_{\bb{q}},\Ah_{\bb{p}}}$ is block-diagonal and we have written $H(\xx) = \frac{1}{2}\bb{x}^\intercal\bb{Ax}+f(\xx)$.  This allows the authors of \cite{sharma2022hamiltonian} to formulate an inference procedure for the linear operator $\Ah$ which decouples block-wise into two $m^2\times m^2$ subproblems for $\Ah_\qq, \Ah_\pp$, provided snapshots of $\nabla f$ can be obtained and this quantity can be simulated online.  More precisely, given snapshots $\XX\in\mathbb{R}^{N\times n_s}$ of the full order solution along with snapshots $\nabla f(\XX)\in\mathbb{R}^{N\times n_s}$ of the nonlinear part of $\nabla H$, the problems to solve are
\begin{align*}
    \argmin_{\Ah\in\mathbb{R}^{m\times m}}\nn{\Xh_{\pp,t}+\Ah\Xh_\qq + \hat{\nabla}_\qq f(\XX)}^2, &\quad\mathrm{s.t.}\quad \Ah^\intercal=\Ah, \\
    \argmin_{\Ah\in\mathbb{R}^{m\times m}}\nn{\Xh_{\qq,t}-\Ah\Xh_\pp - \hat{\nabla}_\pp f(\XX)}^2 &\quad \mathrm{s.t.}\quad \Ah^\intercal=\Ah,
\end{align*}
where subscripts on $\qq,\pp$ denote either the first or second $m$ rows in the snapshot matrix, $\nabla_\qq,\nabla_\pp$ denote partial derivatives, subscript $t$ denotes a finite difference approximation to the time derivative, and ``hat'' indicates the application of a basis projection $\ut$. This yields $\bb{A}_{\bb{q}}$ and $\bb{A}_{\bb{p}}$, respectively, which once learned can be used to simulate the differential equation for $\dot{\xh}$ as usual to yield the approximation $\xt$.  While this procedure has been previously useful for simulating several systems of interest, it will be shown in Section~\ref{sec:hopinf} that the restrictions inherent in this algorithm can be removed, leading to a Hamiltonian OpInf ROM for canonical systems which does not require separability of the Hamiltonian or a specific choice of reduced basis.}

% the canonical Hamiltonian operator inference (H-OpInf) method of \cite{sharma2022} considers separately the

% This also shows that for $1\leq i,j\leq n$, we have
% \[ \lr{\Xh_t\Xh^\intercal}^i_j = \hat{D}^i_j \]

% letting $1\leq i,j\leq n$ we can extract and solve the minimization problem for component $\hat{D}_{ij}$,
% \[\argmin_{\hat{D}_{ij}\in\mathbb{R}} \nn{\bb{u}_i^\intercal\bb{X}_t - \hat{D}_{ij}\,\bb{u}_j^\intercal\bb{X}}^2 = \frac{\bb{u}_i^\intercal\bb{X}_t\bb{X}^\intercal\bb{u}_j}{\nn{\bb{X}^\intercal\bb{u}_j}^2},\]
% showing that each entry $\hat{D}_{ij}$ depends only on the columns of $\bb{U}$ corresponding to $i,j$.  Therefore, the solution to the OpInf problem of size $n<r$ can be extracted from $\hat{\bb{D}}$ by extracting the relevant submatrix.

% which has unique solution 
% \[\hat{D}_{ij} = \frac{\bb{u}_i^\intercal\bb{X}_t\bb{X}^\intercal\bb{u}_j}{\bb{u}_i\bb{XX}^\intercal\bb{u}_j}.\]

\subsection{Linear ROMs and Average Vector Field Integration}\label{subsec:linrom}

To simplify the presentation of later results, it is worth mentioning some facts about linear ROMs and the particular timestepping scheme used in this work.  First, note that the average vector field (AVF) method \cite{celledoni2012preserving,karasozen2013energy} is employed for time integration of all numerical examples, meaning that the Hamiltonian dynamical system \eqref{eq:hamilFOM} is discretized as 
\[ \frac{\xx^{k+1} - \xx^k}{\Delta t} = \LL\lr{\xx^{k+\frac{1}{2}}}\int_0^1 \nabla H\lr{t\xx^{k+1} + (1-t)\xx^k}\, dt,\]
where $\xx^{k+\frac{1}{2}}= \frac{1}{2}\lr{\xx^k+\xx^{k+1}}$, which amounts to linearizing the trajectory of the state between time steps $k$ and $k+1$ and fixing evaluation of $\LL$ at the midpoint.  This integration scheme has appealing properties including exact quadrature for polynomial nonlinearities, as well as second-order convergence in time.  Moreover, it is easy to see that AVF integration is globally energy-conserving: if $\bm{\ell}(t) = t\xx^{k+1} + (1-t)\xx^k$, it follows from the symmetry relation $\LL^\intercal = -\LL$ that
\begin{align*}
    \frac{H\lr{\xx^{k+1}} - H\lr{\xx^k}}{\Delta t} &= \frac{1}{\Delta t}\int_0^1 \frac{d}{dt} H\lr{\bm{\ell}(t)}\,dt = \frac{\xx^{k+1}-\xx^k}{\Delta t}\cdot\int_0^1 \nabla H\lr{\bm{\ell}(t)}\,dt \\ 
    &= \left[\LL\lr{\xx^{k+\frac{1}{2}}}\int_0^1 \nabla H\lr{\bm{\ell}(t)}\,dt\right] \cdot \int_0^1 \nabla H\lr{\bm{\ell}(t)}\,dt = 0,
\end{align*}
so that there can be no loss of energy during AVF timestepping.  

Now, when $\dot{\xx} = \bb{D}\xx$ is linear, it is clear that AVF integration reduces to the implicit midpoint method
\[ \frac{\xx^{k+1} - \xx^k}{\Delta t} = \bb{D}\xx^{k+\frac{1}{2}}, \]
which can be easily solved at each time step $k$ by writing $\xx^{k+1} = \xx^k + \Delta \xx^k$, where $\Delta\xx^k = \xx^{k+1} - \xx^k$ satisfies the linear system
\begin{equation*}\label{eq:linearGROM}
    \lr{\bb{I} -\frac{\Delta t}{2}\bb{D}}\Delta\xx^k = \Delta t\, \bb{D}\xx^k.
\end{equation*}
Note additionally that if $\xt = \xx_0 + \uu\xh \approx \xx$ is a mean-centered Galerkin projection and $\Dh=\ut\bb{D}\uu$ is the intrusive projection of $\bb{D}$, this implies the low-order update formula $\xh^{k+1} = \xh^k + \Delta\xh^k$, where 
\[ \lr{\Ih-\frac{\Delta t}{2}\Dh}\Delta\xh^k = \Delta t\lr{\ut\bb{D}\xx_0 + \Dh\xh^k}, \]
which is an $n\times n$ linear solve leading to the full-order approximate $\xt^{k+1} = \xx_0 + \uu\xh^{k+1}$.  Of course, in the event that $\bb{D}$ is not available and so $\Dh$ must be  inferred, this mean-centering can be ignored.  
% \IKTcomment{I worry about this comment a little bit.  Some of the results you show suggest that without MC, the ROMs can have issues/blow up.  
% Also, typically you'd need $\xx_0$ for Dirichlet BC enforcement in the ROM following an approach proposed by Max actually.  Basically I'm wondering
% if one is ever dead in the water w/o mean centering.  May not be worth getting into in the paper, just asking for my own interest.}
Finally, to specify these expressions to linear Hamiltonian systems $\dot{\xx}=\LL\Aa\xx$, it suffices to replace $\Dh = \widehat{\LL\Aa} = \ut\LL\Aa\uu$ in the case of the G-ROM and $\Dh=\Lh\Ah = \ut\LL\uu\ut\Aa\uu$ in the case of the H-ROM.

\section{Hamiltonian operator inference}\label{sec:hopinf}
It is now possible to discuss the present methods for canonical and noncanonical Hamiltonian OpInf, \rev{which will be} referred to as C-H-OpInf and NC-H-OpInf, respectively. 
% \textcolor{red}{Comment from Irina: I don't think C-H-OpInf and NC-H-OpInf have been defined yet.}
First, note the following computational result central to these techniques. 

% \IKTcomment{Again, I think you need to give a precise definition of canonical and non-canonical.  I propose doing it in Section \ref{subsec:hamiltonian}.}

\begin{theorem}\label{thm:main}
Let $\bb{A}\in\mathbb{R}^{N\times N}$, $\bb{B},\bb{C}\in\mathbb{R}^{N\times n_s}$, and define $\bb{A}\baroplus\bb{B} = \bb{A}\otimes\bb{B} + \bb{B}\otimes\bb{A}$.  Then, every solution to the \rev{symmetry-constrained} least-squares regularization problem 
\[\rev{\argmin_{\DD\in\mathbb{R}^{N\times N}}\nn{\bb{C}-\bb{ADB}}^2, \quad \mathrm{s.t.} \quad \bb{D}^\intercal=\pm\bb{D},} \]
corresponds to a solution to the vectorized problem
\[\rev{\lr{\bb{A}^\intercal\bb{A} \baroplus \bb{BB}^\intercal}\vect\DD = \vect\lr{\bb{A}^\intercal\bb{C}\bb{B}^\intercal\pm\bb{B}\bb{C}^\intercal\bb{A}}.}\]
In particular, the first system is uniquely solvable if and only if the second is also, \rev{which holds whenever $\bb{A,B}$ have maximal rank}.
\end{theorem}

% Let $\eta>0$, $\bb{A}\in\mathbb{R}^{N\times N}$, $\bb{B},\bb{C}\in\mathbb{R}^{N\times n_s}$, and let $\bb{K}\in\mathbb{R}^{N^2\times N^2}$ be such that $\vect\DD^\intercal = \bb{K}\vect\DD$.  Define $\bb{A}\baroplus\bb{B} = \bb{A}\otimes\bb{B} + \bb{B}\otimes\bb{A}$.  Then, every solution to the (Tikhonov regularized) least-squares minimization problem
% \[\argmin_{\DD\in\mathbb{R}^{N\times N}}\lr{\nn{\bb{C}-\bb{A\lr{D\pm D^\intercal}B}}^2 + \eta\nn{\DD}^2}, \]
% corresponds to a solution to the vectorized problem
% \[\lr{\lr{\bb{A}^\intercal\bb{A} \baroplus \bb{BB}^\intercal}\lr{\bb{I}\pm\bb{K}}+\eta\bb{I}}\vect\DD = \vect\lr{\bb{A}^\intercal\bb{C}\bb{B}^\intercal\pm\bb{B}\bb{C}^\intercal\bb{A}}.\]
% In particular the first system is uniquely solvable if and only if the second is also.

\begin{proof}
\rev{First, note that the uniqueness condition follows immediately from the fact that the objective is convex, the symmetry constraint is linear, and $\mathrm{rank}\lr{\bb{B}\otimes\bb{A}} = \mathrm{rank}\lr{\bb{B}}\mathrm{rank}\bb{A}$.  The remainder will follow from a direct calculation using the method of Lagrange multipliers.  More precisely, define the Lagrangian $L\lr{\bb{D},\bm{\Lambda}} = \frac{1}{2}\nn{\bb{C}-\bb{ADB}}^2 + \IP{\bm{\Lambda}}{\bb{D}\mp\bb{D}^\intercal}$ where $\bm{\Lambda}\in\mathbb{R}^{N\times N}$ is a matrix of Lagrange multipliers.}  Then, exterior differentiation yields
\rev{\begin{align*}
    dL\lr{\DD,\bm{\Lambda}} &= -\IP{\bb{C}-\bb{ADB}}{\bb{A}\,d\DD\,\bb{B}} + \IP{\bm{\Lambda}}{d\DD\mp d\DD^\intercal} + \IP{d\bm{\Lambda}}{\DD\mp\DD^\intercal} \\
    &= \IP{d\bb{D}}{-\bb{A}^\intercal\lr{\bb{C}-\bb{ADB}}\bb{B}^\intercal + \bm{\Lambda}\mp\bm{\Lambda}^\intercal} + \IP{d\bm{\Lambda}}{\DD\mp\DD^\intercal}.
\end{align*}
Setting this to zero yields the first-order optimality conditions
\begin{align*}
    \bb{A}^\intercal\lr{\bb{C}-\bb{ADB}}\bb{B}^\intercal &= \bm{\Lambda}\mp\bm{\Lambda}^\intercal, \\
    \DD\mp\DD^\intercal &= \bm{0}.
\end{align*}
Examining the first condition the right-hand side implies symmetry in the left-hand side, allowing for easy elimination of $\bm{\Lambda}$ through the expression
\begin{align*}
    \bb{A}^\intercal\lr{\bb{C}-\bb{ADB}}\bb{B}^\intercal \pm \bb{B}\lr{\bb{C}-\bb{ADB}}^\intercal\bb{A} = \bm{0}
\end{align*}.
Expanding the above and using the second condition $\DD\mp\DD^\intercal=\bm{0}$ then yields
\[\bb{A}^\intercal\bb{ADBB}^\intercal\pm\bb{BB}^\intercal\DD^\intercal\bb{A}^\intercal\bb{A} = \bb{A}^\intercal\bb{ADBB}^\intercal + \bb{BB}^\intercal\bb{DA}^\intercal\bb{A} = \bb{A}^\intercal\bb{C}\bb{B}^\intercal\pm\bb{BC}^\intercal\bb{A},\]
which vectorizes through the ``vec trick'' (c.f. Appendix~\ref{app:kron}) to yield the claimed result.}
\end{proof}

Theorem~\ref{thm:main} provides the solution to a generic symmetric or skew-symmetric operator inference problem, which will be seen to include the C-H-OpInf and NC-H-OpInf procedures employed presently.  As mentioned before, a notable benefit of the generic OpInf procedure is that its solutions satisfy Proposition~\ref{prop:trunc}, meaning that a solution computed using a reduced basis of size $n$ remains optimal via truncation for all $n'<n$.  The next result shows that, under some (fairly strong) assumptions on $\bb{A},\bb{B}$, this ``one-shot'' ability continues to hold for the system in Theorem~\ref{thm:main}.  Since the proof is straightforward but technical, it is deferred to Appendix~\ref{app:proofs}.

% the optimal solution to the above problem can be computed for each $n<N$ by simple truncation.  Note that in these cases this corresponds to throwing away higher-frequency POD modes and retaining the approximation corresponding to the $n$ lowest frequency modes.  

\begin{proposition}\label{prop:newtrunc}
Let $\uu\in\mathbb{R}^{N\times n}$ be a POD basis.  Suppose $\Dh\in\mathbb{R}^{n\times n}$ uniquely solves the optimization problem in Theorem~\ref{thm:main} for given $\hat{\bb{A}} = \ut\bb{A}\uu\in\mathbb{R}^{n\times n}$ and $\hat{\bb{B}},\hat{\bb{C}}\in\mathbb{R}^{n\times n_s}$ defined by $\hat{\bb{B}}=\ut\bb{B}, \hat{\bb{C}}=\ut\bb{C}$.  Let $n'<n$, and for any matrix $\hat{\bb{M}}$ which is multiplied with the POD basis $\uu$, let $\hat{\bb{M}}'$ denote the submatrix obtained by removing the $n-n'$ highest-frequency basis vectors of $\uu$ from every relevant multiplication.  If $\hat{\bb{A}}$ and $\hat{\bb{B}}\hat{\bb{B}}^\intercal$ are both diagonal, \rev{then the unique solution to 
\[ \argmin_{\DD\in\mathbb{R}^{n'\times n'}} \nn{\hat{\bb{C}}'-\hat{\bb{A}}'\DD\hat{\bb{B}}'}^2, \quad \mathrm{s.t.}\quad \DD^\intercal=\pm\DD, \]
is given by the truncation $\Dh'$.}
\end{proposition}
\begin{proof}
    See Appendix~\ref{app:proofs}.
\end{proof}

% \textcolor{red}{Comment from Irina: I suggest putting ``Proof: given in Appendix XXX after the proposition just to make it clear we do have a proof and it's in the paper.}

% \[ \argmin_{\DD\in\mathbb{R}^{n'\times n'}} \lr{\nn{\hat{\bb{C}}'-\hat{\bb{A}}'\lr{\DD\pm \DD^\intercal}\hat{\bb{B}}'}^2 + \eta\nn{\DD}^2} = \Dh'. \]

While Proposition~\ref{prop:newtrunc} is useful to know, it is worth mentioning that its conclusion generally does not hold for any of the structure-preserving OpInf methods known to date, including those discussed here.  Indeed, while the diagonality of $\hat{\bb{A}}$ can often be arranged, it is more difficult to construct a suitable $\hat{\bb{B}}\hat{\bb{B}}^\intercal$ which is diagonal.  On the other hand, there are many cases when the truncated solution to Theorem~\ref{thm:main} is close enough to optimal to produce a well performing ROM, making it useful to employ truncation without the guarantee of Proposition~\ref{prop:newtrunc} provided this property is empirically verified.

% \IKTcomment{Is the problem that $\hat{~A}$ and/or $\hat{~B}\hat{~B}$ are not diagonal?} \red{Usually that $\hat{~B}\hat{~B}$ is not diagonal.}  \IKTcomment{I see.  Maybe state that explicitly?} 

% Letting $\lr{\cdot}'$ denote top-left truncation to size $n'\times n'$, this shows that 
% \begin{align*}
%     \lr{\bb{A}^\intercal\bb{C}\bb{B}^\intercal\pm\bb{B}\bb{C}^\intercal\bb{A}}' &= \lr{\bb{A}^\intercal\bb{A}}'\bb{Y}'\lr{\bb{BB}^\intercal}' + \lr{\bb{BB}^\intercal}'\bb{Y}'\lr{\bb{A}^\intercal\bb{A}}' + \alpha\XX',
% \end{align*}
% meaning that $\XX'$ solves the corresponding problem of size $n'$. 

% \begin{remark}
% This shows that the symplectic lift algo does not have this property, because $\Xh\Xh^\intercal = \bm{\Sigma}\bb{V}^\intercal\bb{O}\bb{V}\bm{\Sigma} \neq \bm{\Sigma}^2$, where $\bb{O}$ is a $(2n_t \times n_t)$ block matrix formed by vertically stacking the zero matrix and the identity matrix.  \red{Unless A is block diagonal, then it ALMOST works block-wise.}
% \end{remark}

\subsection{Canonical Hamiltonian Systems}
The first goal is to present an OpInf method applicable to canonical Hamiltonian systems, and connect it to previous work in \cite{sharma2022hamiltonian}.  Suppose snapshots 
of the form $\xx = \begin{pmatrix}\qq &\pp\end{pmatrix}^\intercal$ can be obtained, say, as the result of post-processing data from a Lagrangian system via a Legendre transformation (c.f. Section~\ref{subsec:hamiltonian} and  Section~\ref{subsec:3d_lin_elast}).  Then, given that $\qq,\pp$ are the canonical position and momentum variables, it must be true that $\LL=\JJ$ in \eqref{eq:hamilFOM} and the Hamiltonian system to be modeled is in canonical form.  In this case, a Hamiltonian OpInf procedure can be considered which requires only knowledge of the nonlinear part of $\nabla H$.  To see this, recall that the discrete Hamiltonian can be expressed as
\[H(\xx) = \rev{\frac{1}{2}}\xx^\intercal\Aa\xx + f(\xx),\]
for a symmetric, potentially unknown $\Aa\in\mathbb{R}^{N\times N}$, and a known nonlinear function $\bb{f}:\mathbb{R}^N\to\mathbb{R}^N$.  It follows that the gradient is given by $\nabla H(\xx) = \Aa\xx + \nabla\bb{f}(\xx)$, and any POD basis $\uu\in\mathbb{R}^{N\times n}$ yields a reduced Poisson operator $\hat{\bb{J}} = \ut\bb{J}\uu$ corresponding to the H-ROM discussed in Section~\ref{subsec:hampod}.  Notice that $\JJ$ has a canonical form, so that this operator can be computed without intrusion into any simulation code.  Making the obvious Galerkin projection $\xt = \uu\xh$ and writing $\hat{H} = H\circ\xt$, $\hat{\bb{f}}=\bb{f}\circ\xt$ then yields the reduced Hamiltonian
\[\hat{H}\lr{\xh} = \rev{\frac{1}{2}}\xh^\intercal\Ah\xh + \hat{f}\lr{\xh},\]
which depends on the symmetric, potentially unknown reduced operator $\Ah\in\mathbb{R}^{n\times n}$.  Provided $\Ah$ can be computed or inferred, access to $\nabla\bb{f}$ then implies solvability of the H-ROM 
\begin{equation}\label{eq:CHrom}
    \dot{\xh} = \Jh\nabla\hat{H}(\xh) = \Jh\lr{\Ah\xh + \nabla\hat{f}(\xh)}, 
\end{equation}
which will be a low-order Hamiltonian system approximating the original dynamics.

% \[\nabla\hat{H}(\xh) = \ut\Aa\uu\xh + \ut\nabla\bb{f}(\uu\xh) = \Ah\xh + \nabla\hat{\bb{f}}(\xh),\]
% which depends on the symmetric, unknown reduced matrix $\Ah\in\mathbb{R}^{n\times n}$.

% Computing a POD approximation $\XX \approx \bb{U\Sigma V}^\intercal$ of snapshot data, there is the obvious Galerkin projection $\xt\approx\uu\xh$ and ``intrusive'' Poisson operator $\hat{\bb{J}} = \ut\bb{J}\uu$, which can be computed without access to the FOM.    Further writing $\hat{H} = H\circ\xt$ and $\hat{\bb{f}}=\bb{f}\circ\xt$, it follows that 
% \[\nabla\hat{H}(\xh) = \ut\Aa\uu\xh + \ut\nabla\bb{f}(\uu\xh) = \Ah\xh + \nabla\hat{\bb{f}}(\xh).\]

Since $\bb{A}$ is unavailable in the present setting, \eqref{eq:CHrom} is most readily solved by setting up a tractable inference problem for $\Ah$.  This means forming the appropriate reduced quantities from snapshot data and \rev{solving the constrained least-squares problem
\begin{equation}\label{eq:min_prob}
    \argmin_{\Ah\in\mathbb{R}^{n\times n}}\nn{\hat{\XX}_t - \hat{\bb{J}}\Ah\hat{\XX}+\hat{\nabla} f(\XX)}^2, \quad \mathrm{s.t.}\quad \Ah^\intercal=\pm\Ah,
\end{equation}
% \begin{equation}\label{eq:min_prob}
%     \Bh=\argmin_{\bb{B}\in\mathbb{R}^{n\times n}}\lr{\nn{\hat{\XX}_t - \hat{\bb{J}}\lr{\lr{\bb{B}+\bb{B}^\intercal}\hat{\XX}+\hat{\nabla}\bb{f}(\XX)}}^2 + \eta\nn{\bb{B}}^2},
% \end{equation}
which has minimizer $\Ah$ satisfying the desired symmetry}.  In \eqref{eq:min_prob},  $\hat{\nabla}f(\XX)=\ut\nabla f(\XX)$ denotes the projection of the snapshot data for the derivative of the nonlinear term.
%, and $\eta>0$ represents an optional regularization parameter.
Applying Theorem~\ref{thm:main} with $\bb{C} = \Xh_t - \Jh\hat{\nabla} f\lr{\XX}$ yields the equivalent linear system 
\rev{\begin{equation}\label{eq:C-H-OpInf-real}
    \lr{\Jh^\intercal\Jh\baroplus\Xh\Xh^\intercal}\vect\Ah = \vect\lr{\Jh^\intercal\Xh_t\Xh^\intercal+\Xh\Xh_t^\intercal\Jh-\Jh^\intercal\Jh\hat{\nabla}\bb{f}(\XX)\Xh^\intercal-\Xh\hat{\nabla} f(\XX)^\intercal\Jh^\intercal\Jh},
\end{equation}}
% \begin{equation}\label{eq:C-H-OpInf-real}
%     \lr{\lr{\Jh^\intercal\Jh\baroplus\Xh\Xh^\intercal}\lr{\bb{I}+\bb{K}} + \eta\bb{I}}\vect\Bh = \vect\lr{\Jh^\intercal\Xh_t\Xh^\intercal+\Xh\Xh_t^\intercal\Jh-\Jh^\intercal\Jh\hat{\nabla}\bb{f}(\XX)\Xh^\intercal-\Xh\hat{\nabla}\bb{f}(\XX)^\intercal\Jh^\intercal\Jh},
% \end{equation}
which is guaranteed (see Section~\ref{sec:theory}) to yield an operator $\Ah$ which converges to $\ut\Aa\uu$ in an appropriate limit.  Interestingly, it is even more useful in practice to make the approximation $\Jh^\intercal\Jh\approx\bb{I}$ in \eqref{eq:C-H-OpInf-real}, which is exact for the cotangent lift algorithm discussed in Section~\ref{subsec:hampod}, yielding the alternative linear system
\rev{\begin{align}\label{eq:C-H-OpInf}
\lr{\bb{I}\baroplus\Xh\Xh^\intercal}\vect\Ah = \vect\lr{\Jh^\intercal\Xh_t\Xh^\intercal+\Xh\Xh_t^\intercal\Jh-\hat{\nabla}f(\XX)\Xh^\intercal-\Xh\hat{\nabla}f(\XX)^\intercal},
\end{align}}
% \begin{align}\label{eq:C-H-OpInf}
%     \lr{\lr{\bb{I}\baroplus\Xh\Xh^\intercal}\lr{\bb{I}+\bb{K}}+\eta\bb{I}}\vect\Bh = \vect\lr{\Jh^\intercal\Xh_t\Xh^\intercal+\Xh\Xh_t^\intercal\Jh-\hat{\nabla}\bb{f}(\XX)\Xh^\intercal-\Xh\hat{\nabla}\bb{f}(\XX)^\intercal},
% \end{align}
which satisfies Proposition~\ref{prop:newtrunc} whenever the POD basis used is drawn from the SVD of $\XX$.  Inferring $\Ah$ by way of solving \eqref{eq:C-H-OpInf} will be called the C-H-OpInf procedure, and is summarized in Algorithm~\ref{alg:C-H-OpInf}. 

% Proceeding as before, differentiation yields 
% \begin{equation*}
% dR(\Bh) = -\IP{\hat{\bb{Y}}}{\Jh\lr{d\Bh+d\Bh^\intercal}\Xh} = -\IP{\Jh^\intercal\Yh\Xh^\intercal + \Xh\Yh^\intercal\Jh}{d\Bh},
% \end{equation*}
% so that $\nabla R = \bm{0}$ becomes 
% \[\Jh^\intercal\lr{\Xh_t-\Jh\Fh(\XX)}\Xh^\intercal + \Xh\lr{\Xh_t-\Jh\Fh(\XX)}^\intercal\Jh = \Jh^\intercal\Jh\lr{\Bh+\Bh^\intercal}\Xh\Xh^\intercal + \Xh\Xh^\intercal\lr{\Bh+\Bh^\intercal}\Jh^\intercal\Jh.\]

% By a similar calculation as before, this is equivalent to solving the linear system
% \[ \lr{\Xh\Xh^\intercal\otimes\bb{I}}\vect\Ah = -\vect\lr{\Jh\Xh_t\Xh^\intercal + \hat{\bb{F}}(\XX)\XX^\intercal}. \]

% \textcolor{red}{This satisfies truncation, note $\Jh^\intercal\Jh$ must be diagonal (block diagonal?)}

\begin{remark}
    Notice that both inference procedures \eqref{eq:C-H-OpInf-real} and \eqref{eq:C-H-OpInf} preserve an approximation to the reduced Hamiltonian $\hat{H}\lr{\xh}=\rev{\frac{1}{2}}\xh^\intercal\Ah\xh + \hat{\bb{f}}\lr{\xh}$, which can be considered a perturbation of the true $\hat{H}$.  The analysis in Section~\ref{sec:theory}\rev{, particularly Theorem~\ref{thm:C-H-OpInf},} guarantees that this perturbation remains bounded throughout the range of the training data \rev{for a high enough snapshot density and large enough basis size}, although, in practice, this property seems to hold for much longer time integrations as well (see Section~\ref{sec:numerics}).
\end{remark}

\begin{algorithm}[htb]
\caption{Canonical Hamiltonian Operator Inference (C-H-OpInf)}\label{alg:C-H-OpInf}
\begin{algorithmic}[1]
\Require Snapshots $\XX\in\mathbb{R}^{N\times n_s}$ of model solution; snapshots $\nabla\bb{f}(\XX)\in\mathbb{R}^{N\times n_s}$ of nonlinear term in the gradient $\nabla H$ of the Hamiltonian; integer $n>0$ and real number $\eta>0$.
\Ensure Symmetric, reduced operator $\Ah\in\mathbb{R}^{n\times n}$ approximating the linear term in the gradient $\nabla \hat{H}$ of the reduced Hamiltonian.
\State Employ the user's preferred algorithm to build a reduced basis $\uu\in\mathbb{R}^{N\times n}$ from snapshot data.
\State Form reduced Poisson operator $\Jh = \ut\JJ\uu\in\mathbb{R}^{n\times n}$, as well as projected quantities $\Xh = \ut\XX \in\mathbb{R}^{n\times n_s}$ and $\hat{\nabla}f(\XX) = \ut\nabla f(\XX) \in\mathbb{R}^{n\times n_s}$.
\State Solve the $n^2\times n^2$ linear system \eqref{eq:C-H-OpInf} for $\rev{\Ah}\in\mathbb{R}^{n\times n}$.
% \Return $\Ah = \Bh + \Bh^\intercal$.
\end{algorithmic}
\end{algorithm}

% Notice that in the case that $H$ is quadratic, which occurs for many physical systems of interest, only snapshot data is required for this procedure, making it almost black-box in this situation. Regardless, the approximate Hamiltonian $\hat{H}$ will be preserved exactly by construction; since $\Jh$ is antisymmetric, we have 
% \[\dot{\hat{H}}(\xh) = \dot{\xh}\cdot\nabla \hat{H}(\xh) = \Jh\nabla\hat{H}(\xh)\cdot\nabla\hat{H}(\xh) = -\nabla\hat{H}(\xh)\cdot\Jh\nabla\hat{H}(\xh) = 0.\]

% \begin{remark}
% As mentioned, it is generally not true that $\Jh^\intercal\Jh=\bb{I}$. However, the OpInf procedure \red{finish this}
% \end{remark}

Before moving to the case of noncanonical systems, it is worth discussing how the C-H-OpInf procedure discussed here relates to the previous H-OpInf work in \cite{sharma2022hamiltonian} \rev{summarized in Section~\ref{subsec:prevhopinf}.  Particularly, if $\uu$ is chosen via the cotangent lift algorithm so that $\ut\JJ\uu = \JJ_n$ is the canonical symplectic matrix of dimension $n$, and $\Ah = \mathrm{diag}\lr{\Ah_{\bb{qq}},\Ah_{\bb{pp}}}$ is assumed to be block diagonal, then the algorithm presented here reduces to \cite[Algorithm 1]{sharma2022hamiltonian}.  This is because the C-H-OpInf problem \eqref{eq:C-H-OpInf} decouples into a pair of problems for each diagonal block $\Ah_{\bb{q}},\Ah_{\bb{p}}$ in $\Ah$, recovering exactly the minimization problems solved by that algorithm.}  The formulation from \cite{sharma2022hamiltonian} has the advantage of requiring the solution to two problems of size $m^2\times m^2$ (still solvable with Theorem~\ref{thm:main}) as opposed to one problem of size $2m^2\times 2m^2$, but does not allow any flexibility in the choice of basis $\uu$ and cannot accurately represent any systems with a nonseparable Hamiltonian.  Therefore, it should only be used when the problem in question is canonical and the continuous operator $\nabla H$ is diagonal in phase space.  Conversely, the inference described in Algorithm~\ref{alg:C-H-OpInf} can accommodate any reduced basis, requires the solution of only one linear system, and is applicable to any Hamiltonian system in canonical form.

\subsection{Noncanonical Hamiltonian Systems}
\rev{A primary} advantage of the OpInf technique inspired by Theorem~\ref{thm:main} is that it extends to Hamiltonian systems in noncanonical form.  To see this, suppose snapshots of a potentially unknown Hamiltonian system are collected in an $(N\times n_s)$-matrix $\XX$, and that a candidate Hamiltonian function $H$ has been identified.  This may occur if, for example, a conserved quantity has been identified but the corresponding Hamiltonian structure remains unknown. Then, an analytic expression for $\nabla H$ can be obtained, and hence it is possible to compute the matrix $\nabla H(\XX)\in\mathbb{R}^{N\times n_s}$ of gradients at the snapshot data $\XX$, as well as a finite difference approximation $\XX_t \approx \dot{\XX}$.  As before, this enables the construction of a POD basis $\uu\in\mathbb{R}^{N\times n}$ via the SVD of the mean-centered data matrix $\bb{Y}=\bb{X}-\bb{X}_0$, where $\XX_0$ denotes the matrix each column of which is the initial state $\xx_0$.  Writing the Galerkin approximation $\xt \approx \xx_0+\uu\xh$ again yields the prototypical H-ROM (see Section~\ref{subsec:hampod}) $\dot{\xx} = \Lh\nabla\hat{H}$ where $\Lh = \ut\LL\uu$.  When $\LL$ is inaccessible, this suggests a similar inference procedure based on Theorem~\ref{thm:main} which preserves the antisymmetry necessary for Hamiltonian preservation.  Particularly, it is possible to form the $n\times n_s$ reduced quantities
\begin{align*}
    \hat{\XX} = \ut\XX, \qquad \hat{\XX}_t = \ut\XX_t, \qquad \hat{\nabla}H(\XX) = \ut\nabla H(\XX),
\end{align*}
and solve the optimization problem \rev{
\[\argmin_{\Lh\in\mathbb{R}^{n\times n}}\nn{\hat{\XX}_t - \Lh\hat{\nabla}H(\XX)}^2, \quad\mathrm{s.t.}\quad \Lh^\intercal=-\Lh,\] }
% \[\Mh=\argmin_{\bb{M}\in\mathbb{R}^{n\times n}}\lr{\nn{\hat{\XX}_t - \lr{\bb{M}-\bb{M}^\intercal}\hat{\nabla}H(\XX)}^2 + \eta\nn{\bb{M}}^2},\]
which is a straightforward least-squares inference for the antisymmetric $\Lh$.  As shown in Theorem~\ref{thm:main}, this is equivalent to solving the $n^2\times n^2$ linear system 
\rev{\begin{align}\label{eq:NC-H-OpInf}
    \lr{\bb{I}\baroplus\hat{\nabla} H(\XX)\hat{\nabla} H(\XX)^\intercal}\vect\Lh = \vect\lr{\Xh_t\hat{\nabla}H(\XX)^\intercal - \hat{\nabla}H(\XX)\Xh_t^\intercal}.
\end{align}}
% \begin{align}\label{eq:NC-H-OpInf}
%     \lr{\lr{\bb{I} \baroplus \hat{\nabla} H(\XX)\hat{\nabla} H(\XX)^\intercal}\lr{\bb{I}-\bb{K}}+\eta\bb{I}}\vect\Mh = \vect\lr{\Xh_t\hat{\nabla}H(\XX)^\intercal - \hat{\nabla}H(\XX)\Xh_t^\intercal}.
% \end{align}
% \IKTcomment{I am just curious, presumably $\lr{\bb{I} \baroplus \hat{\nabla} H(\XX)\hat{\nabla} H(\XX)^\intercal}\lr{\bb{I}-\bb{K}} \rl$ is really poorly conditioned
% in the cases that regularization is needed.  Did you observe this numerically?}  
Inferring $\Lh$ based on solving \eqref{eq:NC-H-OpInf} will be called the NC-H-OpInf method, and is summarized in Algorithm~\ref{alg:NC-H-OpInf}.  While this inference similarly does not satisfy the hypotheses of Proposition~\ref{prop:newtrunc}, it is interesting to note that ``one shot'' computation of $\Lh$ using Algorithm~\ref{alg:NC-H-OpInf} \rev{occasionally} works quite well in practice when the basis $\uu$ is chosen from the SVD of $\bb{Y}$ (see Section~\ref{sec:numerics}).

% Doing this yields a reduced operator $\Lh$ which can be used along with an appropriate time-stepping algorithm to solve the reduced-order model $\dot{\xh}=\Lh\nabla H\lr{\xx_0+\uu\xh}$.  

\begin{algorithm}[htb]
\caption{Noncanonical Hamiltonian Operator Inference (NC-H-OpInf)}\label{alg:NC-H-OpInf}
\begin{algorithmic}[1]
\Require Snapshots $\XX\in\mathbb{R}^{N\times n_s}$ of model solution; snapshots $\nabla H(\XX)\in\mathbb{R}^{N\times n_s}$ of the gradient $\nabla H$ of the Hamiltonian; integer $n>0$ and real number $\eta>0$.
\Ensure Antisymmetric, reduced operator $\Lh\in\mathbb{R}^{n\times n}$ approximating the Poisson operator governing the H-ROM $\dot{\xh}=\Lh\nabla\hat{H}$.
\State Employ the user's preferred algorithm to build a (mean-centered) POD basis $\uu\in\mathbb{R}^{N\times n}$ from snapshot data.
\State Form projected quantities $\Xh = \ut\XX \in\mathbb{R}^{n\times n_s}$ and $\hat{\nabla}H(\XX) = \ut\nabla H(\XX) \in\mathbb{R}^{n\times n_s}$.
\State Solve the $n^2\times n^2$ linear system \eqref{eq:NC-H-OpInf} for $\rev{\Lh}\in\mathbb{R}^{n\times n}$.
% \Return $\Lh = \Mh - \Mh^\intercal$.
\end{algorithmic}
\end{algorithm}

\begin{remark}
    Note that NC-H-OpInf can be used (along with a symplectic time integrator) to obtain dynamics which preserve any quantity $H$, regardless of whether or not it corresponds to a true Hamiltonian structure.  In this way, it can be considered a gray-box method requiring only snapshots and a desired conserved quantity.
\end{remark}

% \begin{remark}
% Notice that truncation is generally not optimal here, since $\hat{\nabla}H(\XX)\hat{\nabla}H(\XX)^\intercal$ is generally not diagonal. 
% \end{remark}

% However, performing OpInf in this way again preserves the conclusion of Proposition~\ref{prop:reuse}.

% \subsection{Newest idea}
% Suppose we only know snapshots in some coordinates $\xx$ which we believe to be Hamiltonian.  We can attempt to infer the product of operators $\Lh\Ah$ in the following way.  First, notice that $\Xh_t = \Lh\Ah\Xh$ implies 
% \[ \Xh^\intercal\XX_t\pm\Xh_t^\intercal\Xh = \Xh^\intercal\lr{\Lh\Ah\mp\Ah^\intercal\Lh}\Xh, \qquad \lr{\Lh\Ah\mp\Ah^\intercal\Lh}^\intercal = \pm\lr{\Lh\Ah\mp\Ah^\intercal\Lh}, \]
% so that $\Lh\Ah\mp\Ah^\intercal\Lh$ is symmetric/antisymmetric.  Therefore, letting $\hat{\bb{C}}_\pm = \Xh^\intercal\Xh_t\pm\Xh_t^\intercal\Xh$ we can solve 
% \[\min_{\Mh}\nn{\hat{\bb{C}}_\pm - \Xh^\intercal\lr{\Mh\pm\Mh^\intercal}\Xh}^2,\]
% yielding a pair of operators $\Mh_+^\intercal = \Mh$ and $\Mh_-^\intercal = -\Mh_-$ which converge to $\Lh\Ah\mp\Ah^\intercal\Lh$.  From these, we can then recover the desired operator $\Lh\Ah = (1/2)\lr{\Mh_+ + \Mh_-}$. \red{this is pretty much black-box. the drawback is that everything is linear. Tested this: exact same performance as generic OpInf.  Maybe I should have seen this coming......}

\section{Analysis}\label{sec:theory}
Now that the C-H-OpInf and NC-H-OpInf procedures have been described, it is important to validate that the inferred operators approximate their intrusive counterparts in an appropriate sense.  To accomplish this, the following mild assumptions are needed.

\begin{assumption}\label{asmp:POD}
    The span of the POD basis $\uu\in\mathbb{R}^{N\times n}$ tends to $\mathbb{R}^N$ as $n\to N$, i.e., for any $\xx\in\mathbb{R}^N$,
    \[\lim_{n\to N}\, \nn{\bb{P}^\perp\xx} = 0,\]
    where $\bb{P}^\perp:=\bb{I}-\uu\ut$.
\end{assumption}

\begin{assumption}\label{asmp:timederiv}
    The approximate time derivatives $\xx_t$ converge to the true derivatives $\dot{\xx}$ as the time step $\Delta t\to 0$, i.e.,
    \[\lim_{\Delta t\to 0} \max_i\, \nn{\xx_{t}(t_i) -\dot{\xx}(t_i)}=0.\]
\end{assumption}

\begin{assumption}\label{asmp:fullrank}
    The snapshot matrices $\XX, \nabla H(\XX) \in \mathbb{R}^{N\times n_s}$ have maximal rank.
\end{assumption}

This allows for the following result regarding the convergence of NC-H-OpInf.  
% \textcolor{red}{Please check these arguments, at least at a high level.  They are not difficult, but a second pair of eyes would be good in case there are any obvious mistakes.} \IKTcomment{I checked the proofs.  You are using some clever tricks (adding/subtracting matrices) and common inequalities like the triangle inequality, etc.  The logic looks sound to me.  Takes me back to days when I used to develop error estimates using
% similar tricks/inequalities.  I added a couple notes on things that could be clarified regarding theorem assumptions.  Another thing that I thought of is that 
% ROMs don't always converge as $\Delta t \to 0$ -- sometimes there is an ``optimal" intermediate $\Delta t$ for which a ROM has 
% the lowest error; but I think the proofs make sense as is.}

\begin{theorem}\label{thm:NC-H-OpInf}
    Under Assumptions~\ref{asmp:POD}, \ref{asmp:timederiv}, and \ref{asmp:fullrank}, the inferred operator $\Lh$ from the NC-H-OpInf procedure in Algorithm~\ref{alg:NC-H-OpInf} converges to the intrusive operator $\bar{\LL}=\ut\LL\uu$ as $\Delta t\to 0$ and $n\to N$.  
\end{theorem}

\begin{proof}
First, notice that
\begin{align*}
    \nn{\Xh_t-\Lh\hat{\nabla}H(\XX)} &= \nn{\lr{\Xh_t-\dot{\Xh}} + \lr{\dot{\Xh}-\bar{\LL}\hat{\nabla}H(\XX)} + \lr{\bar{\LL}-\Lh}\hat{\nabla} H(\XX)} \\
    &= \nn{\ut\lr{\XX_t-\dot{\XX}} + \ut\lr{\dot{\XX}-\LL\nabla H(\XX)} + \ut\LL\bb{P}^\perp\nabla H(\XX) + \lr{\bar{\LL}-\Lh}\hat{\nabla} H(\XX)} \\
    &= \nn{\ut\lr{\XX_t-\dot{\XX}} + \ut\LL\bb{P}^\perp\nabla H(\XX) + \lr{\bar{\LL}-\Lh}\hat{\nabla} H(\XX)} \\
    &\leq \nn{\uu}\lr{\nn{\XX_t-\dot{\XX}}+\nn{\LL}\nn{\bb{P}^\perp\nabla H(\XX)}} + \nn{\bar{\LL}-\Lh}\nn{\hat{\nabla} H(\XX)}.
\end{align*}
Therefore, for each $n\leq N$,
\begin{align*}
    \min_{\Lh}\nn{\Xh_t-\Lh\hat{\nabla}H(\XX)} &\leq \min_{\Lh} \left[\nn{\uu}\lr{\nn{\XX_t-\dot{\XX}}+\nn{\LL}\nn{\bb{P}^\perp\nabla H(\XX)}} + \nn{\bar{\LL}-\Lh}\nn{\nabla H(\XX)}\right] \\
    &= \nn{\uu}\lr{\nn{\XX_t-\dot{\XX}}+\nn{\LL}\nn{\bb{P}^\perp\nabla H(\XX)}}.
\end{align*}
By Assumptions \ref{asmp:POD} and \ref{asmp:timederiv}, for any $\varepsilon>0$ there exists an $n' < N$ and $\Delta t' > 0$ such that 
\[\nn{\ut\lr{\XX_t-\dot{\XX}} + \ut\LL\bb{P}^\perp\nabla H(\XX)} \leq \nn{\uu}\lr{\nn{\XX_t-\dot{\XX}}+\nn{\LL}\nn{\bb{P}^\perp\nabla H(\XX)}} < \frac{\varepsilon}{2}.\]
Therefore, for $n\geq n'$ and $\Delta t \leq \Delta t'$ it follows from an elementary calculation that 
\[ \min_{\Lh}\nn{\lr{\bar{\LL}-\Lh}\hat{\nabla} H(\XX)} < \varepsilon,\]
from which it can be concluded that $\Lh\to\bar{\LL}$, since $\hat{\nabla} H(\XX)$ has maximal rank.  
\end{proof}

A similar result holds for C-H-OpInf provided a cotangent lift basis $\uu$ is used.

\begin{theorem}\label{thm:C-H-OpInf}
    Under Assumptions~\ref{asmp:POD}, \ref{asmp:timederiv}, \ref{asmp:fullrank}, and using a cotangent lift POD basis $\uu$, the inferred operator $\Ah$ from the C-H-OpInf procedure in Algorithm~\ref{alg:C-H-OpInf} converges to the intrusive operator $\bar{\Aa}=\ut\Aa\uu$ as $\Delta t\to 0$ and $n\to N$.
\end{theorem}

\begin{remark}
    Note that the assumption of a cotangent lift basis in Theorem~\ref{thm:C-H-OpInf} can be dropped provided \eqref{eq:C-H-OpInf-real} is solved instead of \eqref{eq:C-H-OpInf} in the C-H-OpInf Algorithm~\ref{alg:C-H-OpInf}.
\end{remark}

\begin{proof}
First, notice that 
\begin{align*}
    \dot{\XX}-\bb{J}\uu\ut&\lr{\Aa\uu\ut\XX+\nabla f(\XX)} \\
    &= \lr{\dot{\XX}-\bb{J}\lr{\Aa\XX+\nabla f(\XX)}} + \bb{J}\lr{\bb{P}^\perp\bb{AX} + \bb{P}^\perp\nabla f(\XX)+ \uu\ut\Aa\bb{P}^\perp\XX} \\
    &= \bb{J}\lr{\bb{P}^\perp\bb{AX} + \bb{P}^\perp\nabla f(\XX)+ \uu\ut\Aa\bb{P}^\perp\XX},
\end{align*}
since $\JJ\nabla H(\XX) = \bb{J}\lr{\Aa\XX+\nabla f(\XX)}$.  Therefore, it follows as before that for every $n\leq N$,
\begin{align*}
    &\nn{\Xh_t-\Jh\lr{\Ah\Xh+\hat{\nabla}f(\XX)}} = \nn{\lr{\Xh_t -\dot{\Xh}} + \lr{\dot{\Xh}-\Jh\lr{\bar{\Aa}\Xh+\hat{\nabla}f(\XX)}} + \Jh\lr{\bar{\Aa}-\Ah}\Xh} \\
    &\qquad= \nn{\ut\lr{\XX_t-\dot{\XX}} + \ut\bb{J}\lr{\bb{P}^\perp\bb{AX} + \bb{P}^\perp\nabla f(\XX)+ \uu\ut\Aa\bb{P}^\perp\XX} + \Jh\lr{\bar{\Aa}-\Ah}\Xh},
\end{align*}
Now, for any $\varepsilon>0$ we can choose $n'<N$ and $\Delta t'>0$ so that 
\begin{align*}
    &\nn{\ut\lr{\XX_t-\dot{\XX}} + \ut\JJ\lr{\bb{P}^\perp\bb{AX} + \bb{P}^\perp\nabla f(\XX) + \uu\ut\Aa\bb{P}^\perp\XX}} \\
    &\leq \nn{\uu}\lr{\nn{\XX_t-\dot{\XX}} + \nn{\bb{P}^\perp} \nn{\JJ}\lr{\nn{\Aa}\nn{\XX}+\nn{\nabla f(\XX)}+\nn{\uu}^2\nn{\Aa}\nn{\XX}}} < \frac{\varepsilon}{2},
\end{align*}
and therefore we have
\[\min_{\Ah}\nn{\Jh\lr{\bar{\Aa}-\Ah}\Xh} < \varepsilon,\]
provided $n\geq n'$ and $\Delta t<\Delta t'$.  Hence, $\Ah\to\bar{\Aa}$ as desired , since $\Jh,\Xh$ have maximal rank.
\end{proof}

\begin{remark}
While useful, the results of this section \rev{only hold} in the ``infinite data limit'', and so cannot guarantee good performance of the OpInf methods (and projection-based ROMs in general) in all cases of practical interest, particularly in the predictive regime.  It is an ongoing effort to develop rigorous estimates which are more valuable in the presence of partial or limited data. 
\end{remark}

\section{Numerical Examples}\label{sec:numerics}

% \IKTcomment{You should probably say somewhere at the beginning of this section that the online computational costs of all the models evaluated are approximately the same for a particular basis size $n$.}

Here, numerical results are reported on several benchmark problems from hydrodynamics and linear elasticity, including a linear wave equation, \rev{a manufactured test case which has a non-separable canonical Hamiltonian form,} the Korteweg-de Vries equation, the Benjamin-Bona-Mahoney equation, and a 3D linear elastic clamped plate problem undergoing high-frequency oscillations.  The primary error metrics used for comparison will be relative $\ell_2$ error in the state approximation,
\[ R\ell_2\lr{\XX,\tilde{\XX}} = \frac{\nn{\XX-\tilde{\XX}}_2}{\nn{\XX}_2}, \]
as well as \rev{signed} error in the Hamiltonian (or other conserved quantity) approximation \rev{$H(\xx(t))-H_0$} where $H_0 = H(\xx(0))$.  When speaking about the properties of POD bases, it will also be useful to evaluate the snapshot energy, computed for a given rank $r$ snapshot matrix $\XX$ with singular values $\{\sigma_i\}_{i=1}^r$ and POD basis size \rev{$n\leq r$} as
\[E_s\lr{\XX,n} = \frac{\sum_{k=1}^n \sigma_k}{\sum_{k=1}^r \sigma_k}.\]
Note that, when appropriate, both uncentered ($\xt=\uu\xh$) and mean-centered ($\xt = \xx_0+\uu\xh$) Galerkin projections will be considered.  This will be denoted by the letters ``MC'' in the figures below.  Of course, mean-centering requires a POD of the centered snapshot matrix discussed in Section~\ref{subsec:hampod}, and is infeasible for a general OpInf method.  On the other hand, NC-H-OpInf is amenable to this technique, since the inferred operator $\Lh$ does not interface directly with the approximate solution $\xt$.

When evaluating the performance of the H-OpInf methods in Section~\ref{sec:hopinf}, comparisons are drawn with the standard intrusive Galerkin ROM (G-ROM) and Hamiltonian ROM (H-ROM) discussed previously, as well as the standard Galerkin OpInf (G-OpInf) when appropriate.  Reproductive as well as predictive problems are considered, encompassing both prediction in time as well as prediction across parameter space.  Note that all ROMs considered are equally efficient online; since the chosen examples have polynomial nonlinearities, their resulting ROMs do not depend on the full-order state space $N$, instead scaling only with the reduced basis size $n$.

\rev{\begin{remark}
    On canonical Hamiltonian examples, the NC-H-OpInf algorithm will infer only $\hat{\JJ} \approx \ut\JJ\uu$, which is already known.  Since it is instructive to see that the NC-H-OpInf ROM behaves appropriately on these examples, comparisons including it are presented for these cases, although it should be noted that this is not the intended purpose of NC-H-OpInf.
\end{remark}}

\subsection{Linear Wave Equation} \label{sec:lin_wave}
First, consider the one-dimensional linear wave equation with constant speed 
$c$,
\begin{equation}\label{eq:lin_wave_eqn}
\begin{split}
\vp_{tt} &= c^2 \vp_{ss}, \qquad 0\leq s\leq l, \\
\vp(0) &= h(y(s)), \qquad \vp_t(0) = 0,
\end{split}
\end{equation}
where the boundary conditions are periodic and the (parameterized) initial condition is a cubic spline defined by
\[ h(y) = \begin{cases} 1 - \frac{3}{2}y^2 + \frac{3}{4}y^3 & 0 \leq y \leq 1, \\ \frac{1}{4}\lr{2-y}^3 & 1<y\leq 2, \\ 0 & y > 2,
\end{cases} \qquad y(s,\alpha) = \alpha\nn{s-\frac{1}{2}}. \]
Letting $\xx = \begin{pmatrix}q & p\end{pmatrix}^\intercal \in\mathbb{R}^2$ where $q=\vp$ and $p=\vp_t$, this problem is readily recast in the canonical Hamiltonian form
\[ \dot{\xx} = \JJ\nabla H(\xx) = \begin{pmatrix}0 & 1 \\ -1 & 0\end{pmatrix}\begin{pmatrix}-c^2\partial_{ss} & 0 \\ 0 & 1 \end{pmatrix}\begin{pmatrix}q \\ p\end{pmatrix},\]
where the Hamiltonian functional is given by
\[ H\lr{\xx} = \frac{1}{2}\int_0^l \lr{p^2 + c^2q_s^2}\, ds, \]
and it follows quickly from differentiation that $H_q = -c^2q_{ss}, H_p = p$. As discussed in Section~\ref{subsec:linrom}, semi-discretizing in $\xx$ and applying AVF integration to this system yields the implicit midpoint rule
\[ \frac{\xx^{k+1}-\xx^{k}}{\Delta t} = \JJ\bb{A}\lr{\frac{\xx^{k+1}+\xx^{k}}{2}} = \begin{pmatrix}0 & \bb{I} \\ -\bb{I} & 0\end{pmatrix}\begin{pmatrix}-c^2\bb{D}_2 \\ 0 & \bb{I}\end{pmatrix}\lr{\frac{\xx^{k+1}+\xx^k}{2}},\]
where $\xx = \begin{pmatrix} \qq & \pp \end{pmatrix}^\intercal$ has been overloaded, $\bb{D}_2$ denotes the circulant matrix which results from using a three-point stencil finite difference method to discretize the 1-D Laplace operator, and the discrete Hamiltonian (also overloaded as $H$) is given by 
\[ H(\xx) = \frac{1}{2}\sum_{i=1}^{N/2}\lr{ p_i^2 + \frac{\lr{q_{i+1}-q_i}^2 + \lr{q_i-q_{i-1}}^2}{4\Delta x^2}}.\]
Note that the AVF method will preserve this discrete Hamiltonian exactly by construction.  Some snapshots of this solution for different values of $\alpha$ are displayed in Figure~\ref{fig:ToyWaveFOM}.

\begin{figure}[htb]
    \centering
    \includegraphics[width=\textwidth]{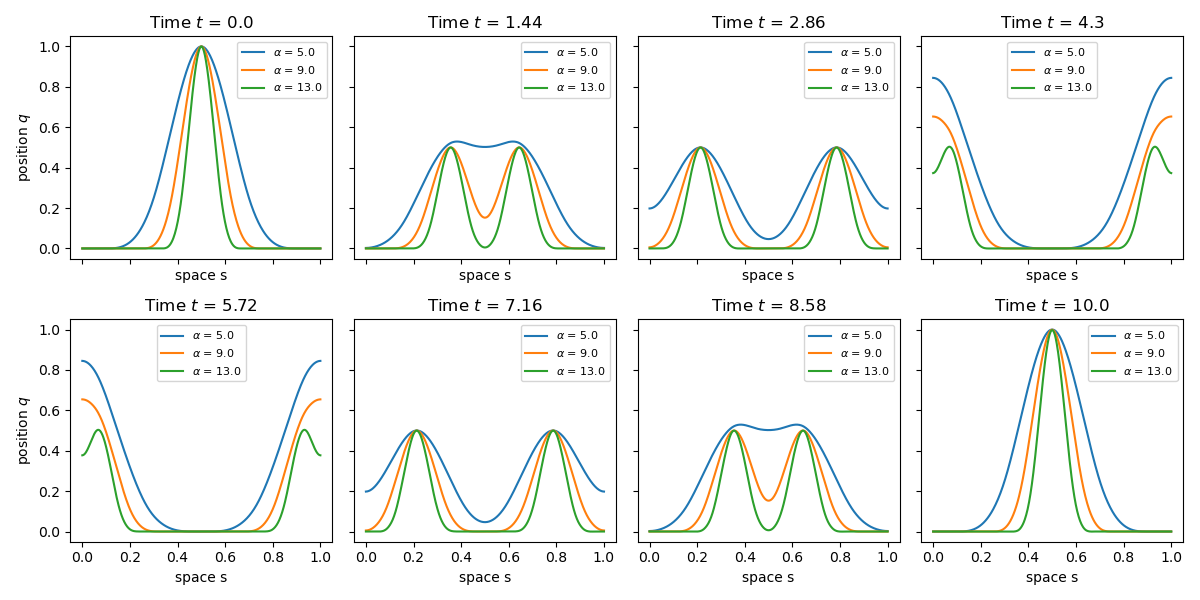}
    \caption{Solution snapshots from the linear wave example for different values of the parameter $\alpha$.}
    \label{fig:ToyWaveFOM}
\end{figure}

% \IKTcomment{Why not call it $H_N$ and avoid the overloading?}

% \IKTcomment{Aren't you missing the FOM in Figure 3?  Seems you'd need that for completeness, to have something to compare the ROMs to.  Also, this figure really needs axis labels as well as a colorbar!}

To evaluate the performance of the ROMs discussed thusfar, two experiments will be conducted: one testing prediction in time, and one testing prediction in parameter space.  For each, the wave speed is fixed to $c=0.1$, the length to $l=1$, and the spatial domain is divided into $M=500$ equally sized intervals (yielding a state vector $\xx$ of dimension $N=2M=1000$).  
% \rev{Note that G-OpInf is applied with a Tikhonov regularization of $10^{-6}$ as discussed in Remark~\ref{rem:tikhonov}.}

% Additionally, the regularization parameters $\eta$ are chosen to be $10^{-6}, 10^{-12}, 10^{-12}$ for G-OpInf, NC-H-OpInf, and C-H-OpInf, respectively.

% As mentioned in Section~\ref{sec:pre....}, this is a symplectic integration scheme which will exactly preserve the discrete energy of the system.

\subsubsection{Reproductive versus Predictive Dynamics}

The first goal is to compare the C-H-Opinf and NC-H-Opinf ROM methods discussed in Section~\ref{sec:hopinf} to their intrusive counterparts when predicting trajectories outside the temporal range of their training data.  For this, a total of 501 snapshots of the FOM solution with initial condition parameter $\alpha=5$ are uniformly collected on the time interval $[0,T]$ where $T=10$.  These data are used to train three POD bases: one constructed in the ``ordinary way'' by forming the SVD of a data matrix of size $N\times n_t$ containing snapshots of $\xx$, another constructed using the cotangent lift algorithm described in Section~\ref{subsec:hampod}, and the final constructed block-wise using the SVD of snapshot data for position $\qq$ and momentum $\pp$ separately (also described in Section~\ref{subsec:hampod}).  The snapshot energies and projection errors associated to these bases are shown in Figure~\ref{fig:ToyWavePOD}. It is evident that all bases are capable of capturing roughly $99\%$ of the snapshot energy with only $n=15$ modes, despite exhibiting a slowly decaying projection error characteristic of hyperbolic problems.

\begin{figure}[htb]
    \centering
    \includegraphics[width=\textwidth]{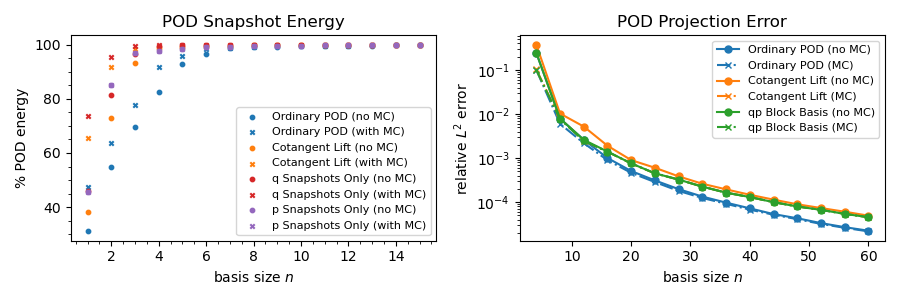}
    \caption{POD snapshot energies (left) and  projection errors (right) corresponding to the bases used in the nonparametric $(\alpha=5)$ linear wave example.  ``MC'' indicates mean-centering of the snapshots was performed.}
    \label{fig:ToyWavePOD}
\end{figure}

Figure~\ref{fig:ToyWaveRepr} plots the relative ROM errors as a function of basis size in the case where the ROMs are integrated only in the range of the training data, i.e. $t\in[0,10]$.  Notice that both the intrusive G-ROM and the G-OpInf ROM are less accurate than their Hamiltonian counterparts, and that the G-OpInf ROM is somewhat unstable with the addition of basis modes.  It is further interesting to observe the differences in performance between the ROM algorithms as the underlying basis is changed.  Particularly, both the cotangent lift and $(q,p)$-block basis lead to lower relative errors than ordinary POD, although ordinary POD has the significant (empirical) advantage of stability under OpInf truncation.  Indeed, in the case of the ordinary POD basis, all operators used in the OpInf ROMs were computed in ``one shot'' via truncation from the operators learned at the largest basis size.  While this is not guaranteed to be optimal according to Proposition~\ref{prop:newtrunc}, it is interesting to note that this resulted in almost no degradation of performance.  This contrasts highly with the case of the cotangent lift and $(q,p)$ block bases, for which OpInf truncation led to unusable results (not pictured). 

\begin{figure}[htb]
    \centering
    \includegraphics[width=\textwidth]{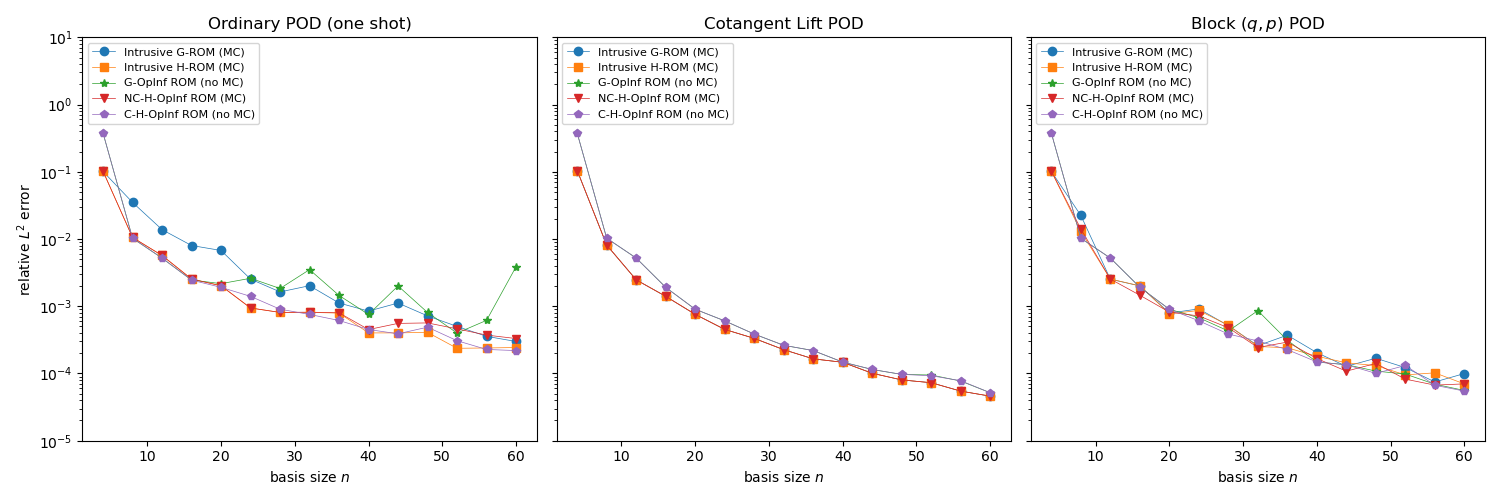}
    \caption{Relative state errors as a function of basis modes for the ROMs in the linear wave example (reproductive case $T=10$).  ``MC'' indicates the use of a mean-centered reconstruction.}
    \label{fig:ToyWaveRepr}
\end{figure}

\begin{figure}[htb]
    \centering
    \includegraphics[width=\textwidth]{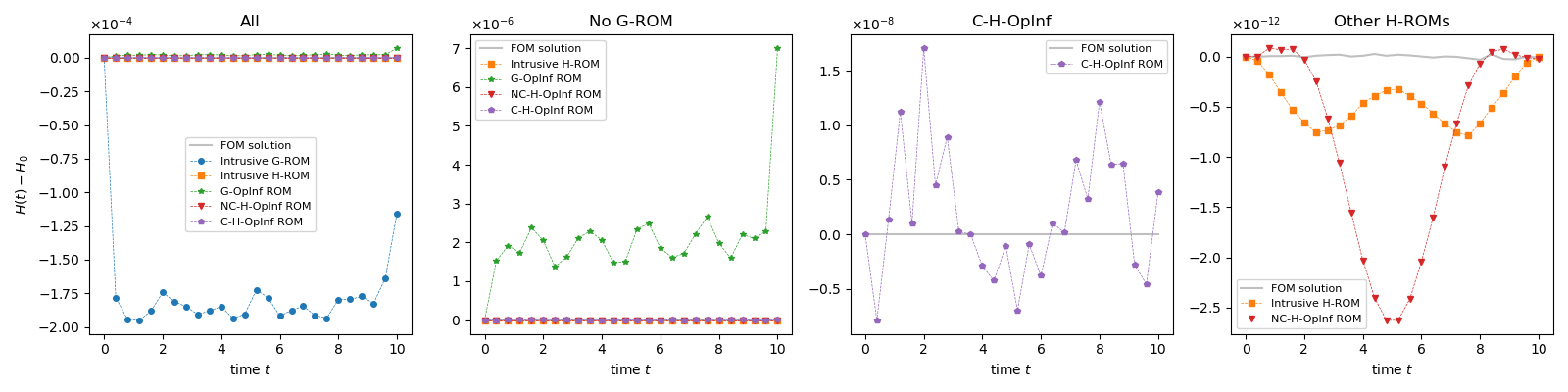}
    \caption{ROM energy errors for the linear wave example in the reproductive case ($T=10$) when using a block $(\qq,\pp)$ POD basis with mean-centering (where applicable) and with $n=16$ modes.}
    \label{fig:ReprWaveEnergy}
\end{figure}

To show the effect of each ROM on energy preservation, Figure~\ref{fig:ReprWaveEnergy} uses the block $(q,p)$ basis case with $n=16$ modes to show the change in the Hamiltonian $H$ over time.  From this, it is seen that the intrusive H-ROM and NC-H-OpInf ROM conserve energy exactly, while the C-H-OpInf ROM conserves energy to order $10^{-8}$.  Of course, this is a consequence of the fact that the matrix $\Ah$ learned by C-H-OpInf represents only an approximation to the gradient of the true reduced Hamiltonian $\hat{H}$.  On the other hand, note that C-H-OpInf still conserves $H$ much better than G-OpInf or the intrusive G-ROM, and is guaranteed to exactly preserve the approximate reduced energy $\tilde{H} = \frac{1}{2}\xh^\intercal\Ah\xh$ (not pictured), which follows since the matrix $\Jh = \ut\JJ\uu$ is skew-symmetric.  It is further remarkable that the conservation properties of the H-ROMs displayed in these plots do not depend on the  basis construction mechanism or the number of basis modes, $n$. 
% \IKTcomment{How come the G-H-OpInf approach has no MC and the others have MC in  Figure~\ref{fig:ToyWaveRepr}?  Same goes for some of the later figures.  Might be worth commenting preemptively, since this is something reviewers may ask about too.}

\begin{figure}[htb]
    \centering
    \includegraphics[width=\textwidth]{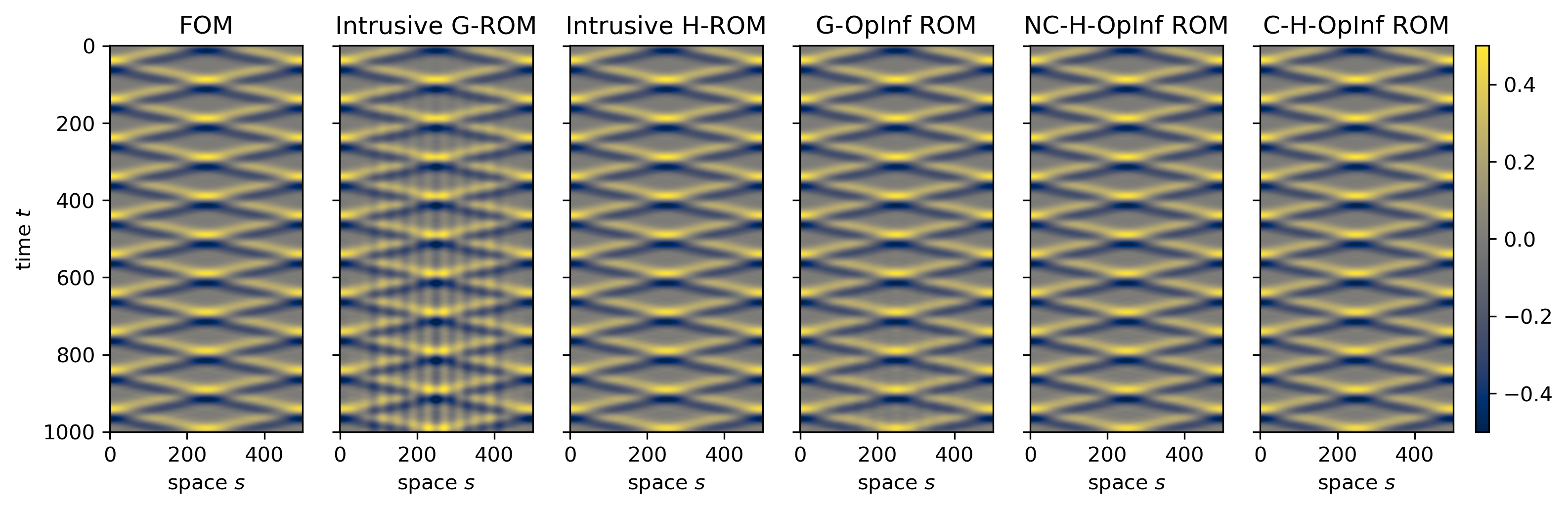}
    \caption{Plots of the FOM and ROM solutions to the linear wave equation in the predictive case ($T=100$) when using a standard POD basis with $n=16$ modes.  Note that mean-centered reconstructions were used for all but the G-OpInf and C-H-OpInf ROMs.}
    \label{fig:ToyWaveImshow}
\end{figure}

% \IKTcomment{Isn't the solution just going to repeat itself in time?  That would make prediction in time somewhat trivial...  and leaves me wondering why all the ROMs don't perform well.}

Moving beyond the reproductive case, it is useful to see what happens when the ROMs are tested on an interval of integration which is much larger.  Figure~\ref{fig:ToyWavePred} plots the relative ROM errors as a function of basis modes when the ROMs are tested over an interval of $[0,T]$ with $T=100$, which is ten times the interval of training.  Here the instabilities in the G-OpInf ROM are made readily apparent, as certain numbers of modes lead extreme blow-ups regardless of the underlying basis construction.  It is interesting to note that the intrusive G-ROM also exhibits similar blow-up in the cases (not pictured here) when the POD basis is constructed with ordinary POD and no mean-centering is applied. Conversely, both the intrusive and OpInf H-ROMs exhibit a steady and predictable decrease in error with the addition of basis modes.  Note that a comparative visualization of the FOM and ROM solutions is shown in Figure~\ref{fig:ToyWaveImshow}, which  plots each solution when an ordinary POD basis is used with $n=16$ modes.

\begin{figure}[htb]
    \centering
    \includegraphics[width=\textwidth]{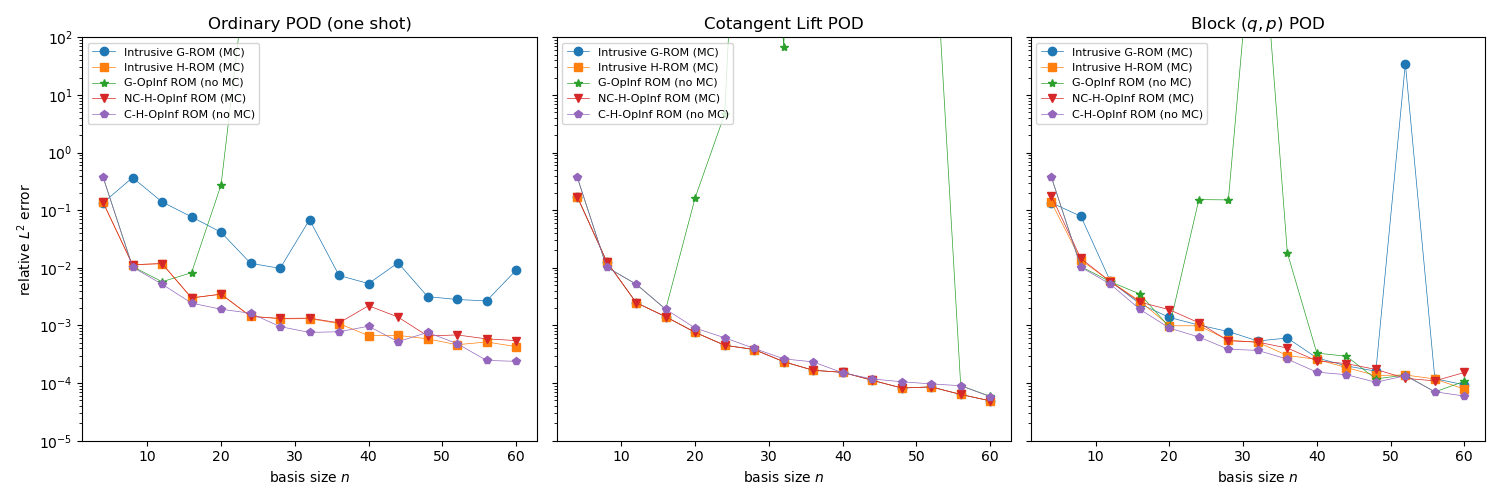}
    \caption{Relative state errors as a function of basis modes for the ROMs in the linear wave equation example (predictive case $T=100$). ``MC'' indicates the use of a mean-centered reconstruction.}
    \label{fig:ToyWavePred}
\end{figure}

Figure~\ref{fig:PredWaveEnergy} displays the variation in the value of the Hamiltonian over this larger integration range when the ROMs are computed using a $(q,p)$ block basis of $n=16$ modes.  As before, the intrusive G-ROM and OpInf G-ROM are not sufficiently conservative, which has consequences for their accuracy and stability.  Conversely, the NC-H-OpInf ROM conserves $H$ on the same order as the intrusive H-ROM, and the C-H-OpInf ROM conserves $H$ to order $10^{-8}$, exhibiting similar performance to integration over the training interval.

\begin{figure}[htb]
    \centering
    \includegraphics[width=\textwidth]{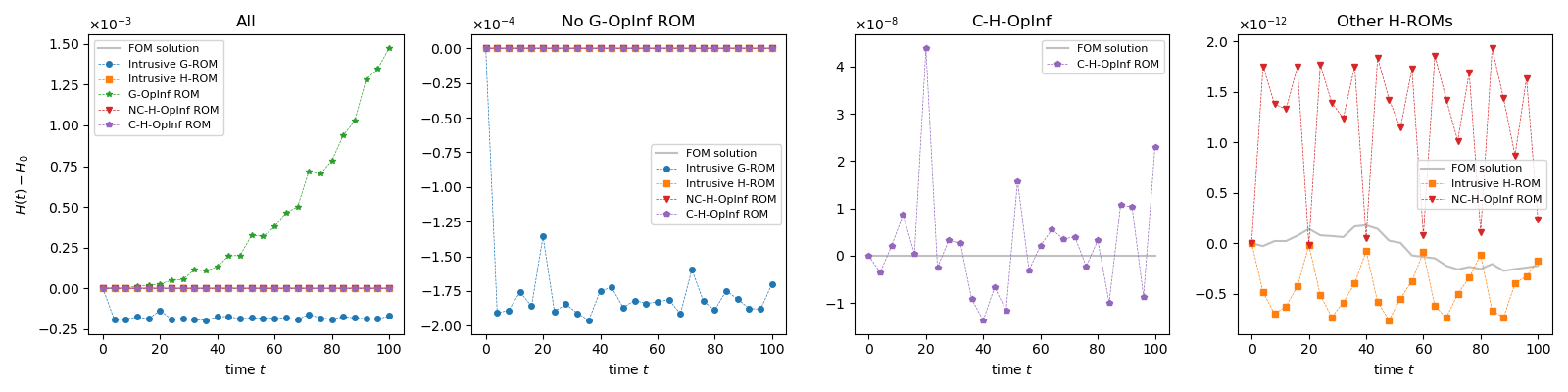}
    \caption{ROM energy errors for the linear wave equation example in the predictive case ($T=100$) when using a block $(\qq,\pp)$ POD basis with mean-centering (where applicable) and with $n=16$ modes.}
    \label{fig:PredWaveEnergy}
\end{figure}

\subsubsection{Parametric Case}

In addition to prediction in time, it is also useful to consider applying ROMs for the prediction of solutions across the parameter space spanned by $\alpha\in\mathbb{R}$, which controls the initial state of the wave (c.f. Figure~\ref{fig:ParaWaveSolns}).  To that end, the next experiment examines how well the present ROM methods are able to predict solutions with variable initial conditions.  To accomplish this, eleven uniformly distributed parameters $\alpha\in[5,15]$ are chosen for training, and $501$ snapshots of the FOM solution in the range $[0,10]$ are collected using each parameter instance.  These data are then concatenated to form the snapshot matrix which is used to train the POD decompositions.  The snapshot energies and projection errors associated to this procedure are shown in Figure~\ref{fig:ToyWavePODparam}, where it is evident that the inclusion of multiple solution trajectories slows down both the increase in the snapshot energy and the decay of the projection error.

% \IKTcomment{It might be interesting to comment to what extent the solution changes with the variation of the ICs.  Does it change trivially or significantly?}

\begin{figure}[htb]
    \centering
    \includegraphics[width=\textwidth]{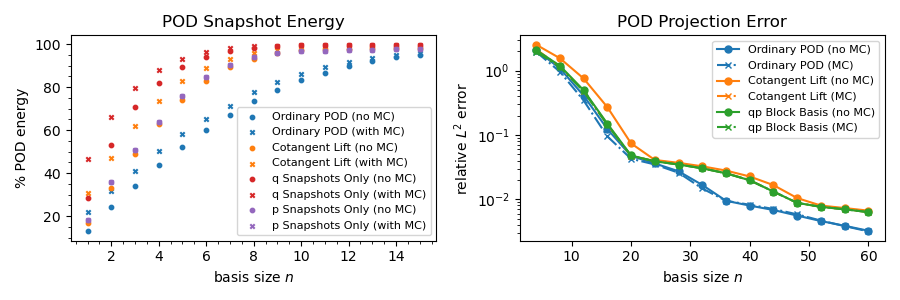}
    \caption{POD snapshot energies (left) and  projection errors (right) corresponding to the bases used in the parametric linear wave example. ``MC'' indicates that mean-centering of the snapshots was performed.}
    \label{fig:ToyWavePODparam}
\end{figure}

For testing, six uniformly distributed parameters $\alpha \in [5.5,14.5]$ are chosen (note that these are disjoint from the training parameters), and snapshot data of each solution in the temporal range $t \in [0,100]$ is collected for comparison with the ROM integration.  The ROMs are then tested over this interval beginning from each unseen initial condition, and the average relative error over all test snapshots is reported.

% \textcolor{red}{Comment from Irina: it's interesting that the NC-H-OpInf ROM doesn't work well for larger basis sizes with the qp basis.  Is this worth commenting on?  Curiously the approach preserves energy conservation despite being very inaccurate.}

\begin{figure}[htb]
    \centering
    \includegraphics[width=\textwidth]{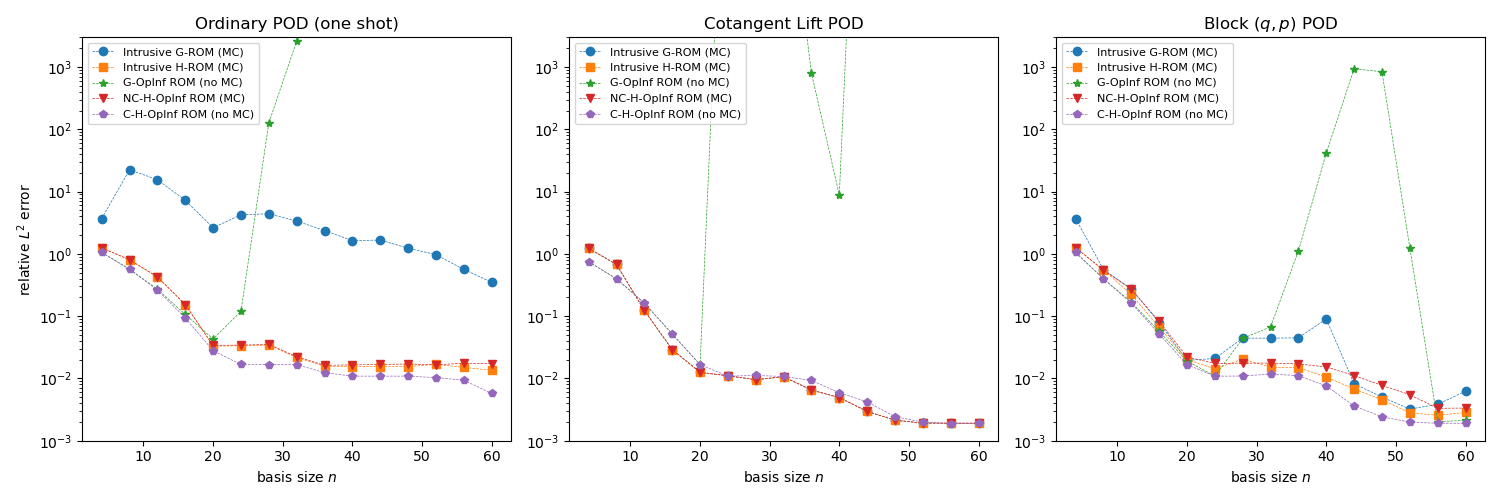}
    \caption{Relative state errors as a function of basis modes for the ROMs in the linear wave equation example (parametric predictive case $T=100$). ``MC'' indicates the use of a mean-centered reconstruction.}
    \label{fig:ToyWavePara}
\end{figure}

% Note that mean-centering was applied in this case only to the ROMs built with the ordinary POD basis, where it was shown to stabilize the intrusive G-ROM.

Figure~\ref{fig:ToyWavePara} illustrates the results of this experiment.  As in the purely predictive case, we see that the G-OpInf ROM is highly sensitive to basis size, while the intrusive H-ROM and H-OpInf ROMs exhibit a predictable increase in accuracy with the addition of basis modes.  Moreover, the intrusive G-ROM is significantly less accurate in the case of an ordinary POD basis, and indeed blows up similarly to the G-OpInf ROM in the case (not pictured) that mean-centering is not applied.  A consequence of this is illustrated in Figure~\ref{fig:ParaWaveSolns}, which shows the FOM and ROM solutions in the case that $\alpha=9.1$ and the ROMs are computed using an ordinary POD basis with mean-centering and with $n=28$ modes.  Notice that the G-OpInf ROM becomes increasingly unstable while the others remain bounded and close to the FOM solution throughout the range of integration.

\begin{figure}[htb]
    \centering
    \includegraphics[width=\textwidth]{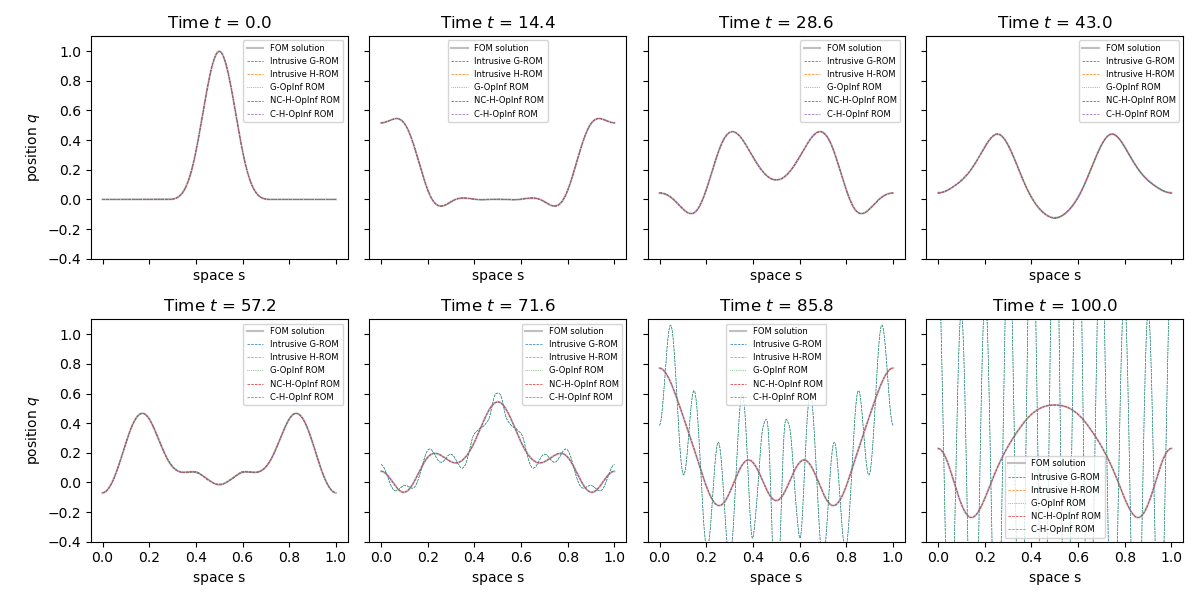}
    \caption{FOM and ROM solutions for the linear wave equation example in the parametric predictive case ($\alpha=9.1, T=100$) when using an ordinary POD basis without mean-centering and with $n=28$ modes.}
    \label{fig:ParaWaveSolns}
\end{figure}

Finally, it is illustrative to observe the energy plots in Figure~\ref{fig:ParaWaveEnergy}, computed using a block $(q,p)$ basis with mean-centering and with $n=16$ modes.  Here it is obvious that the improved conservation properties of the intrusive and OpInf H-ROMs persist in this setting as well, leading to improved accuracy and stability over time when compared to the G-ROMs which do not have this property.

\begin{figure}[htb]
    \centering
    \includegraphics[width=\textwidth]{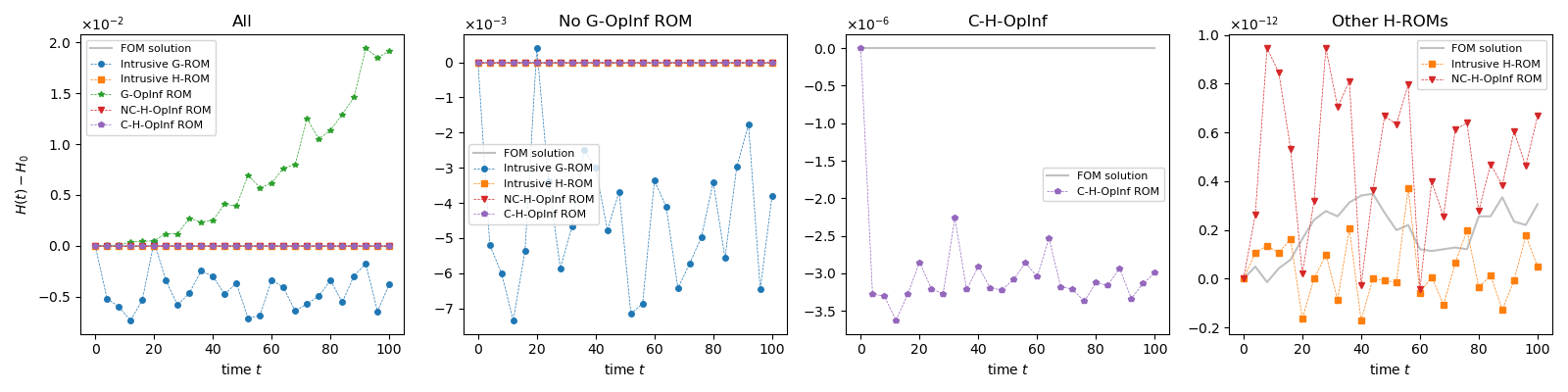}
    \caption{ROM energy errors for the linear wave equation example in the parametric predictive case ($\alpha=9.1, T=100$) when using a block $(q,p)$ POD basis with mean-centering (where applicable) and with $n=16$ modes.}
    \label{fig:ParaWaveEnergy}
\end{figure}

    % \begin{minipage}{0.75\textwidth}
    %     \includegraphics[width=\textwidth]{figs/1d3res.png}
    % \end{minipage}
    % \begin{minipage}{0.75\textwidth}
    %     \includegraphics[width=\textwidth]{figs/1d5res.png}
    % \end{minipage}
    % \begin{minipage}{0.75\textwidth}
    %     \includegraphics[width=\textwidth]{figs/1d10res.png}
    % \end{minipage}
    % \begin{minipage}{0.75\textwidth}
    %     \includegraphics[width=\textwidth]{figs/1d15res.png}
    % \end{minipage}

% The inferred operator exactly conserves the discrete Hamiltonian....

\subsection{A Non-separable Canonical Example}\label{subsec:nonseparable}\rev{
Since the linear wave equation can be similarly handled with the techniques in \cite{sharma2022hamiltonian}, it is worth considering a simple canonical example where the C-H-OpInf method is necessary.  Consider the Hamiltonian $H(q,p) = 1+qp$, which, after discretization as before, generates the canonical dynamics
\[ \frac{\xx^{k+1}-\xx^{k}}{\Delta t} = \JJ\bb{A}\lr{\frac{\xx^{k+1}+\xx^{k}}{2}} = \begin{pmatrix}0 & \bb{I} \\ -\bb{I} & 0\end{pmatrix}\begin{pmatrix} 0 & \bb{I} \\ \bb{I} & 0\end{pmatrix}\lr{\frac{\xx^{k+1}+\xx^k}{2}}.\]
Clearly, $\nabla H(\bb{x}) = \bb{Ax}$ does not satisfy the separability hypothesis of \cite[Algorithm 1]{sharma2022hamiltonian}, and therefore that method should not be effective at learning this system.  Conversely, the C-H-OpInf Algorithm~\ref{alg:C-H-OpInf} applies regardless of the separability of $H$, so it is expected that this system can still be learned through this approach.  To see that this is the case, a parameterized initial condition is considered,
\begin{align*}
\xx_0(\alpha
) = \begin{pmatrix}\qq_0 & \pp_0\end{pmatrix}^\intercal = \begin{pmatrix} e^{-\alpha(\qq+1)}\sin(\alpha\qq) & \pp\end{pmatrix}^\intercal,
\end{align*}
and, as before, eleven uniformly distributed parameters $\alpha\in[5,15]$ are chosen for training, and $501$ snapshots of the FOM solution in the range $[0,2]$ are collected using each parameter instance.  The resulting POD snapshot energies and projection errors are shown in Figure~\ref{fig:canonPODparametric}, along with some solution snapshots in Figure~\ref{fig:canonFOM}}.

\begin{figure}[htb]
    \centering
    \includegraphics[width=\textwidth]{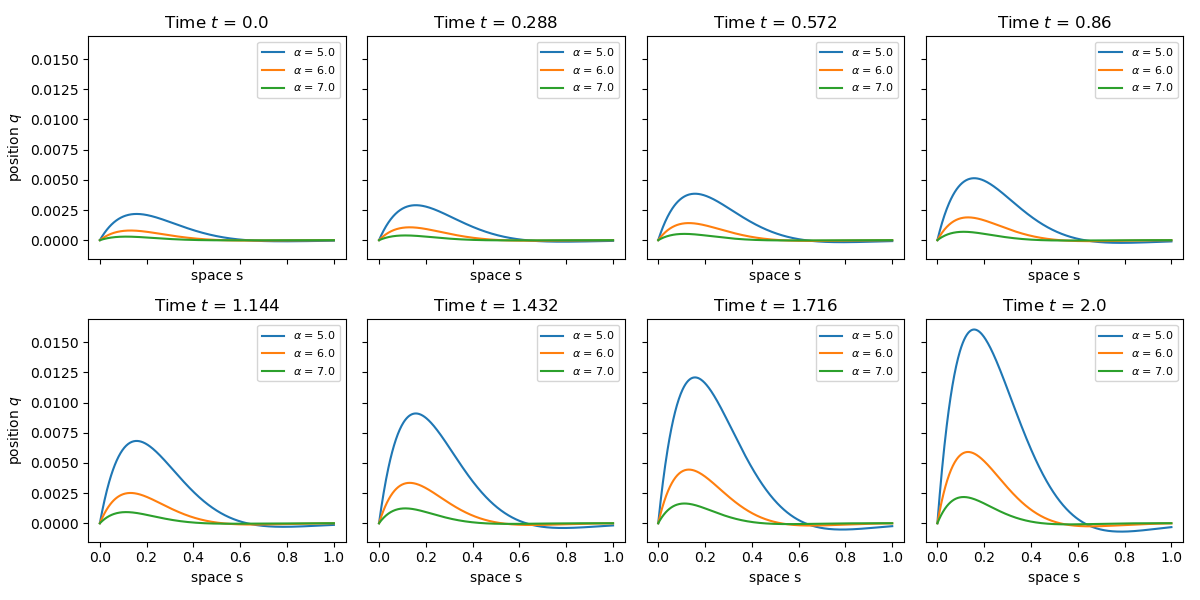}
    \caption{\rev{Solution snapshots from the non-separable canonical example for different values of the parameter $\alpha$.}}
    \label{fig:canonFOM}
\end{figure}

\begin{figure}[htb]
    \centering
    \includegraphics[width=\textwidth]{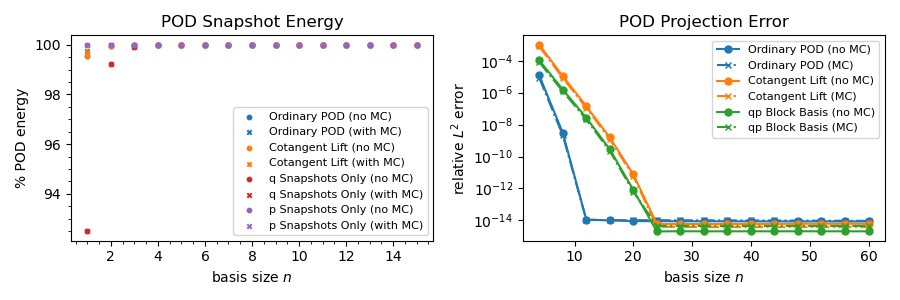}
    \caption{\rev{POD snapshot energies (left) and projection errors (right) corresponding to the bases used in the non-separable canonical example. ``MC'' indicates that mean-centering of the snapshots was performed.}}
    \label{fig:canonPODparametric}
\end{figure}

\rev{Again, the predictive case is considered.  For testing, six uniformly distributed parameters $\alpha \in [5.5,14.5]$, disjoint from the training data, are chosen, and snapshot data of each solution in the range $[0,10]$ is collected for comparison with the ROM integration.  The ROMs are then tested over this interval beginning from each unseen initial condition, and the average relative error over all test snapshots is reported.  In addition to the ROMs seen in the linear wave equation example, note that the H-OpInf ROM of \cite{sharma2022hamiltonian} discussed in Section~\ref{subsec:prevhopinf} is also reported.}

\begin{figure}[htb]
    \centering
    \includegraphics[width=\textwidth]{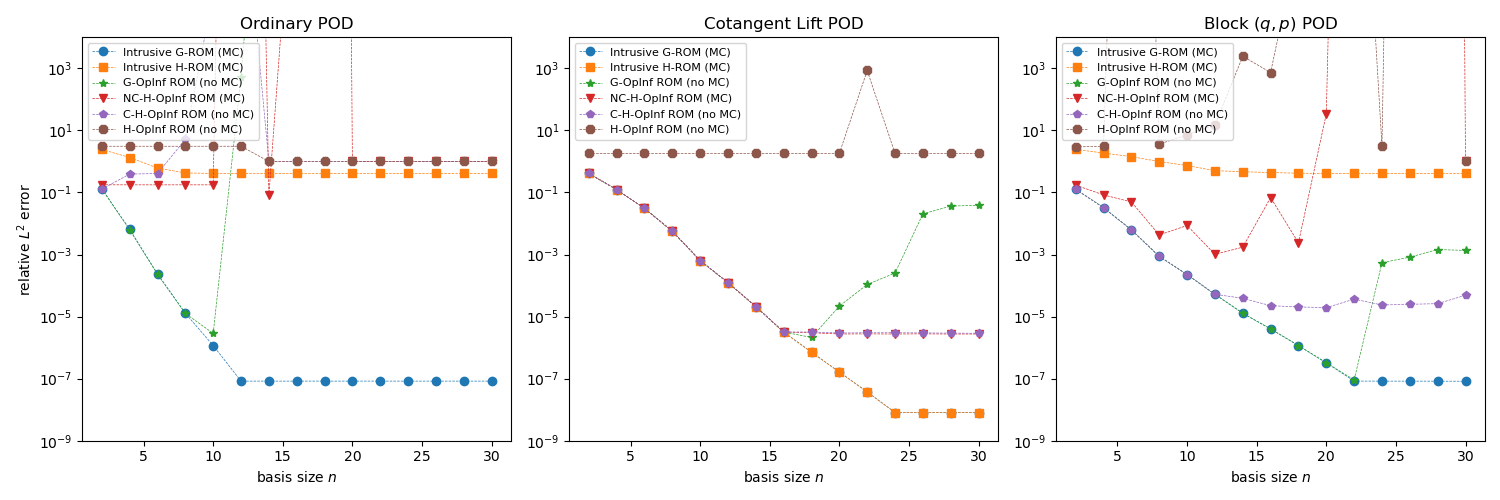}
    \caption{\rev{Relative state errors as a function of basis modes for the ROMs in the non-separable canonical example (parametric predictive case $T=10$). ``MC'' indicates the use of a mean-centered reconstruction.}}
    \label{fig:canonicalROMerrors}
\end{figure}

\rev{The results of this experiment are displayed in Figure~\ref{fig:canonicalROMerrors}.  As expected, H-OpInf cannot produce a useful ROM, while C-H-OpInf is effective whenever the POD basis is built block-wise or with the cotangent lift algorithm.  Interestingly, no Hamiltonian ROM algorithm is useful in the case where the POD basis is built from an SVD of the full snapshot matrix, while the intrusive Galerkin ROM appears to work quite well. This could be due to the fact that this Hamiltonian system decouples over $\qq$ and $\pp$: a quick calculation shows that $\qq = e^t\qq_0$ and $\pp=e^{-t}\pp_0$, so the scale separation in $\qq,\pp$ grows exponentially as $t$ increases. Conversely, C-H-OpInf with a cotangent lift basis learns an accurate and stable ROM, while the H-OpInf algorithm is unable to do so due to its assumption of a block diagonal $\Ah$.  In addition to the state errors, conservation of the system Hamiltonian is displayed in Figure~\ref{fig:canonicalEnergy}, where it is clear that C-H-OpInf is conservative to a much higher order than either the Galerkin ROMs or the H-OpInf ROM.}

\begin{figure}[htb]
    \centering
    \includegraphics[width=\textwidth]{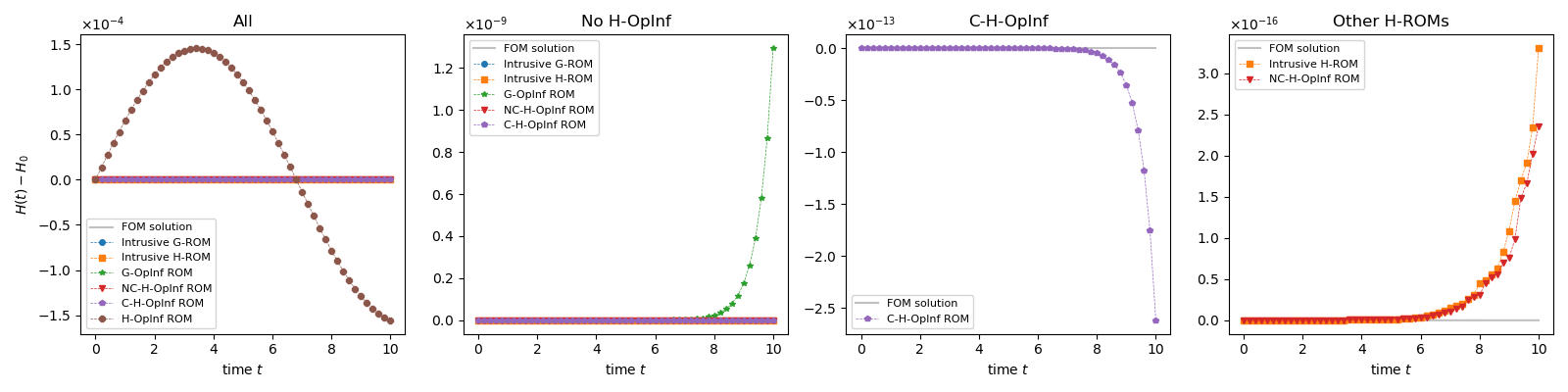}
    \caption{\rev{ROM energy errors for the non-separable canonical example in the parametric predictive case ($\alpha=7.3, T=10$) when using a cotangent lift POD basis with mean-centering (where applicable) and with $n=10$ modes.  Note that the intrusive G-ROM and intrusive H-ROM are identical in this case.}}
    \label{fig:canonicalEnergy}
\end{figure}

\subsection{Korteweg-De Vries equation}\label{subsec:kdv}
Moving beyond canonical Hamiltonian systems, consider the Korteweg-De Vries (KdV) equation \cite{karasozen2013energy}
\[ \dot{x} = \alpha xx_s + \rho x_s + \gamma x_{sss}, \qquad x\in [-l,l]\times [0,T], \]
which depends on the parameters $\alpha,\rho,\gamma\in\mathbb{R}$.  This equation has infinitely many integrals of motion \cite{nutku1985on}, the first few of which are mass, momentum, and energy:
\begin{align*}
    M(x) = \int_{-l}^l x\,ds, \qquad P(x) = \int_{-l}^l x^2\,ds, \qquad E(x) = \int_{-l}^l \lr{\frac{\alpha}{6}x^3 + \frac{\rho}{2}x^2 - \frac{\gamma}{2}x_s^2} ds.
\end{align*}
Moreover, KdV has a noncanonical bi-Hamiltonian structure, meaning that it can be recast as a Hamiltonian system in two distinct ways.  While only the first form will be considered here, the second form is also interesting and (to date) no POD-ROMs for it have been seen in the literature.  Therefore, some additional discussion regarding this second form is included in Appendix~\ref{app:kdv}.

\subsubsection{First Hamiltonian Formulation}
Consider the Hamiltonian functional $H(x)=E(x)$,
and note that its gradient satisfies 
\[\nabla H(x) = \frac{\alpha}{2}x^2 + \rho x + \gamma x_{ss}.\]
Then, recalling that $L:=\partial_s$ is an antisymmetric operator with respect to the usual metric on $L^2(\mathbb{R})$, it follows that $\dot{x} = L\nabla H(x)$ is a Hamiltonian system equivalent to the KdV equation.  Since $L$ has nontrivial kernel, this system is not canonical, meaning that there is no obvious way to separate the state $x$ into position and momentum variables.  Assuming periodic boundary conditions and a discretization $\xx\in\mathbb{R}^N$, the differential operators $\partial_s$ and $\partial_{ss}$ can be discretized with central finite differences as the circulant matrices
\[ \LL = \frac{1}{2\Delta x}\begin{pmatrix}0 & 1 & 0 & 0 & \ldots & -1 \\-1 & 0 & 1 & 0 & \ldots & 0 \\ & & \ddots & \ddots & \ddots & \\ 0 & \ldots & 0 & -1 & 0 & 1 \\ 1 & \ldots & 0 & 0 & -1 & 0 \end{pmatrix}, \qquad \bb{B} = \frac{1}{\lr{\Delta x}^2}\begin{pmatrix}-2 & 1 & 0 & 0 & \ldots & 1 \\1 & -2 & 1 & 0 & \ldots & 0 \\ & & \ddots & \ddots & \ddots & \\ 0 & \ldots & 0 & 1 & -2 & 1 \\ 1 & \ldots & 0 & 0 & 1 & -2 \end{pmatrix}, \]
yielding the semidiscrete Hamiltonian system 
% \IKTcomment{I suggest using $H_N$ for the discrete Hamiltonian, as noted earlier.}
\[ \dot{\xx} = \LL\nabla H(\xx) = \LL\lr{\frac{\alpha}{2}\xx^2 + \rho\xx + \nu\bb{B}\xx}. \]
Notice that the only nonlinearity in this system is polynomial in $\xx$, meaning that the quadrature necessary for AVF time discretization (see Section~\ref{subsec:linrom}) can be computed exactly.  This leads to the fully discrete system
\[ \frac{\xx^{k+1}-\xx^k}{\Delta t} = \LL\left[\frac{\alpha}{6}\lr{\lr{\xx^k}^2 + \xx^k\xx^{k+1} + \lr{\xx^{k+1}}^2} + \lr{\rho\bb{I} + \nu\bb{B}}\xx^{k+\frac{1}{2}}\right], \]
where $\xx^{k+\frac{1}{2}} = \lr{1/2}\lr{\xx^k + \xx^{k+1}}$ and vector products are interpreted element-wise.  This represents the KdV FOM and is solved by Newton iteration.  More precisely, at each time step $k$ we have the ($\Delta t$-normalized) residual and Jacobian functions
\begin{align*}
\bb{R}^k\lr{\vv} &= \vv-\xx^k - \Delta t\,\LL\left[\frac{\alpha}{6}\lr{\lr{\xx^k}^2 + \xx^k\vv + \vv^2} + \lr{\frac{\rho\bb{I} + \nu\bb{B}}{2}}\lr{\xx^{k}+\vv}\right], \\
\bb{J}^k\lr{\vv} &= \bb{I} - \frac{\Delta t}{2}\LL\left[\frac{\alpha}{3}\lr{\mathrm{Diag}\lr{\xx^k} + 2\,\mathrm{Diag}\lr{\vv}} + \rho\bb{I}+\nu\bb{B}\right],
\end{align*}
which are easily constructed and used to iterate $\vv^{i+1} = \vv^i - \JJ^k\lr{\vv^i}^{-1}\bb{R}^k\lr{\vv^i}$ until convergence.  It can be checked that this scheme exactly preserves the discrete Hamiltonian,
\[H(\xx) = \frac{1}{2}\sum_{j=1}^N\lr{\frac{\alpha}{3}x_j^3 + \rho x_j^2-\nu\lr{\frac{x_{j+1}-x_j}{\Delta x}}^2}\Delta x.\]

From this, it is possible to compute the intrusive G-ROM and intrusive H-ROM as described in Section~\ref{subsec:linrom}. Particularly, straightforward Galerkin projection onto a reduced basis contained in the columns of $\uu$ yields the reduced-order G-ROM system
\begin{align*}
\dot{\xh} &= \ut\LL\nabla H\lr{\xx_0 + \uu\xh} = \ut\LL\left[\frac{\alpha}{2}\lr{\xx_0+\uu\xh}^2 + \lr{\rho\bb{I}+\nu\bb{B}}\lr{\xx_0+\uu\xh}\right] \\
&= \lr{\frac{\alpha}{2}\ut\LL\xx_0^2 + \ut\LL\lr{\rho\bb{I}+\nu\bb{B}}\xx_0} + \ut\LL\lr{\alpha\,\mathrm{Diag}\lr{\xx_0} + \lr{\rho\bb{I}+\nu\bb{B}}}\uu\xh + \frac{\alpha}{2}\ut\LL\lr{\uu\xh}^2 \\
&:= \hat{\bb{c}} + \hat{\bb{C}}\xh + \hat{\bb{T}}\lr{\xh,\xh},
\end{align*}
where $\hat{\bb{T}}$ is a precomputable order-three tensor with components $\hat{T}^a_{bc} = \lr{\alpha/2}U^a_iL^i_jU^j_bU^j_c$.  Similarly, a reduced-order H-ROM system is given by 
\begin{align*}
\dot{\xh} &= \Lh\nabla\hat{H}\lr{\xh} = \Lh\ut\left[\frac{\alpha}{2}\lr{\xx_0+\uu\xh}^2 + \lr{\rho\bb{I}+\nu\bb{B}}\lr{\xx_0+\uu\xh}\right] \\
&= \Lh\left[\frac{\alpha}{2}\lr{\ut\xx_0^2 + 2\,\ut\mathrm{Diag}\lr{\xx_0}\uu\xh + \ut\lr{\uu\xh}^2} + \lr{\rho\bb{I}+\nu\bb{B}}\lr{\xx_0+\uu\xh}\right], \\
&= \Lh\left[\lr{\frac{\alpha}{2}\ut\xx_0^2 + \ut\lr{\rho\bb{I}+\nu\bb{B}}\xx_0} + \ut\lr{\alpha\,\mathrm{Diag}\lr{\xx_0} + \lr{\rho\bb{I}+\nu\bb{B}}}\uu\xh + \frac{\alpha}{2}\ut\lr{\uu\xh}^2 \right] \\
&:= \Lh\lr{\hat{\bb{c}} + \hat{\bb{C}}\xh + \hat{\bb{T}}\lr{\xh,\xh}},
\end{align*}
where $\hat{\bb{T}}:\mathbb{R}^n\times\mathbb{R}^n\to\mathbb{R}^n$ is a precomputable order-three tensor with components $\hat{T}^a_{bc} = (\alpha/2) U^a_i U^i_b U^i_c$.  In either case, applying AVF for temporal discretization and using the fact that $\hat{\bb{T}}$ is symmetric in its lower indices yields the fully discrete ROM (note that $\Lh=\bb{I}$ in the G-ROM),
\begin{align*}
\frac{\xh^{k+1}-\xh^k}{\Delta t} = \Lh\left[\hat{\bb{c}} + \hat{\bb{C}}\xh^{k+\frac{1}{2}} + \frac{1}{3}\lr{2\,\hat{\bb{T}}\lr{\xh^{k},\xh^{k+\frac{1}{2}}} + \hat{\bb{T}}\lr{\xh^{k+1},\xh^{k+1}}}\right],
\end{align*}
which is again solvable with Newton iterations.  In this case, the ($\Delta t$-normalized) residual and Jacobian at time step $k$ are given by
\begin{align*}
\hat{\bb{R}}^k\lr{\hat{\vv}} &= \hat{\vv}-\xh^k - \Delta t\,\Lh\left[\hat{\bb{c}} + \frac{1}{2}\bb{C}\lr{\xh^k+\hat{\vv}} + \frac{1}{3}\lr{\hat{\bb{T}}\lr{\xh^k,\xh^k+\hat{\vv}} + \hat{\bb{T}}\lr{\hat{\vv},\hat{\vv}}} \right], \\
\hat{\bb{J}}^k\lr{\hat{\vv}} &= \bb{I} - \Delta t\,\Lh\left[ \frac{1}{2}\bb{C} + \frac{1}{3}\lr{\Th\lr{\xh^k} + 2\,\Th\lr{\hat{\vv}}}\right],
\end{align*}
where $\Th\lr{\xh^k}$ indicates that the symmetric tensor $\Th$ is applied to the vector $\xh$ in either of its lower indices, yielding an $n\times n$ matrix.

The goal is now to compare these intrusive ROMs to the NC-H-OpInf ROM from Section~\ref{sec:hopinf} as well as a G-OpInf ROM which does not incorporate any structure information.  To facilitate a fair comparison, the G-OpInf procedure employed presently will not be black-box, but will instead aim to infer $\Lh$ in the intrusive H-ROM $\dot{\xh} = \Lh\hat\nabla H(\xh)$ similarly to NC-H-OpInf, but using the generic technique of Section~\ref{subsec:genopinf}.  This way, both the G-OpInf ROM and the NC-H-OpInf ROM are assumed to use analytic knowledge of the nonlinear part of $\nabla\hat{H}$, and both OpInf ROMs can be integrated
similarly to the intrusive H-ROM, but with the intrusive governing operator replaced by the inferred one.  For experimental parameters, we choose $l=20$, $(\alpha,\beta,\gamma) = (-6,0,-1)$, 
$N=500$, and an initial condition 
\[x_0(s) = \sech^2\lr{\frac{s}{\sqrt{2}}},\]
which generates a soliton solution for $s \in \mathbb{R}$.  To train the OpInf ROMs, 1001 snapshots of the solution $\xx$ and the gradient $\nabla H(\xx)$ are collected uniformly on the interval $[0,T]$ with $T=20$.

\begin{figure}[htb]
    \centering
    \includegraphics[width=\textwidth]{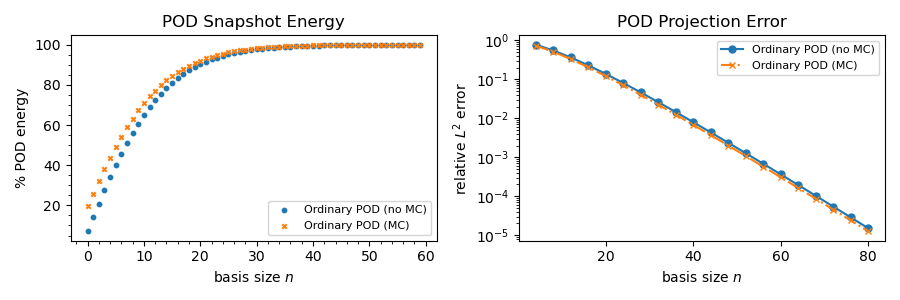}
    \caption{POD snapshot energies (left) and projection errors (right) corresponding to the bases used in the KdV equation example.  ``MC'' indicates mean-centering of the snapshots was performed.}
    \label{fig:KdVpod}
\end{figure}

Recall that there is no analogue of a block basis or cotangent lift method in the case of noncanonical Hamiltonian systems, so the POD bases $\bb{U}$ employed here are trained using the full snapshot matrix.  The associated snapshot energies and projection errors are displayed in Figure~\ref{fig:KdVpod}, where it is evident that the snapshot energy accumulates quite slowly with the addition of basis modes.  On the other hand, the use of ordinary POD bases again allows for the one-shot computation of all OpInf ROMs via truncation from the OpInf solution at the highest number of modes, creating large savings in computational cost.  While this is unlikely to be provably optimal in view of Proposition~\ref{prop:newtrunc}, the empirical difference in performance is small enough to justify the substantial decrease in computational time necessary for computing the ROMs.

% Note that, in the case of noncanonical Hamiltonian systems, there is no analogue of a block basis or cotangent lift method.   \IKTcomment{I think you should say this earlier when you are describing the different basis construction approaches.}

\begin{figure}[htb]
    \centering
    \begin{minipage}{0.5\textwidth}
         \includegraphics[width=\textwidth]{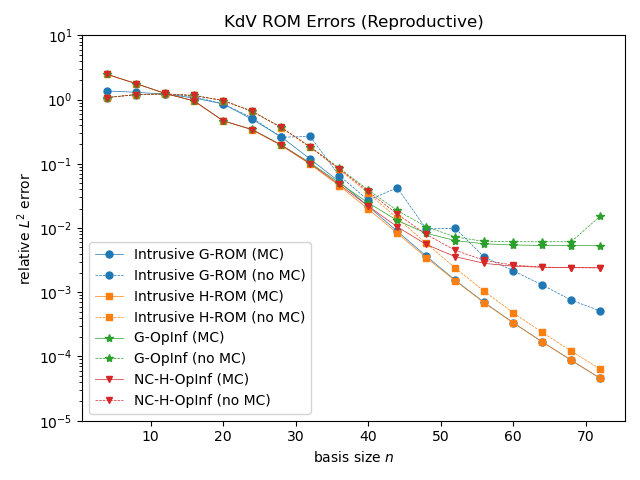}
    \end{minipage}%
    \begin{minipage}{0.5\textwidth}
         \includegraphics[width=\textwidth]{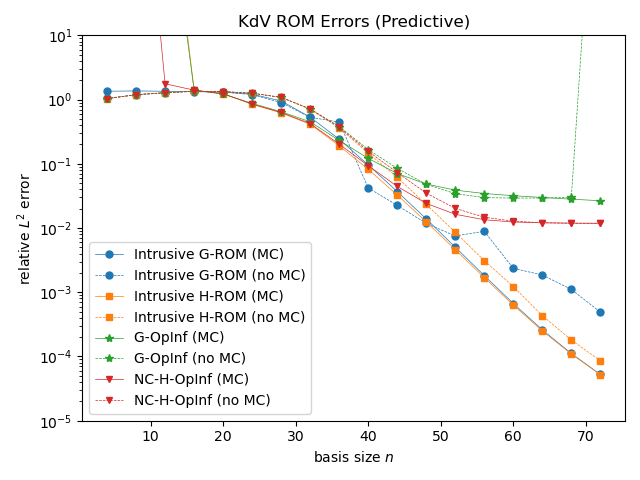}
    \end{minipage}
    \caption{Relative state errors as a function of basis modes for the ROMs in the KdV equation example.  Left: reproductive case ($T=20$).  Right: predictive case ($T=100$). ``MC'' indicates the use of a mean-centered reconstruction.}
    \label{fig:KdVerrors}
\end{figure}

As before, the performance of these ROMs in both predictive and reproductive cases is considered.
The relative state errors of each ROM as a function of basis modes are shown in Figure~\ref{fig:KdVerrors}, with the reproductive case ($T=20$) on the left and the predictive case ($T=100$) on the right.  Reported are the errors with and without mean-centering by $\xx_0$, as it is interesting to observe the effect of this choice.  Notice that 
mean-centering in the POD basis appears to make the intrusive ROMs more accurate and the OpInf ROMs more stable, perhaps because it ensures that the value of the Hamiltonian is exact at $\xh=\bb{0}$.
% \IKTcomment{Is there intuition for this result?}
However, in either case the NC-H-OpInf ROM remains more accurate and stable than the G-OpInf ROM, demonstrating the benefits of preserving antisymmetry in the learned operator.  Figures \ref{fig:KdVimshow} and \ref{fig:KdVsoln} provide a comparative illustration of the FOM and ROM solutions in the case that $n=32$ modes and mean-centering is applied.  While both OpInf ROMs are capable of predicting the general trajectory of the soliton, the NC-H-OpInf ROM exhibits much less artifacting over the rest of the domain\textemdash a consequence of capturing the correct latent space dynamics.  Note that, in either case, the performance of the OpInf ROM improves substantially as the number of modes increases, eventually leveling off around $n=60$ as a consequence of the failure of the learned dynamics to remain Markovian (see \cite{peherstorfer2020sampling}).

\begin{figure}[htb]
    \centering
    \includegraphics[width=\textwidth]{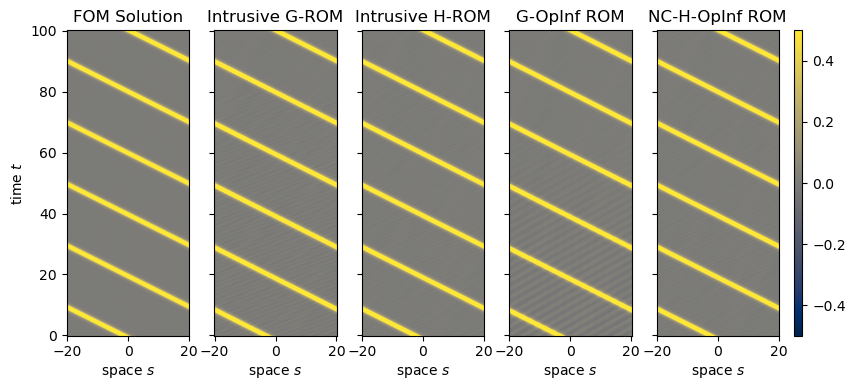}
    \caption{FOM and ROM solutions to the KdV  equation example with mean-centering (where applicable) and $n=32$ modes (predictive case $T=100$).}
    \label{fig:KdVimshow}
\end{figure}

\begin{figure}[htb]
    \centering
    \includegraphics[width=\textwidth]{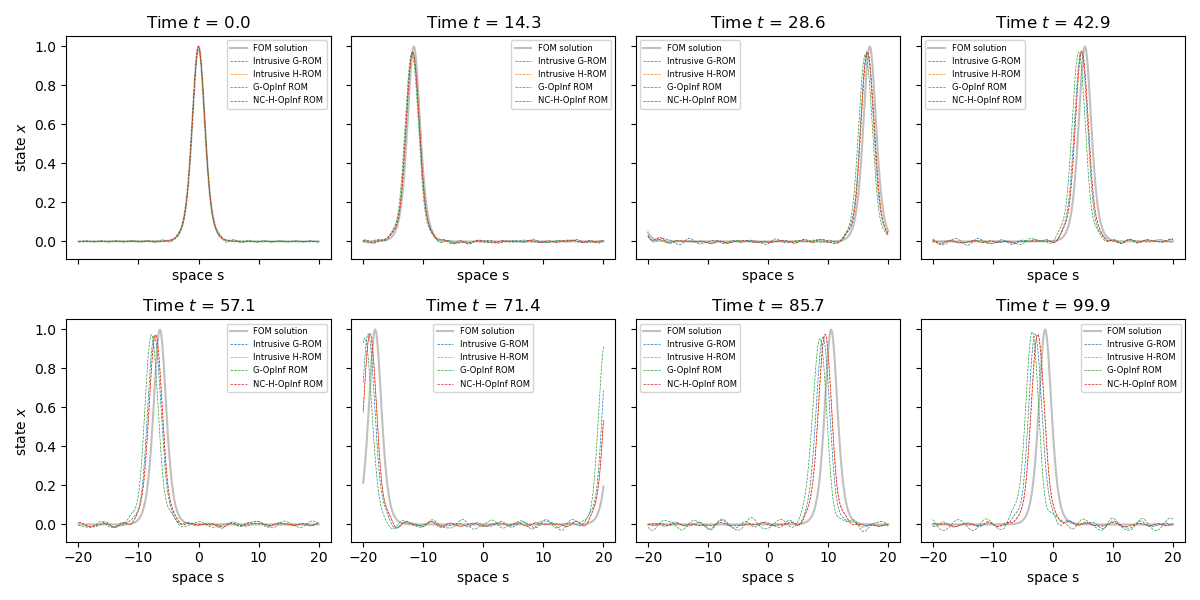}
    \caption{Snapshots in time correesponding to the FOM and ROM solutions to the KdV equation example with mean-centering (where applicable) and $n=32$ modes (predictive case $T=100$). }
    \label{fig:KdVsoln}
\end{figure}

Besides decreased state errors, Figure~\ref{fig:KdVConservedT100} shows the improved conservation of energy, mass, and momentum displayed by the H-ROMs over the G-ROMs when a mean-centered POD basis with $n=48$ modes is used.  Again, the conservation behavior of the intrusive H-ROM and the NC-H-OpInf ROM is orders of magnitude more accurate than the intrusive G-ROM or the G-OpInf ROM, reflecting the notion of the Hamiltonian as a conserved quantity.  It is also clear that the mass and momentum are preserved by the H-ROMs at least as well as the by the G-ROMs, demonstrating that other conserved quantities are not sacrificed for Hamiltonian preservation.  Finally, it is useful to note that, as before in the canonical case, this behavior persists regardless of the number of basis modes used in the ROM.

\begin{figure}[htb]
    \centering
    \includegraphics[width=\textwidth]{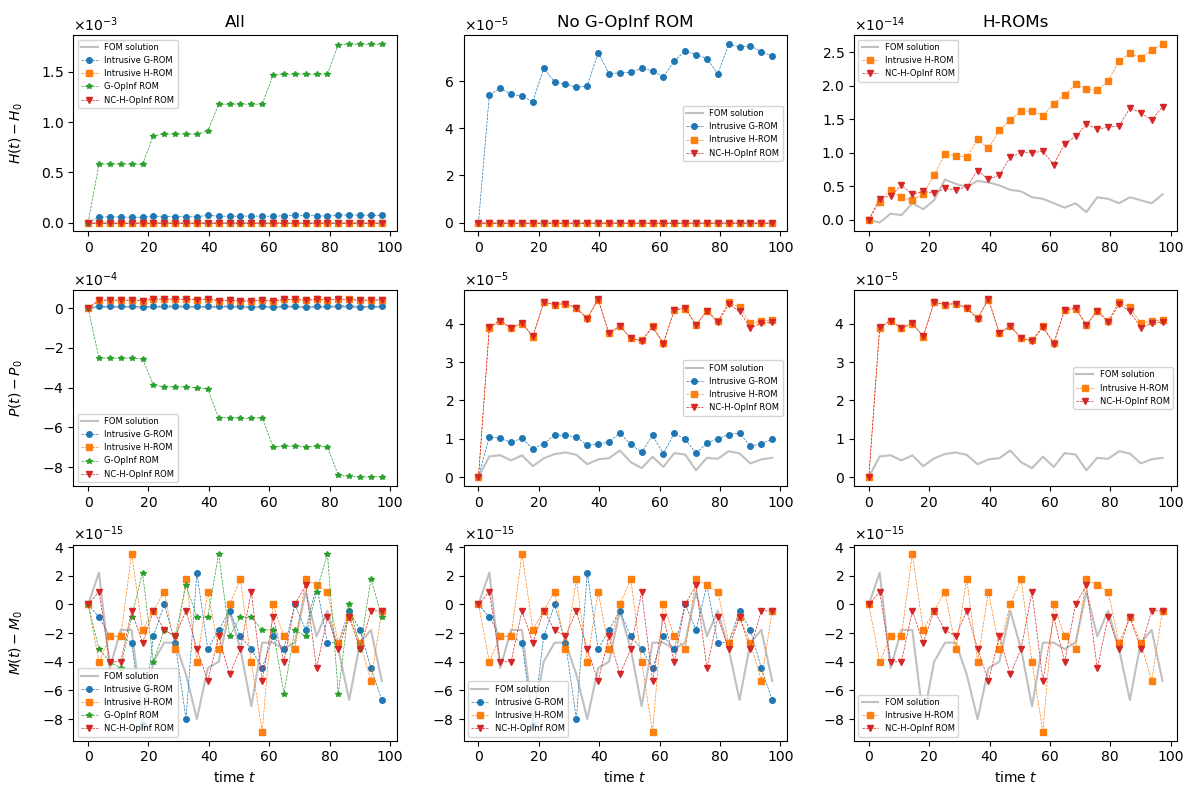}
    \caption{Errors in conserved quantities for the (mean-centered) ROMs in the KdV equation example in the predictive case ($T=100$) when using a POD basis with $n=48$ modes.}
    \label{fig:KdVConservedT100}
\end{figure}

\subsection{Benjamin-Bona-Mahoney Equation}

% \IKTcomment{Maybe I missed it, but do you say what spatial mesh resolution was used for the BBM problem?}

% \IKTcomment{Figure 10 seems out of place...  it's shown with the BBM results and appears not to be referenced in the text.  Should it go earlier with the wave equation results?  The figure really needs x and y labels, like all of them, or perhaps more so.  You also need a colorbar to explain what regions are high/low. }

As another example, consider the Benjamin-Bona-Mahoney (BBM) equation, also referred to as the regularized long-wave equation,
\[\dot{x} = \alpha x_s + \beta xx_s - \gamma \dot{x}_{ss}.\]
The BBM equation represents an alternative to the KdV equation introduced in \cite{peregrine1966calculations} and later in \cite{benjamin1972model}, intended as a model for the unidirectional propagation of long-range water waves with small amplitude.  This equation has a noncanonical Hamiltonian form defined by the data:
\[L=-\lr{1-\partial_s^2}^{-1}\partial_s,\qquad H(x) = \frac{1}{2}\int_0^\ell \alpha x^2 + \frac{\beta}{3}x^3\,ds, \]
where $L^2\lr{\mathbb{R}}$ is skew-symmetric and $H$ is the Hamiltonian.  The BBM equation is distinct from KdV in that it is not completely integrable, possessing only three globally conserved quantities.  In addition to $H$, these are the momentum and kinetic energy:
\[P(x) = \int_0^\ell \lr{x - \gamma x_{ss}} ds, \qquad KE(x) = \frac{1}{2}\int_0^\ell \lr{x^2 + \gamma x_s^2} ds. \]

Because the governing operator $L$ is unwieldy to spatially discretize, the BBM equation has (to the authors' knowledge) never been simulated in Hamiltonian form.  On the other hand, it is straightforward to discretize this system with pseudospectral techniques.  In particular,  denote the Fourier and inverse Fourier transforms of a function $f:\mathbb{R}\to\mathbb{R}$ by
\[\hat{f}(\xi) := \lr{\mathcal{F}f}(\xi) = \int_{-\infty}^{\infty}f(\xi)e^{-2\pi i\xi x}, \qquad f(x) = \lr{\mathcal{F}^{-1}\hat{f}}(x) = \int_{-\infty}^{\infty}\hat{f}(\xi)e^{2\pi i\xi x}.\]
Then, basic properties of the Fourier transform (see e.g. \cite{evans2010partial}) show that the BBM equation has the equivalent (non-Hamiltonian) expression
\[ \dot{x} = \mathcal{F}^{-1}\lr{\frac{-2\pi i\,\mathcal{F}\lr{\nabla H(x)}(\xi)}{1+4\gamma\pi^2\xi^2}}(x), \]
where $\nabla H(x) = \alpha x + (\beta/2)x^2$.  The FOM used presently is generated from this expression by semidiscretizing $x$ as $\xx\in\mathbb{R}^N$ with $N=1024$ and utilizing the fast Fourier transform and ``solve\_ivp'' functions found in the SciPy library \cite{scipy2020}. More precisely, given the discrete Hamiltonian
\[H(\xx) = \frac{1}{2}\sum_{j=1}^N\lr{\alpha x_j^2 + \frac{\beta}{3}x_j^3}\Delta x,\]
the FOM is computed by solving the system
\[ \dot{\xx} = \bm{\mathcal{F}}^{-1}\lr{\frac{-2\pi i\,\bm{\mathcal{F}}\lr{\nabla H(\xx)}}{1+4\gamma\pi^2\bm{\xi}^2}},\]
with an explicit Runge-Kutta method of order 8, and $\bm{\mathcal{F}},\bm{\mathcal{F}}^{-1}$ are the discrete Fourier and inverse Fourier transforms defined in terms of the vector $\bb{k}=\bb{m}=\begin{pmatrix}0 & 1 & ... & N-1\end{pmatrix}^\intercal$ of nonnegative integers at most $N-1$,
\[\bm{\mathcal{F}}(\xx) = \sum_{m=0}^{N-1} x_m\,\mathrm{exp}\lr{-\frac{2\pi i}{N}\bb{k}m}, \qquad \bm{\mathcal{F}}^{-1}(\bm{\xi}) = \frac{1}{N}\sum_{k=0}^{N-1} \xi_k\,\mathrm{exp}\lr{\frac{2\pi i}{N}k\bb{m}}.\]

Since it is challenging to build an intrusive Hamiltonian ROM for the BBM system, it is useful to see if the governing operator can be effectively learned by the OpInf methods seen in Section~\ref{sec:hopinf}.  This would allow for a nonintrusive spatial ROM which preserves the underlying Hamiltonian structure, which could be valuable in cases where conservation is paramount.  As before, accomplishing this means inferring $\Lh$ in $\dot{\xh}=\Lh\nabla\hat{H}\lr{\xh}$, which is readily done by solving the linear system in equation \eqref{eq:NC-H-OpInf}.  Similar to the case of KdV, this result will be compared to the nonintrusive ROM generated by inferring $\Lh$ using the generic OpInf technique described in Section~\ref{subsec:genopinf}.  Provided a suitable $\Lh$ is available, the desired Hamiltonian ROM becomes 
\begin{align*}
    \dot{\xh} &= \Lh\nabla\hat{H}\lr{\xx_0 + \uu\xh} = \Lh\ut\left[\alpha\lr{\xx_0+\uu\xh} + \frac{\beta}{2}\lr{\xx_0+\uu\xh}^2\right] \\ 
    &= \Lh\left[ \ut\lr{\alpha\xx_0+\frac{\beta}{2}\xx_0^2} + \ut\lr{\alpha\bb{I}+\beta\,\mathrm{Diag}\,\xx_0}\uu + \Th\lr{\xh,\xh} \right] \\
    &= \Lh\lr{\hat{\bb{c}}+\hat{\bb{C}}\xh + \Th\lr{\xh,\xh}},
\end{align*}
where $\Th$ is precomputable with components $T^a_{bc} = (\beta/2)U^a_iU^i_bU^i_c$ similar to the case of KdV.  With these definitions of $\hat{\bb{c}},\hat{\bb{C}},\Th$, AVF integration yields the BBM ROM system
%($\Lh=\bb{I}$ for the G-ROM)
\begin{align*}
\frac{\xh^{k+1}-\xh^k}{\Delta t} = \Lh\left[\hat{\bb{c}} + \hat{\bb{C}}\xh^{k+\frac{1}{2}} + \frac{1}{3}\lr{2\,\hat{\bb{T}}\lr{\xh^{k},\xh^{k+\frac{1}{2}}} + \hat{\bb{T}}\lr{\xh^{k+1},\xh^{k+1}}}\right],
\end{align*}
which is solvable with Newton iterations identically to the KdV system.

% Moreover, there is a corresponding OpInf G-ROM which can be assembled for comparison,
% \begin{align*}
%     \dot{\xh} &= \ut\LL\nabla\hat{H}\lr{\xx_0 + \uu\xh} = \ut\LL\left[\alpha\lr{\xx_0+\uu\xh} + \frac{\beta}{2}\lr{\xx_0+\uu\xh}^2\right] \\ 
%     &= \ut\LL\lr{\alpha\xx_0+\frac{\beta}{2}\xx_0^2} + \ut\LL\lr{\alpha\bb{I}+\beta\,\mathrm{Diag}\,\xx_0}\uu + \Th\lr{\xh,\xh} \\
%     &= \hat{\bb{c}}+\hat{\bb{C}}\xh + \Th\lr{\xh,\xh},
% \end{align*}
% where $\hat{T}^a_{bc}=(\beta/2)U^a_iL^i_jU^j_bU^j_c$ is precomputable.

\begin{figure}[htb]
    \centering
    \includegraphics[width=\textwidth]{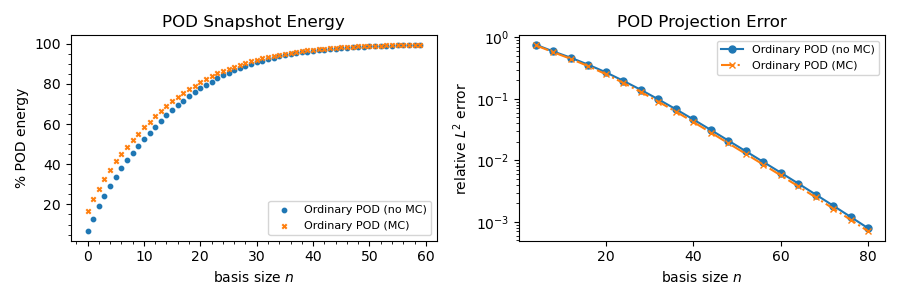}
    \caption{Left: POD snapshot energies (left) and projection errors (right) corresponding to the bases used in the BBM equation example. ``MC'' indicates mean-centering of the snapshots was performed.}
    \label{fig:BBMpodmodes}
\end{figure}

% \IKTcomment{Do figures 13-14 have MC?  Should make this clear in caption.}

% \IKTcomment{I would refer to the POD reconstruction error in Fig. 11 as the projection error and give a formula for it.  Also, is it relative or absolute error?  I think it's relative based on the text, but would clarify it in the caption. 
%  For the left two subplots, I feel like showing the percent energy captured might be more illustrative.  I'd plot it on a semilogy scale.  }

For the present experiment, the parameters in the governing equation are set to $(\alpha,\beta,\gamma)=\lr{1,1,10^{-4}}$ and 2001 snapshots of $\xx, \nabla H(\xx)$ are collected in the interval $[0,T]$ for $T=0.5$, starting from the initial condition 
\[\xx_0 = 7\sech^2\lr{\sqrt{\frac{1}{5\gamma}}\lr{\bb{s}-0.25}} + 3\sech^2\lr{\sqrt{\frac{1}{6\gamma}}\lr{\bb{s}-0.35}}.\]
This generates a nonperiodic 2-solitary wave solution, which experiences an inelastic collision over the length of the training integration.  The relative POD energies and reconstruction errors of the computed POD bases are shown in Figure~\ref{fig:BBMpodmodes}, where it is seen that the reconstruction error decays quite slowly as a function of basis modes.  It is further interesting to observe that the first eigenvector of the mean-centered basis contains much more information than the others, although this does not appear to yield a faster decrease in reconstruction error.  These data are used to train the NC-H-OpInf ROM and a corresponding G-OpInf ROM.

% \IKTcomment{I'd have to see the percent energies on a semilogy scale to evaluate better, but I am not really surprised by this.  You see in contrast w/o the MC, there is a bigger difference in energy between the first and second POD mode - this is because the first POD mode basically represents the mean/base state.}

For testing, AVF time integration is carried out to $T=0.5$ and $T=1$, respectively, representing reproductive and predictive scenarios.  The relative errors of these ROMs as a function of basis size are displayed in Figure~\ref{fig:BBMerrors}, where it can be seen that the errors for the NC-H-OpInf ROMs are about half of those for the G-OpInf ROMs.  However, it is also clear that this example poses a much greater challenge for either OpInf ROM, likely due to the nonperiodic and inelastic nature of the solitary wave collisions present in the BBM solution, as well as the complicated form of the governing operator $\LL$.  It is interesting to note the effect of mean-centering here:  in either case, there is a significant gain in performance for middling numbers of modes (20-60) which diminishes as more modes are added.

\begin{figure}[htb]
    \centering
    \begin{minipage}{0.5\textwidth}
         \includegraphics[width=\textwidth]{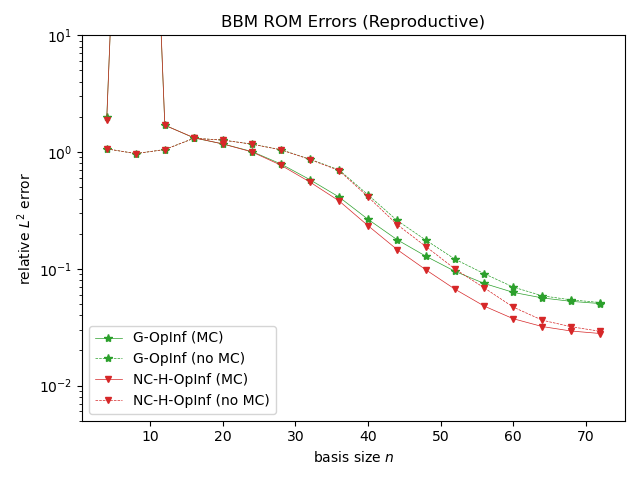}
    \end{minipage}%
    \begin{minipage}{0.5\textwidth}
         \includegraphics[width=\textwidth]{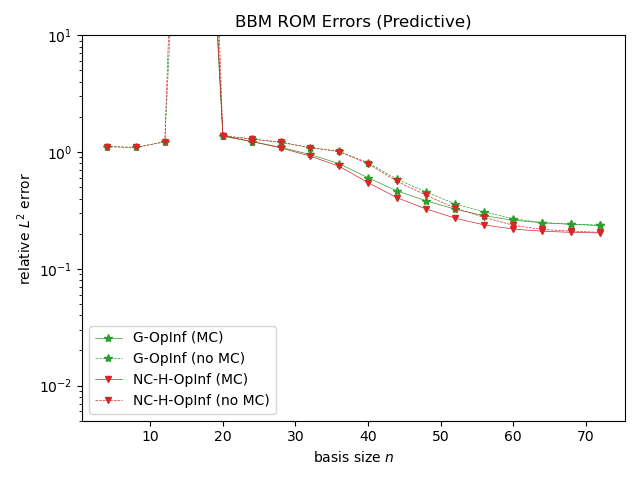}
    \end{minipage}
    \caption{Relative state errors as a function of basis modes for the ROMs in the BBM equation example.  Left: reproductive case ($T=0.5$).  Right: predictive case ($T=1$).  ``MC'' indicates the use of a mean-centered reconstruction.}
    \label{fig:BBMerrors}
\end{figure}

Visual comparisons of the FOM and ROM solutions in the predictive case are shown in Figures~\ref{fig:BBMsoln} and \ref{fig:BBMimshow}, where two inelastic collisions are pictured and the second collision occurs outside the range of the training data (note the difference in the tails).  Even at $n=44$ modes, the collisions are relatively well captured, validating the hypothesis that the NC-H-OpInf procedure can produce a useful and nonintrusive spatial ROM even when the FOM is pseudospectral and the involved operator $\Lh$ cannot be readily discretized by standard techniques.  

% It is also remarkable that the NC-H-OpInf R

\begin{figure}[htb]
    \centering
    \includegraphics[width=\textwidth]{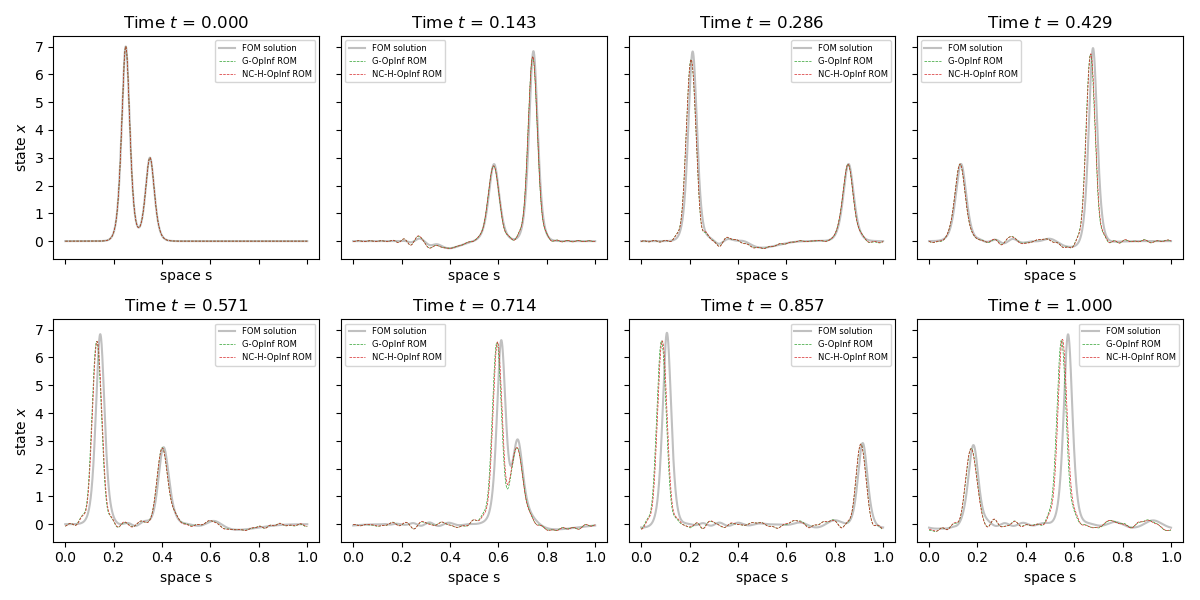}
    \caption{Space-time plots showing the evolution of the BBM ROMs in the predictive case ($T=1$) with mean-centering and using $n=44$ modes.}
    \label{fig:BBMsoln}
\end{figure}

\begin{figure}[htb]
    \centering
    \includegraphics[width=\textwidth]{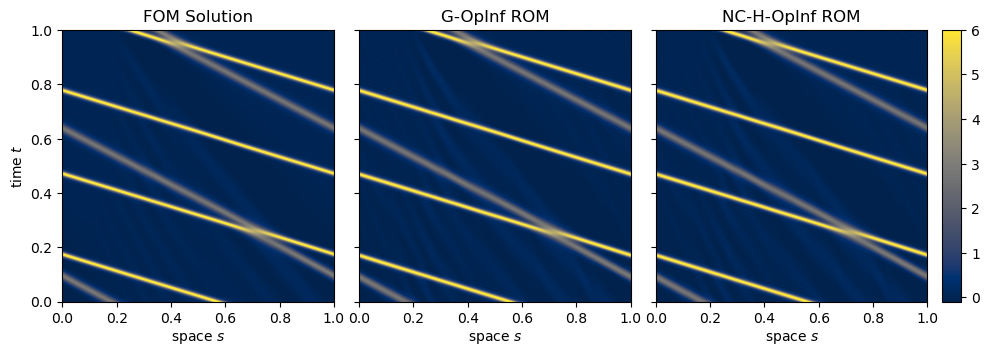}
    \caption{Snapshots showing the evolution of the BBM ROMs in the predictive case ($T=1$) with mean-centering and using $n=44$ modes.}
    \label{fig:BBMimshow}
\end{figure}

Moving beyond state errors, the difference in conserved quantities between the NC-H-OpInf and G-OpInf ROMs is displayed in Figure~\ref{fig:BBMconserved}, using mean-centered POD bases and $n=44$ modes.  From this, it is evident that the Hamiltonian is conserved exactly by the NC-H-OpInf ROM but not by the G-ROM (note that the FOM is conservative to $\mathcal{O}\lr{10^{-12}}$), likely enabling the NC-H-OpInf ROM to capture small-scale features like the tails of the solitons in Figure~\ref{fig:BBMsoln} much more realistically.  Moreover, it appears that both OpInf ROMs are capable of conserving momentum exactly and kinetic energy to a relatively low-order.  It is interesting to note that mean-centering makes a difference here:  without this choice (not shown here), the momentum conservation of both ROMs is on the same order as the kinetic energy conservation.

% While this shows that either OpInf ROM is capable of conserving momentum exactly, it is clear that both the Hamiltonian and kinetic energy functionals are much better conserved by the NC-H-OpInf ROM solution.  It is interesting to note that mean centering makes a difference here: without this choice, the conservation properties of the G-OpInf ROM are much worse \IKT{(not shown here)}{}.

\begin{figure}[htb]
    \centering
    \includegraphics[width=\textwidth]{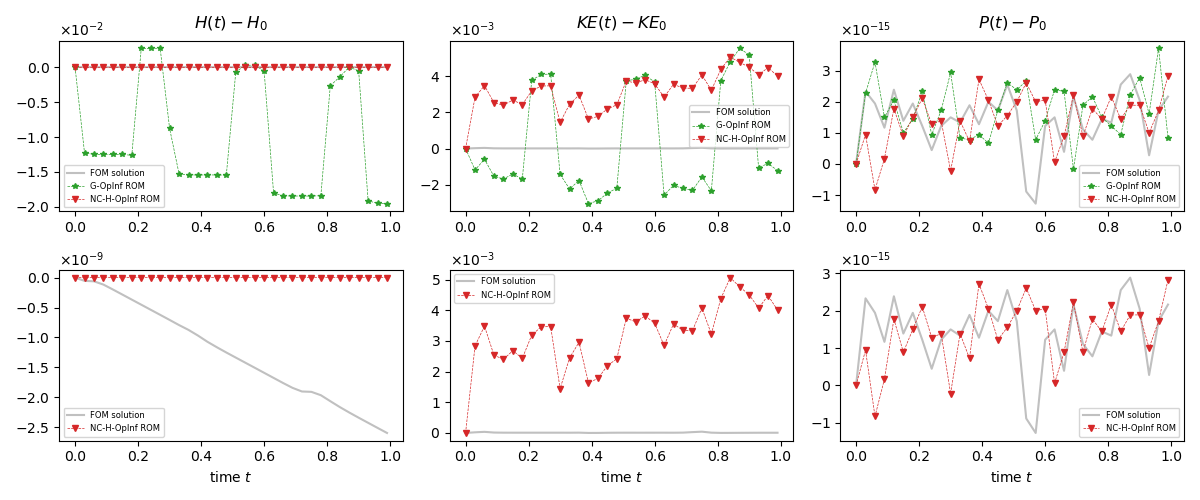}
    \caption{Plots showing the evolution of the conserved quantities for the mean-centered ROMs in the BBM equation example in the predictive case ($T=1$) using $n=44$ modes.}
    \label{fig:BBMconserved}
\end{figure}

% \IKTcomment{This is nit-picky and optional, but you could comment that the FOM in Figure \ref{fig:BBMconserved} is conservative to $O(1e-12)$ even though it looks like there is a lack of conservation when one first glances at the relevant subfigure.}

% \textcolor{red}{Irina: you are invited to start writing on this section also.  I think we should include some brief background on the problem (probably not more than about 1 page, but I leave that discretion up to you), and then I can fill in the experimental details with pictures.  Please suggest visualizations you would like to see, if you have preferences.} 
% \IKTcomment{I noticed that I am using the boldtensors library to create bold matrices representing vectors/matrices, and you are not, Anthony.  We should probably make the notation consistent in the paper...}

\subsection{Three-dimensional Linear Elasticity} \label{subsec:3d_lin_elast}
The final example considered in this work involves a moderate size three-dimensional (3D) linear elasticity problem, given by the following equations of motion:
\begin{equation} \label{eq:3delasticity}
    \rho \ddot{\bb{q}} = \nabla \cdot  \bm{\sigma}, \hspace{0.5cm} \text{on } \Omega \in \mathbb{R}^3.
\end{equation}
In \eqref{eq:3delasticity}, $\bb{q} \in \mathbb{R}^3$ is the displacement vector, $\rho>0$ is the material density, and $\bm{\sigma}$ is the Cauchy stress tensor. We assume that the material is elastic and follows Hooke's law, so that the components of $\bm{\sigma}$ satisfy
\begin{equation} \label{eq:sigma}
    \sigma_{ij}:= \lambda \text{Tr}(\bm{\epsilon}) \delta_{ij} + 2\mu \epsilon_{ij}, \quad 1\leq i \leq j \leq 3,
\end{equation}
where $\lambda, \mu >0$ are the Lam\'{e} coefficients and 
\begin{equation}
    \bm{\epsilon} := \frac{1}{2} \left[ \nabla \bb{q} + (\nabla \bb{q})^\intercal \right]
    \end{equation}
    is the infinitesimal strain tensor.  It can be shown \cite{Marsden_BOOK_1998} that the Hamiltonian for \eqref{eq:3delasticity} can be expressed in terms of (noncanonical) position and velocity variables: 
\begin{equation} \label{eq:3d_ham_lin_elastic}
    H(\bb{q}, \dot{\bb{q}}) = \frac{1}{2}\int_{\Omega} \left(\rho |\dot{\qq}|^2  + 
    \lambda [\text{Tr}(\bm{\epsilon})]^2 + 2\mu \nn{\bm{\epsilon}}^2\right)dV.
\end{equation}
We remark that, in 1D, \eqref{eq:3d_ham_lin_elastic} reduces to the linear wave equation \eqref{eq:lin_wave_eqn} considered earlier in Section \ref{sec:lin_wave}.  The main purpose of this example is to demonstrate the utility of non-intrusive ROMs on an application in which the FOM is implemented within a large HPC code without embedded ROM capabilities, making intrusive model reduction infeasible.

\begin{figure}[htbp!]
        \begin{center}
                \subfigure[$t = 0$ s]{
      \includegraphics[width=0.25\textwidth]
                      {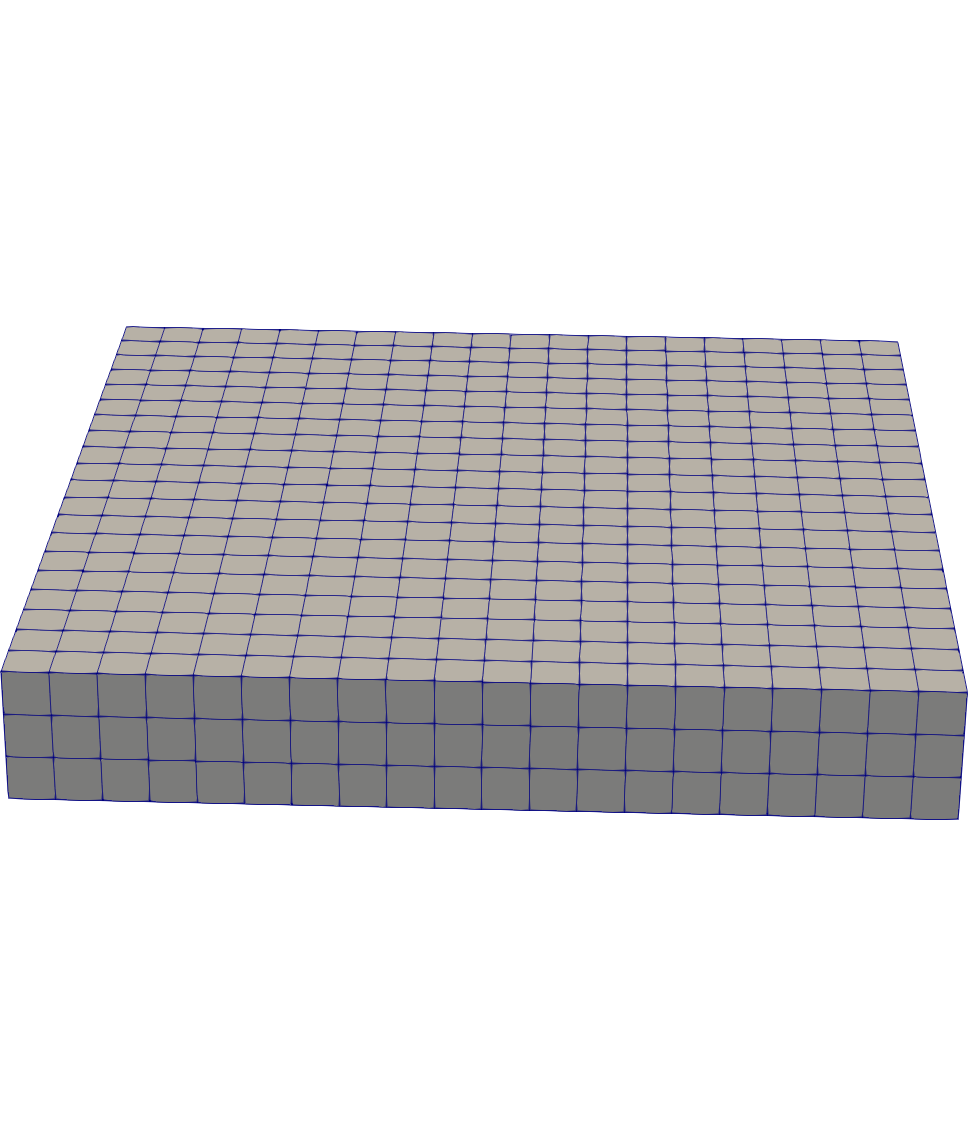}} \hspace{1cm}
    \subfigure[$ t = 1.0 \times 10^{-3}$ s]{
      \includegraphics[width=0.25\textwidth]
                      {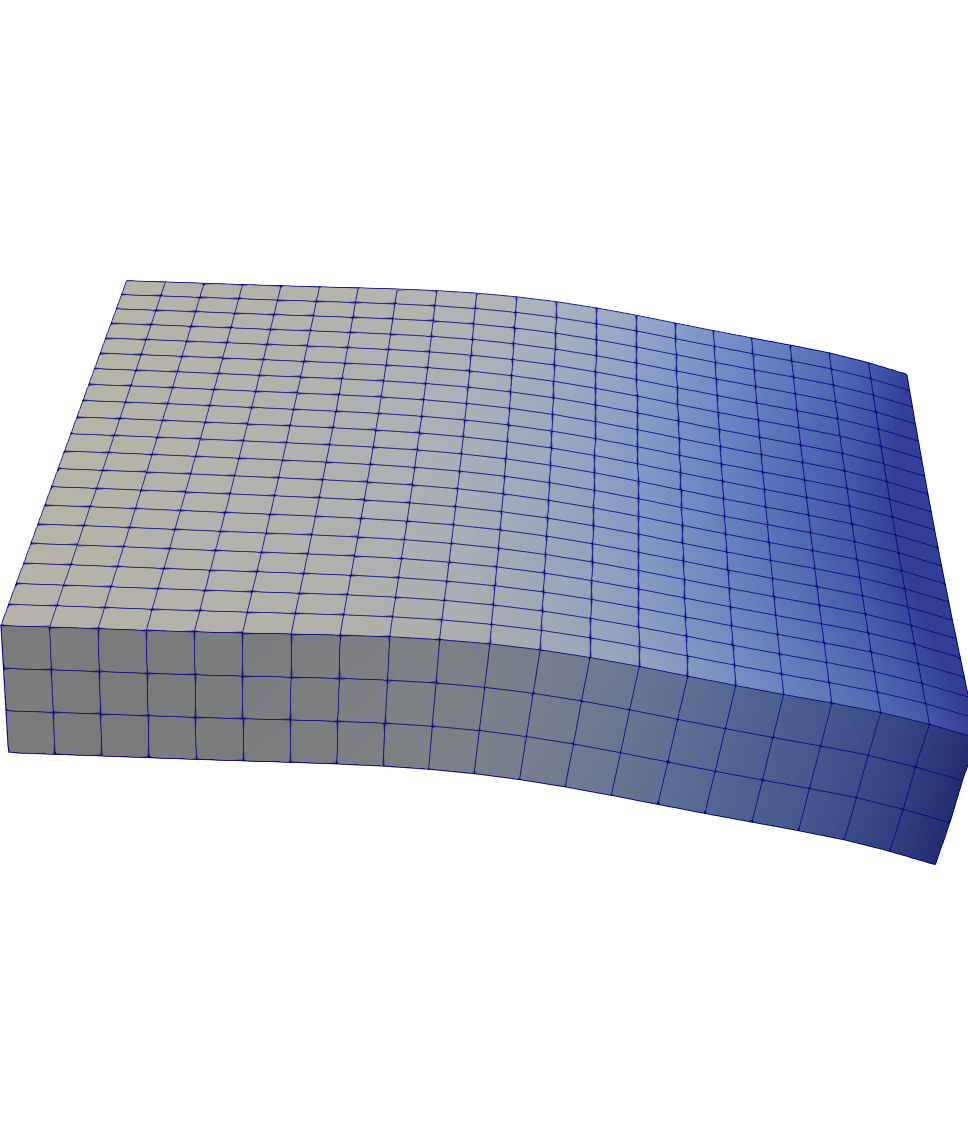}} \hspace{1cm}
                       \subfigure[$ t = 2.0 \times 10^{-3}$ s]{
      \includegraphics[width=0.25\textwidth]
                      {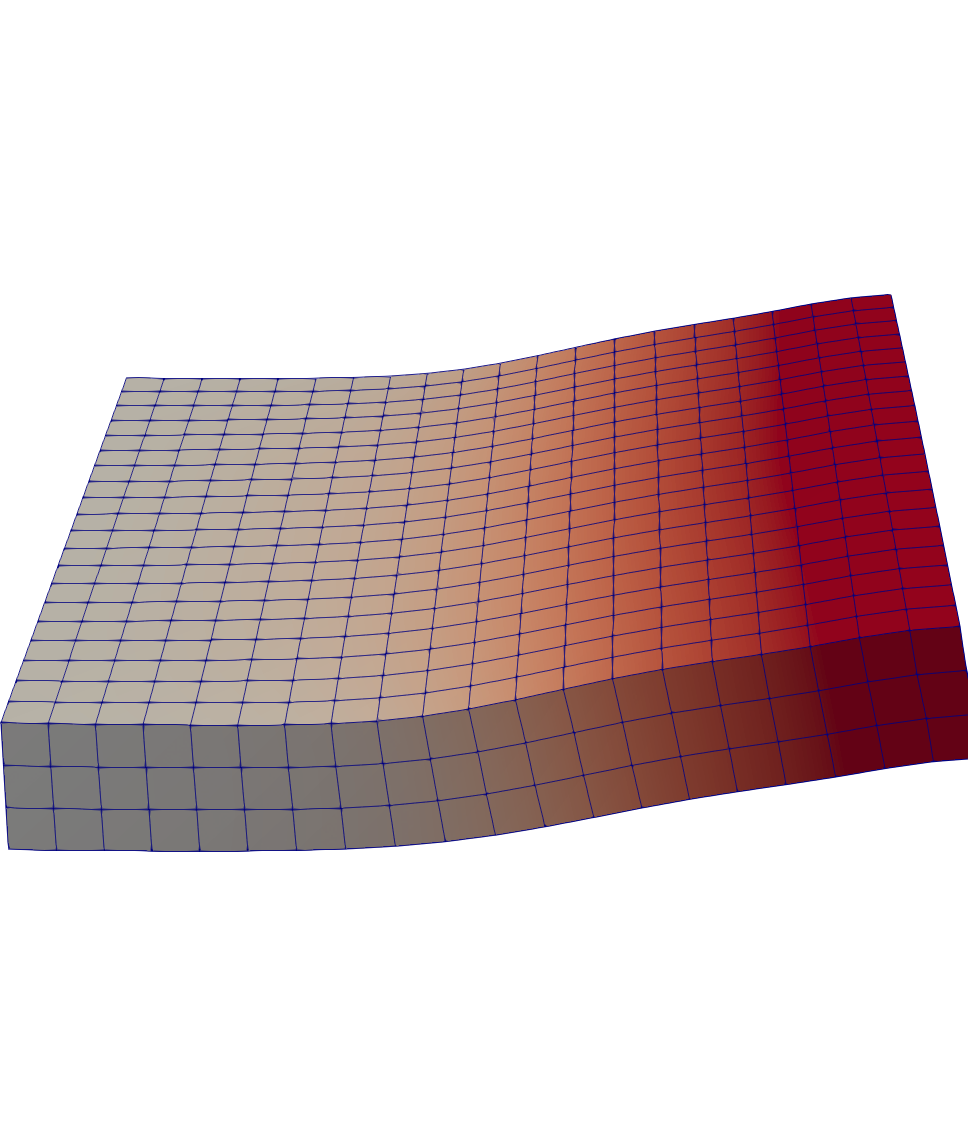}}
                       \subfigure[$ t = 9.0 \times 10^{-3}$ s]{
      \includegraphics[width=0.25\textwidth]
                        {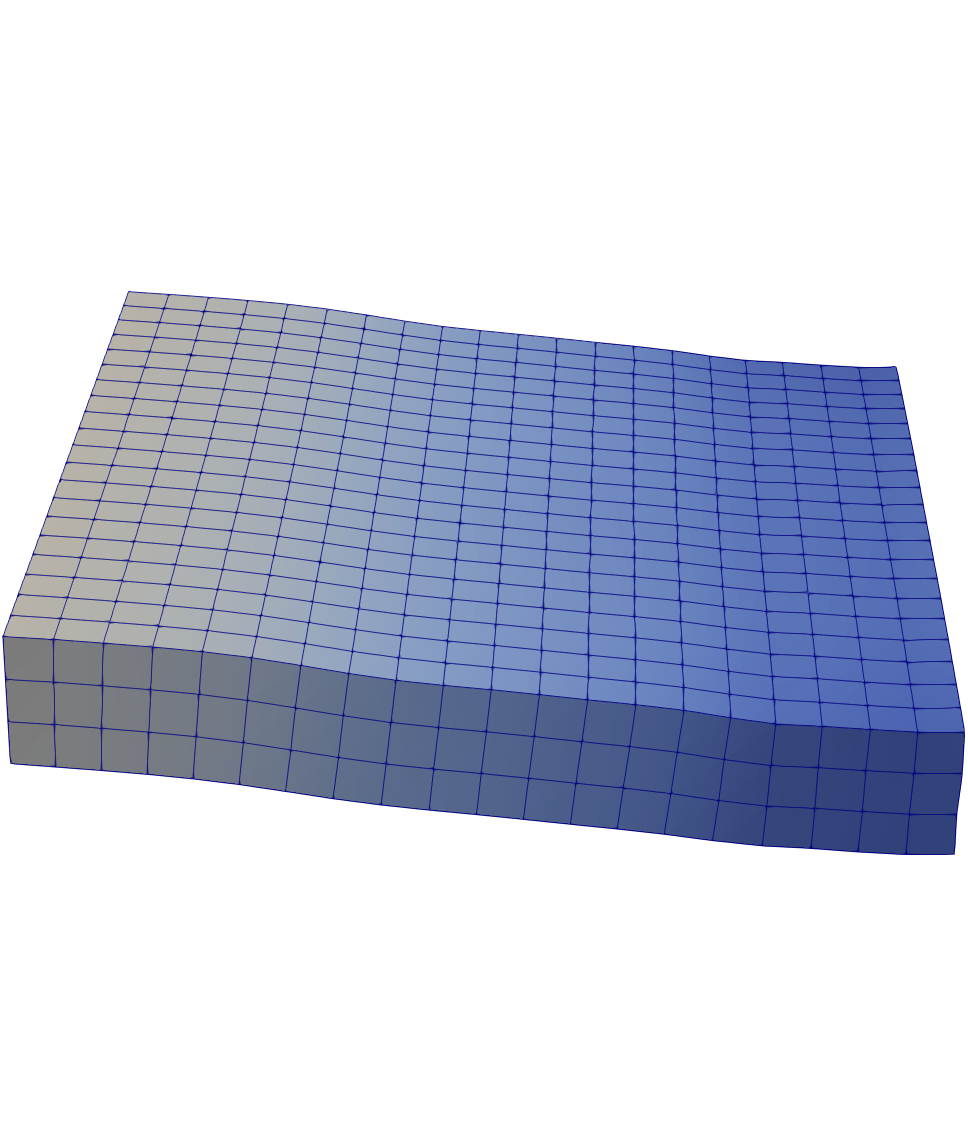}} \hspace{1cm}
                       \subfigure[$ t = 1.80 \times 10^{-2}$ s]{
      \includegraphics[width=0.25\textwidth]
                      {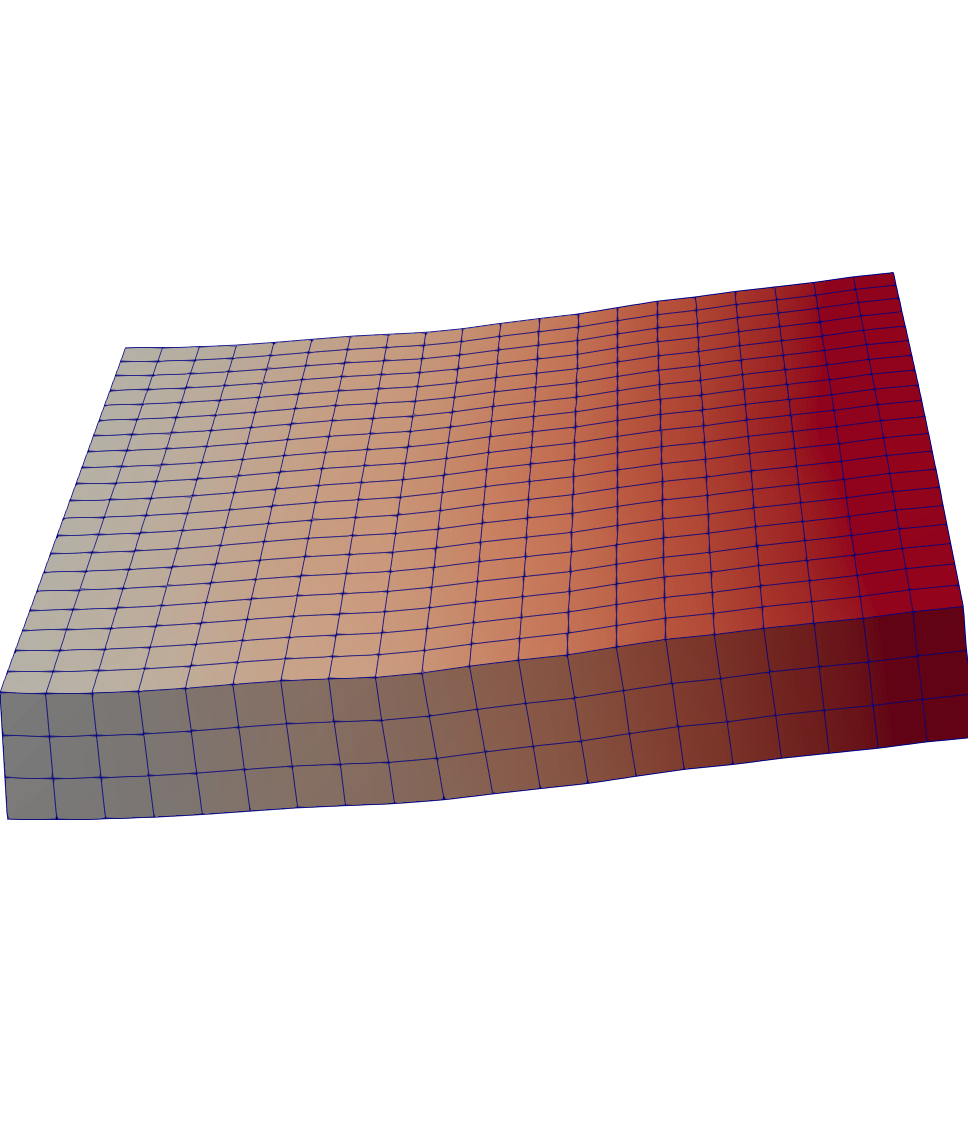}} \hspace{1cm}
      %                 \subfigure[$ t = 1.80 \times 10^{-2}$ s]{
      %\includegraphics[width=0.025\textwidth]
      %                {figs/legend.png}} 
        \end{center}
        \caption{Plot of FOM $s_3$--displacement, scaled by a factor of ten, at several times for the 3D linear elastic cantilever beam problem.  The colorbar range is $-2.3\times 10^{-3}$ m (blue) to $2.3 \times 10^{-3}$ m (red).}
        \label{fig:3d_elasticity_fom}
\end{figure}

% \IKTcomment{Note that one can also write \eqref{eq:3d_ham_lin_elastic} using the Einstein notation, which sums over repeated indices.  Anthony, if you like that better, feel free to change the equation.}  

Herein, equation \eqref{eq:3delasticity} is assumed to be discretized in space using the finite element method (FEM), per common practice in the field of solid mechanics.  Doing so gives a semi-discrete system of the form
\begin{equation} \label{eq:lin_elasticity_discrete}
    \bb{M} \ddot{\qq} + \bb{K} \qq = \bb{0}, 
\end{equation}
where (overloading notation) $\qq \in \mathbb{R}^M$ is the discretized displacement field, and $\bb{M} \in \mathbb{R}^{M \times M}$ and $\bb{K} \in \mathbb{R}^{M \times M}$ are the mass and stiffness matrices, respectively.  
% \IKTcomment{You can't define $\qq$ to be in $\mathbb{R}^M$ and $\qq\in\mathbb{R}^3$...  this is why I did not like the overloaded notation.}
Letting $\pp := \bb{M} \dot{\qq}$ denote the (overloaded) momentum, $N=2M$, and defining $\xx := \begin{pmatrix}\qq & \pp \end{pmatrix}^\intercal \in \mathbb{R}^{N}$, \eqref{eq:lin_elasticity_discrete} can be written as the following canonical Hamiltonian system: 
\begin{equation} \label{eq:sm_first_order_sys}
    \dot{\xx} = \JJ \nabla H(\xx) = \left( \begin{array}{cc}
    \bb{0} & \bb{I} \\
    -\bb{I} & \bb{0} 
    \end{array}\right) \left( \begin{array}{cc}
    \bb{K} & \bb{0} \\
    \bb{0} & \bb{M}^{-1} 
    \end{array}\right) \bb{x},
\end{equation}
where $H$ is a quadratic discrete Hamiltonian of the form 
\begin{equation}
    H(\bb{x}) = \frac{1}{2} \left(\qq^\intercal \bb{K}\qq + \pp^\intercal \bb{M}^{-1}\pp \right).
\end{equation}

% \IKTcomment{I added a different symbol for the discretized $\qq$ to avoid confusion.  Anthony, it is your call whether you want to keep this or to ``overload" $\qq$.}

The test case considered presently is a classical solid mechanics benchmark involving a vibrating rectangular cantilever plate of size 0.2 $\times$ 0.2 $\times$ 0.03 meters, so that $\Omega = (0, 0.2) \times (0,0.2) \times (0, 0.03) \in \mathbb{R}^3$.  Let $\bb{s}^\intercal:=(s_1, s_2, s_3)^\intercal \in \mathbb{R}^3$ denote the coordinate (position) vector.  Here,  the left side of the plate is clamped, meaning that a homogeneous Dirichlet boundary condition $\qq = \bb{0}$ is imposed on $\Gamma_l:=\{s_2, s_3 \in \bar{\Omega}: s_1 = 0\}$.  Homogeneous Neumann boundary conditions are prescribed on the remaining boundaries of $\Omega$, indicating that these boundaries are free surfaces.  The problem is initialized by prescribing an initial velocity of 100 m/s in the $s_3$-direction on the right boundary of the domain, $\Gamma_r:=\{s_2,s_3 \in \bar{\Omega}:s_1 = 0.2\}$:
\begin{equation} \label{eq:init_velo}
    \dot{\qq}(\bb{s}, 0) = \left( \begin{array}{c}
    0 \\ 0 \\ 100
    \end{array}\right), \hspace{0.5cm} \text{for } \bb{s} \in \Gamma_r.
\end{equation}
A one-dimensional cartoon illustrating the problem setup is shown in Figure \ref{fig:cartoon_3d_elast}.  The initial velocity perturbation \eqref{eq:init_velo} will cause the plate to vibrate and undergo a flapping motion, as shown in Figure \ref{fig:3d_elasticity_fom}.  As the 
plate vibrates, waves will form and propagate in all three coordinate directions within the plate.  Assuming the plate is made of steel, the material parameters\footnote{It is straightforward to calculate the Lam\'{e} coefficients appearing in \eqref{eq:sigma} from the Young's modulus $E$ and the Poisson ratio $\nu$ using the formulas $\lambda = \frac{E\nu}{(1+\nu)(1-2\nu)}$ and $\mu = \frac{E}{2(1+\nu)}$.} are as follows: $E = 200$ GPa (Young's modulus), $\nu = 0.25$ (Poisson's ratio), and $\rho = 7800$ kg/m$^3$ (density).

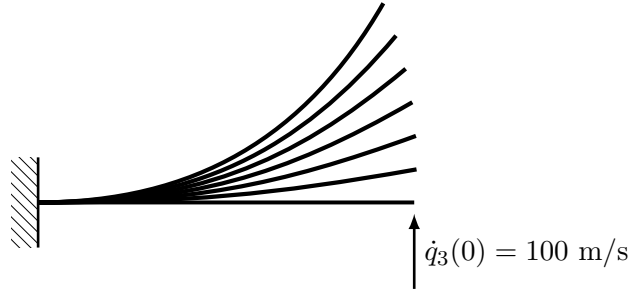
\begin{figure}[htbp!]
        \begin{center}
  \begin{tikzpicture}
    %the points
    \point{begin}{0}{0};
    \point{middle}{2.5}{0};
    \point{end}{5}{0};
   \point{end2}{5}{-1}
    %the beam
    \beam{2}{begin}{end};
    %the support
    \support{3}{begin}[-90];
    %the load
 %   \load{1}{middle}[90];
    \load{1}{end}[-90];
    %the inscription of the load
   % \notation{1}{middle}{$F_1$};
    \notation{1}{end2}{$\dot{q}_{3}(0)=100$ m/s};
    %the deflection curves
% without correction
%    \foreach [evaluate={\in=180-\b*2}] \b in {5,10,...,30}
%      \draw[red,-, ultra thick] (begin) to[out=0,in=\in] (-\b:5);
% quater circle with full radius
%    \draw[red] (begin) -- (end) arc (0:-90:5) -- cycle;
% polar coordinates with some correction of the radius to account for bend
    \foreach [evaluate={\in=180+\b*2}] \b in {5,10,...,30}{
      \draw[-, ultra thick] (begin) to[out=0,in=\in] (\b:5+\b*0.01);
% quater circles with sortened radius
%      \draw[red] (begin) -- (5-\b*0.01,0) arc (0:-90:5-\b*0.01) -- cycle;
    }
  \end{tikzpicture} 
  \caption{One-dimensional cartoon illustrating 3D linear elastic cantilever plate problem setup.}
  \label{fig:cartoon_3d_elast}
  \end{center}
  \end{figure} 

\begin{figure}[htb]
    \centering
    \includegraphics[width=\textwidth]{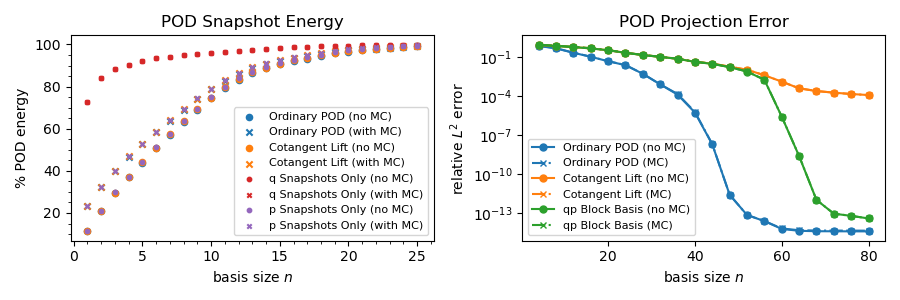}
    \caption{POD snapshot energies (left) and  projection errors (right) corresponding to the bases used in the 3D cantilever plate example.  ``MC'' indicates that mean-centering of the snapshots was performed.}
    \label{fig:SMpodEnergy}
\end{figure}

To build the full order model from which our non-intrusive OpInf ROMs are constructed, we utilize the open-source\footnote{Albany-LCM is available on github at the following URL: \url{https://github.com/sandialabs/LCM}.} Albany-LCM multi-physics code base \cite{albany, Mota:2017, Mota:2022}\footnote{For details on how to reproduce the results in this subsection, the reader is referred to Section \ref{sec:repro}.} and discretize the domain $\Omega$ with a uniform mesh of $20 \times 20 \times 3$ hexahedral elements.  To generate snapshots, the FOM system \eqref{eq:lin_elasticity_discrete} is advanced forward from time $t=0$ to time $t = 2 \times 10^{-2}$ s using a symplectic implicit Newmark time-stepping scheme with parameters $\beta = 0.25$ and $\gamma = 0.5$, and time-step $\Delta t = 1.0 \times 10^{-4}$ s.  Plots of the $s_3$ component of the displacement 
%scaled by a factor of 10
are shown at several different times in Figure \ref{fig:3d_elasticity_fom}.  
The resulting 201 snapshots, each of length 5292, are used to build POD bases of varying sizes, from 4 to 100 POD modes.  Figure~\ref{fig:SMpodEnergy} shows the snapshot energies and reconstruction errors of these bases as a function of basis modes.  Once the POD bases are constructed, several intrusive and non-intrusive ROMs are created and evaluated as discussed earlier in this manuscript.  All ROMs are evaluated in the time-predictive regime, by integrating the governing system forward in time until $t = 0.1$ s ($5 \times$ longer than the training time).  

\begin{figure}[htb]
    \centering
    \includegraphics[width=\textwidth]{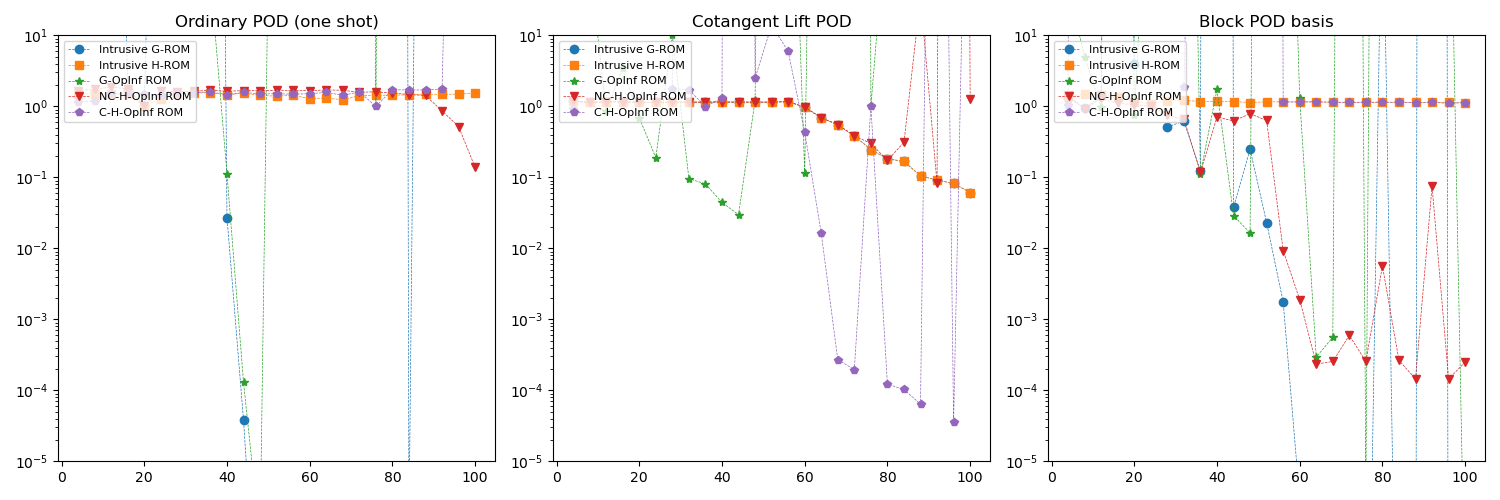}
    \caption{Relative state errors as a function of basis modes for the ROMs in the 3D cantilever plate example (predictive case $T=0.1$). ``MC'' indicates the use of a mean-centered reconstruction.}
    \label{fig:SMerrorsPred}
\end{figure}

\begin{landscape}
\begin{table}[htbp]
% \centering
\caption{Table corresponding to the plots in Figure~\ref{fig:SMerrorsPred}, corresponding to the 3D cantilever plate example and showing the ROM errors as a function of basis size. Dashes indicate lack of convergence.}
\label{tab:SMerrors}\rev{
\begin{adjustbox}{max width=1.3\textwidth}
\begin{tabular}{|ll|lllllllllllll|}
\hline
\multicolumn{1}{|l|}{\textbf{POD Basis}} & \textbf{ROM Type} & \multicolumn{13}{c|}{\textbf{basis size $n$}} \\ \hline
\multicolumn{2}{|l|}{} & 4 & 12 & 20 & 28 & 36 & 44 &  52 & 60 & 68 & 76 & 84 & 92 & 100 \\ \hline
\multicolumn{1}{|l|}{\multirow{5}{*}{Ordinary POD}} & Intrusive G-ROM & \num{5.67e+02} & \num{2.30e+10} & \num{1.00e+00} & \num{2.66e+10} & \num{1.43e+24} & \num{3.79e-05} & - & - & \num{1.72e+23} & \num{5.56e+23} & \num{1.25e-06} & \num{8.71e+02} & \num{6.56e+04} \\
\multicolumn{1}{|l|}{} & Intrusive H-ROM & \num{1.50e+00} & \num{1.50e+00} & \num{1.00e+00} & \num{1.51e+00} & \num{1.56e+00} & \num{1.53e+00} & \num{1.39e+00} & \num{1.28e+00} & \num{1.23e+00} & \num{1.43e+00} & \num{1.45e+00} & \num{1.47e+00} & \num{1.53e+00} \\
\multicolumn{1}{|l|}{} & G-OpInf ROM & \num{8.58e+19} & \num{1.51e+22} & - & \num{2.75e+05} & \num{6.55e+01} & \num{1.29e-04} & \num{2.46e+10} & \num{4.13e+34} & - & \num{1.00e+00} & \num{3.78e+146} & \num{1.05e+44} & \num{6.77e+55} \\
\multicolumn{1}{|l|}{} & NC-H-OpInf ROM & \num{1.66e+00} & \num{1.87e+00} & \num{1.00e+00} & \num{1.60e+00} & \num{1.68e+00} & \num{1.66e+00} & \num{1.69e+00} & \num{1.66e+00} & \num{1.68e+00} & \num{1.60e+00} & \num{1.47e+00} & \num{8.60e-01} & \num{1.38e-01} \\
\multicolumn{1}{|l|}{} & C-H-OpInf ROM & \num{1.16e+00} & \num{1.48e+00} & \num{1.51e+00} & \num{1.49e+00} & \num{1.59e+00} & \num{1.57e+00} & \num{1.50e+00} & \num{1.51e+00} & \num{1.43e+00} & \num{1.00e+00} & \num{1.70e+00} & \num{1.74e+00} & \num{5.09e+03} \\ \hline
\multicolumn{1}{|l|}{\multirow{5}{*}{Cotangent Lift}} & Intrusive G-ROM & \num{1.18e+00} & \num{1.16e+00} & \num{1.14e+00} & \num{1.14e+00} & \num{1.15e+00} & \num{1.15e+00} & \num{1.15e+00} & \num{9.75e-01} & \num{5.51e-01} & \num{2.46e-01} & \num{1.68e-01} & \num{9.14e-02} & \num{6.10e-02} \\
\multicolumn{1}{|l|}{} & Intrusive H-ROM & \num{1.18e+00} & \num{1.16e+00} & \num{1.14e+00} & \num{1.14e+00} & \num{1.15e+00} & \num{1.15e+00} & \num{1.15e+00} & \num{9.75e-01} & \num{5.51e-01} & \num{2.46e-01} & \num{1.68e-01} & \num{9.14e-02} & \num{6.10e-02} \\
\multicolumn{1}{|l|}{} & G-OpInf ROM & \num{4.83e+01} & \num{8.50e-01} & \num{6.88e-01} & \num{1.00e+01} & \num{8.07e-02} & \num{2.97e-02} & \num{1.90e+21} & \num{1.15e-01} & \num{6.08e+09} & \num{1.00e+00} & \num{7.34e+35} & \num{1.90e+60} & \num{5.65e+121} \\
\multicolumn{1}{|l|}{} & NC-H-OpInf ROM & \num{1.18e+00} & \num{1.16e+00} & \num{1.14e+00} & \num{1.14e+00} & \num{1.15e+00} & \num{1.15e+00} & \num{1.15e+00} & \num{9.72e-01} & \num{5.50e-01} & \num{3.01e-01} & \num{3.12e-01} & \num{8.21e-02} & \num{1.25e+00} \\
\multicolumn{1}{|l|}{} & C-H-OpInf ROM & \num{9.17e-01} & \num{4.01e+90} & \num{1.22e+05} & \num{1.77e+00} & \num{9.91e-01} & \num{3.15e+67} & \num{1.39e+01} & \num{4.34e-01} & \num{2.70e-4} & \num{1.00e+00} & \num{1.03e-4} & \num{1.48e+14} & \num{5.62e+06} \\
 \hline
\multicolumn{1}{|l|}{\multirow{5}{*}{Block $(q,p)$}} & Intrusive G-ROM & \num{1.32e+00} & \num{1.39e+00} & \num{3.95e+00} & \num{5.12e-01} & \num{1.22e-01} & \num{3.79e-02} & \num{2.30e-02} & \num{2.92e-06} & \num{3.03e-11} & \num{2.94e-11} & \num{2.91e-11} & \num{1.29e+127} & \num{1.75e-11} \\
\multicolumn{1}{|l|}{} & Intrusive H-ROM & \num{1.38e+00} & \num{1.46e+00} & \num{1.16e+00} & \num{1.19e+00} & \num{1.17e+00} & \num{1.17e+00} & \num{1.16e+00} & \num{1.15e+00} & \num{1.14e+00} & \num{1.14e+00} & \num{1.14e+00} & \num{1.13e+00} & \num{1.12e+00} \\
\multicolumn{1}{|l|}{} & G-OpInf ROM & \num{4.74e+01} & \num{9.65e-01} & \num{7.61e-01} & \num{5.46e+80} & \num{1.12e-01} & \num{2.89e-02} & \num{1.93e+28} & \num{1.31e+00} & \num{5.59e-04} & \num{4.05e-06} & \num{4.10e+60} & \num{2.45e+43} & \num{5.95e-08} \\
\multicolumn{1}{|l|}{} & NC-H-OpInf ROM & \num{1.43e+00} & \num{1.29e+00} & \num{1.06e+00} & \num{7.49e-01} & \num{1.17e-01} & \num{6.23e-01} & \num{6.32e-01} & \num{1.90e-03} & \num{2.55e-04} & \num{2.55e-04} & \num{2.63e-04} & \num{7.50e-02} & \num{2.50e-04} \\
\multicolumn{1}{|l|}{} & C-H-OpInf ROM & \num{1.05e+00} & \num{1.68e+69} & \num{4.29e+36} & - & \num{1.55e+10} & \num{1.44e+05} & \num{5.04e+76} & \num{1.15e+00} & \num{1.14e+00} & \num{1.14e+00} & \num{1.14e+00} & \num{1.14e+00} & \num{1.13e+00} \\ \hline
\end{tabular}
\end{adjustbox}}
% \end{table}
% \end{minipage}}
\end{table}

\end{landscape}
\restoregeometry

The results of this experiment are displayed in Figure~\ref{fig:SMerrorsPred} and Table~\ref{tab:SMerrors}.  Clearly, the ROMs are highly sensitive to the basis construction technique as well as the number of modes used.  While the intrusive G-ROM and G-OpInf ROMs constructed with a block $(q,p)$ basis yield the lowest minimum errors, they are highly volatile, exhibiting unpredictable behavior as basis modes are added.  \rev{Conversely, the the NC-H-OpInf ROM constructed with a block $(q,p)$ basis and the C-H-OpInf ROM constructed with a cotangent lift basis exhibit some attempts at convergence, although still with significant oscillations.}
%Conversely, the NC-H-OpInf ROM constructed with a block $(q,p)$ basis exhibits more even convergence behavior, \rev{despite some oscillation}, and the same is true of the C-H-OpInf ROM constructed with a cotangent lift basis.
It is interesting to note that the intrusive H-ROM represents a different extreme \rev{with all choices of bases}: it is perfectly stable with the addition of modes, but exhibits $\mathcal{O}(1)$ errors unless a cotangent lift basis is used.  It is further remarkable that the ROM errors in the reproductive case $T=0.02$ (not pictured) are slightly lower (within one order of magnitude), but their stability properties remain unchanged.

\rev{\begin{remark}
    While not pictured here, we have observed that intrusive ROMs based directly on the second-order Euler-Lagrange equations \eqref{eq:lin_elasticity_discrete} do not suffer from the same degree of instability with respect to the addition of basis modes as those based on the first-order Hamiltonian formulation \eqref{eq:sm_first_order_sys}.  This could be due to the fact that Galerkin projection of Lagrangian systems onto a reduced basis automatically respects energy conservation, which is not true in the Hamiltonian case, where an additional corrective projection is needed.  
\end{remark}}

For another visualization, Figure \ref{fig:disp_mags_and_errs} shows plots of the displacement magnitude at the final time $t = 0.1$ s for the FOM (a) and various OpInf ROMs (b)--(d).  Here, we showcase ``best-case scenarios" for each ROM: (b) a G-OpInf ROM with \rev{100} POD modes calculated via the $(q,p)$-block basis approach, (c) a C-H-OpInf ROM with \rev{96} POD modes calculated via the cotangent lift basis, and (d) an NC-H-OpInf ROM with \rev{96} POD modes calculated via the $(q,p)$-block basis approach.  The reader can observe that each ROM is capable of producing solutions which are visually indistinguishable from the FOM solution (see subplots (b)--(d)), although their error distribution patterns are quite different (see subplots (e)--(g)).  We emphasize that, while the G-OpInf ROM is the most accurate, it is also by far the most sensitive \rev{to the size of the reduced basis} (see e.g. Figure~\ref{fig:SMerrorsPred}): \rev{there is no visible trend in terms of the basis size, in contrast with the NC-H-OpInf and C-H-OpInf ROMs which are still volatile but roughly decreasing.  Additionally, it is likely that the results seen here could be improved somewhat by regularizing the OpInf problem in some way; since the choice of regularization technique is a non-obvious matter which is currently under active investigation (e.g., \cite{mcquarrie2021data}), this is left for future work.}

\begin{figure}[h!tb]
 \centering
   \begin{tabular}{ccc}
   & \includegraphics[width=0.3\textwidth]{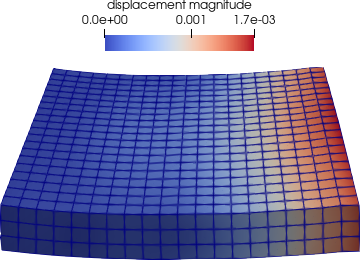} \hspace{0.5cm} \vspace{0.5cm}  \\
     & (a) FOM \vspace{0.5cm} \\
 \includegraphics[width=0.3\textwidth]{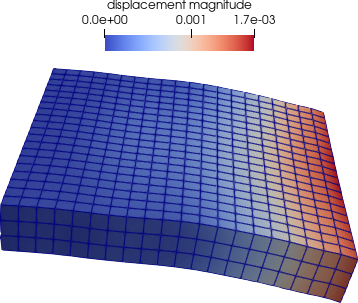} \hspace{0.5cm} & 
  \includegraphics[width=0.3\textwidth]{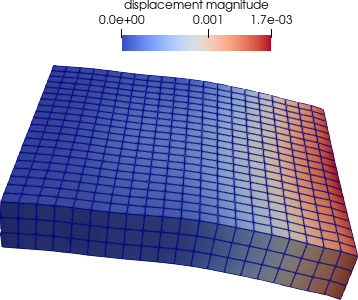} \hspace{0.5cm} 
  &  \includegraphics[width=0.3\textwidth]{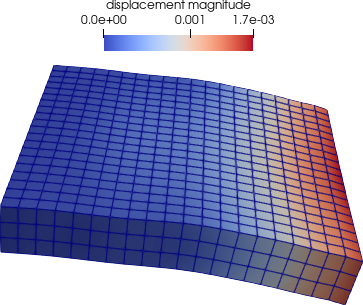} \vspace{0.5cm}\\
    (b) G-OpInf ROM \vspace{0.2cm} & (c) C-H-OpInf ROM \vspace{0.2cm}& (d) NC-H-OpInf ROM \vspace{0.2cm} \\
   \includegraphics[width=0.3\textwidth]{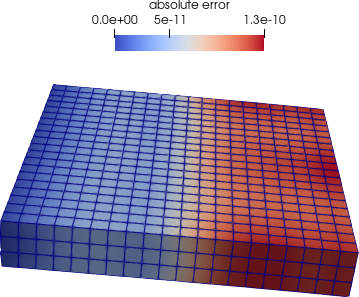} \hspace{0.5cm} & 
  \includegraphics[width=0.3\textwidth]{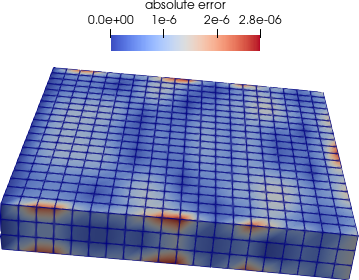} \hspace{0.5cm} 
  &  \includegraphics[width=0.3\textwidth]{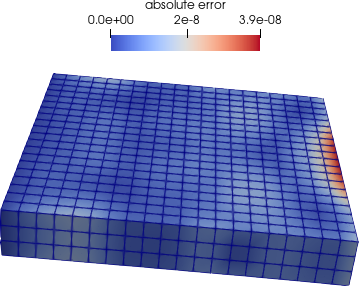} \vspace{0.5cm}  \\
   (e) G-OpInf ROM \vspace{0.1cm} & (f)  C-H-OpInf ROM\vspace{0.1cm}& (g) NC-H-OpInf ROM\vspace{0.1cm} \\
   \\
  \end{tabular}
  \caption{``Best case'' plots of the displacement magnitude at the final time $t = 0.1$ s for the FOM (a) and various OpInf ROMs (b)--(d) for the 3D linear elastic cantilever plate problem.  Subplots (e)--(g) show the spatial distribution of the absolute errors in the displacement magnitude for the various ROMs evaluated, again at the final time $t=0.1$ s.}
  \label{fig:disp_mags_and_errs}
\end{figure}

% Table~\ref{tab:SMreg} and Figure~\ref{fig:SMregularization} studies the effect of the Tikhonov regularization parameter $\eta$ in the OpInf methods presently considered.  While its influence seems to be basis-dependent in all cases, it is clear that less regularization is better for the C-H-OpInf and NC-H-OpInf ROMs (although some is often necessary to avoid a singular left-hand side).  Conversely, the performance of G-OpInf appears to obey this same trend only for the ordinary POD basis, instead displaying a high sensitivity to the order of $\eta$ whenever a cotangent lift or block $(q,p)$ basis is used (this sensitivity is investigated more thoroughly in \cite{mcquarrie2021data}).  This suggests a greater degree of stability for the Hamiltonian structure-preserving OpInf methods developed here, corroborating what was observed in the error plots from before.
To test conservation, Figure~\ref{fig:SMconserved} plots the errors in the value of the Hamiltonian.  Again, it can be seen that the C-H-OpInf and NC-H-OpInf preserve this quantity much better than the G-OpInf ROM, even in cases where the G-OpInf ROM is more accurate.  Unsurprisingly, the conservation properties of the intrusive H-ROM are still superior in all cases, although it is remarkable that this does not always translate to better accuracy in the ROM solution. \rev{This could be due to the fact that the H-ROMs require an additional projection step onto the column space of $\uu$, limiting their accuracy in order to gain exact property preservation.}
% \IKTcomment{Should we cite Shane's work on ``optimal" regularization \cite{mcquarrie2021data}?}

% \begin{figure}[htb]
%     \centering
%     \includegraphics[width=\textwidth]{figs/PlatePlotReg.png}
%     \caption{The effect of the regularization parameter $\eta$ on the OpInf ROMs in the 3D cantilever beam example for three different POD bases. ``MC'' indicates the use of a mean-centered reconstruction. }
%     \label{fig:SMregularization}
% \end{figure}

\begin{figure}[htb]
    \centering
    \includegraphics[width=\textwidth]{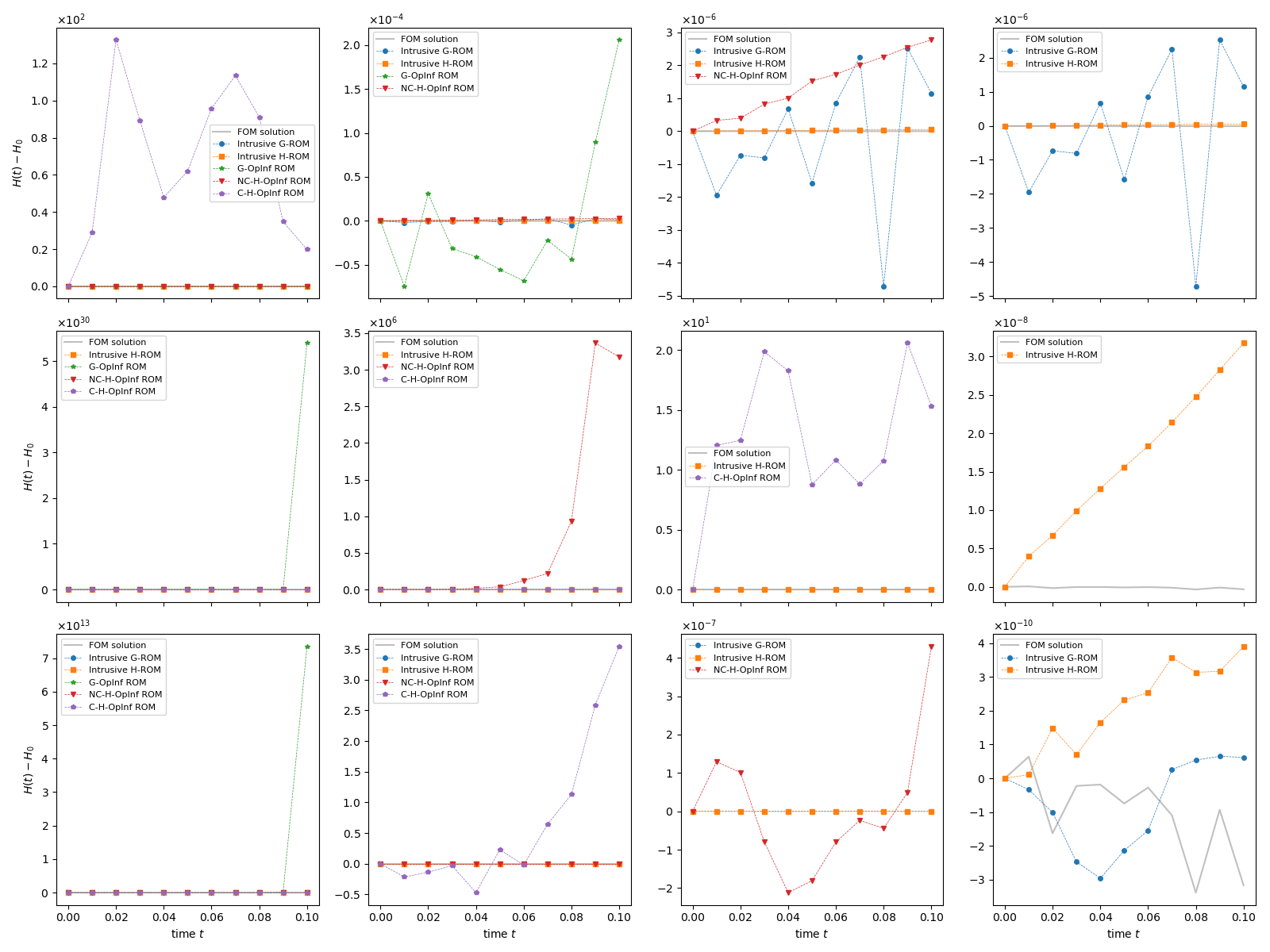}
    \caption{Plots of the error in the Hamiltonian for three different simulations corresponding to the 3D cantilever beam example.  (Top) Ordinary POD basis, $n=48$ modes; (Middle) Cotangent Lift POD basis, \rev{$n=84$} modes; (Bottom) Block $(q,p)$ POD basis, \rev{$n=92$} modes.}
    \label{fig:SMconserved}
\end{figure}

\section{Conclusions and Future Work} \label{sec:concl}
Two gray-box operator inference (OpInf) methods for the nonintrusive model reduction of Hamiltonian dynamical systems have been introduced, and their utility has been demonstrated on several canonical and noncanonical benchmarks.  Being provably convergent to their intrusive counterparts in the limit of infinite data, these OpInf ROMs are shown to recover desirable properties of carefully built intrusive Hamiltonian ROMs such as improved energy conservation without requiring access to FOM simulation code, making them flexible to deploy and leading to improved performance over generic OpInf techniques in reproductive and predictive problems.  Moreover, the technique introduced here has been shown to strictly generalize previous state-of-the-art work on Hamiltonian OpInf methods, reducing to it when the Hamiltonian system in question is canonical, the basis used is a cotangent lift, and the operator to be inferred is block diagonal.

Despite the improvements made here, there are plenty of avenues for future work in the area of Hamiltonian model reduction.  First, the gray-box requirement that the nonlinear part of the Hamiltonian system is known can be feasibly removed when this nonlinearity is polynomial, making the Hamiltonian OpInf methods described  potentially black-box in this case.  Similarly, it would be interesting to apply this technique to systems which have a known conserved quantity but no known Hamiltonian structure, to see if the NC-H-OpInf ROM which preserves this quantity is more accurate and predictively useful than a generic OpInf ROM.  Additionally, it is clear that all structure-preserving ROM techniques to date, intrusive or OpInf, are quite sensitive to basis size when problems become large with complex dynamics.  It would be useful to have stabilized techniques which produce ROMs with more predictable convergence behavior and which do not destroy the delicate mathematical structure important for long-term behavior of the FOM system.  Finally, it would be interesting to extend the techniques mentioned here to quadratic POD bases as well as more general Lie-Poisson variational problems.

% \IKTcomment{Should we mention extensions to quadratic bases and the OpInf for second-order-in-time Euler-Lagrange equations as logical next steps of this work?  It's fine if you don't want to mention it out of fear of being scooped.}

\section{Reproducibility and software availability} \label{sec:repro}

The numerical results presented in Section \ref{subsec:3d_lin_elast}  were generated by running the Albany-LCM open-source
HPC code, available for download on github at the following URL: \url{https://github.com/sandialabs/LCM}. The Albany-LCM code has a strong dependency on Trilinos, available at: \url{https://github.com/trilinos/Trilinos}.  The following shas for Albany-LCM and Trilinos were used in generating the results herein:\\ {\tt 79d3a68bd6176c80ae10e693cd15b1040f6be10f}  and\\ {\tt 322132d613777d48b85d70ed95c0ff4a07c8aed0}, respectively.  
To ensure transparency and reproducibility, we have made available the Albany-LCM input files needed to
reproduce our results, as well as configure scripts for Albany-LCM and Trilinos. These input files can be downloaded from the following github repository: \url{https://github.com/ikalash/HamiltonianOpInf}.  The handwritten Python files and scripts for reproducing the results in Section~\ref{sec:numerics} can be found in this same repository. 

% \IKTcomment{Anthony TODO: put Python code / Jupyter notebooks into this repo.}

\section{Acknowledgement}\label{sec:acknowl}

Support for this work was received
through the U.S. Department of Energy, Office of Science, Office of Advanced Scientific Computing Research, Mathematical Multifaceted Integrated Capability Centers (MMICCS) program, under Field Work Proposal 22025291 and the Multifaceted Mathematics for Predictive Digital Twins (M2dt) project.  The work of the first author (Anthony Gruber) was additionally supported by the John von Neumann Fellowship at Sandia National Laboratories.  The writing of this manuscript was funded in part by the second author’s (Irina Tezaur's) Presidential Early Career Award for Scientists and Engineers (PECASE). 

Sandia National Laboratories is a multi-mission laboratory managed and operated by National Technology and Engineering Solutions of Sandia, LLC., a wholly owned subsidiary of Honeywell International, Inc., for the U.S. Department of Energy’s National Nuclear Security Administration under contract DE-NA0003525.

The authors wish to thank Alejandro Mota for assisting with the formulation of the 3D elasticity problem described herein, and Max Gunzburger for helpful suggestions regarding the choice of basis and treatment of boundary conditions. 

%The authors would also like to thank Eric Parish
%for providing useful feedback on the first full draft of this work, which resulted in an improved
%manuscript.
%References

% \section*{Acknowledgments}
% This work is partially supported by U.S. Department of Energy Scientific Discovery through Advanced Computing under grants DE-SC0020270 and DE-SC0020418.

\bibliographystyle{siam}
\bibliography{biblio}

\section{Appendix}

\subsection{Appendix A: Kronecker Products and Vectorization}\label{app:kron}

We briefly recall some properties of the Kronecker product which are necessary for the results in the body.  Interested readers can find more details in, e.g., \cite{van2000ubiquitous}.  Let $\bb{A}\in\mathbb{R}^{m\times n}, \bb{B}\in\mathbb{R}^{p\times q}$.  The Kronecker product $\bb{A}\otimes\bb{B}\in\mathbb{R}^{mp\times nq}$ is then the matrix of size $mp\times nq$ whose $i,j$-th block (of size $p\times q$) is given by $\lr{\bb{A}\otimes\bb{B}}^i_j = a^i_j\bb{B}$.  It is straightforward to show that $\otimes$ is the matricization of the usual tensor product when expressed with respect to a lexicographical ordering of the standard bases for $\mathbb{R}^n\otimes\mathbb{R}^q$ and $\mathbb{R}^m\otimes\mathbb{R}^p$, since $\lr{\bb{A}\otimes\bb{B}}\lr{\xx\otimes\bb{y}} = \bb{A}\xx\otimes\bb{By}$ for any $\xx\in\mathbb{R}^n$ and $\bb{y}\in\mathbb{R}^q$.  Moreover, there is a linear vectorization operator ``vec'' which stacks the columns of a matrix into a long vector, i.e. $A_{ij} = \lr{\vect\bb{A}}_{m(j-1)+i}$. Since vectorization is obviously invertible, this allows for the following computationally convenient reformulation of linear systems with matrix unknowns.  

\begin{theorem}[Vec trick]
$\vect\lr{\bb{AXB}}=\lr{\bb{B}^\intercal\otimes\bb{A}}\vect\bb{X}$.
\end{theorem}

\begin{proof}
Let $\bb{a}_i, \bb{x}_i,\bb{b}_i$ denote the $i^{\mathrm{th}}$ column of $\bb{A},\bb{X},\bb{B}$ respectively.  Then, the $i^{\mathrm{th}}$ column of $\bb{AXB}$ is 
\[\lr{\bb{AXB}}_i = \bb{A}\bb{X}\bb{b}_i = \bb{A}b^j_i\bb{x}_j = \lr{b^j_i\bb{A}}\bb{x}_j = \lr{\bb{b}_i^\intercal\otimes \bb{A}}\vect\XX.\]
The conclusion now follows by stacking columns.
\end{proof}

There is also a very concrete (but rather inefficient) way to obtain the transposition matrix $\bb{K}$ satisfying $\vect{\XX}^\intercal = \bb{K}\vect{\XX}$.  While this is true generally for $m\times n$ matrices $\XX$, we state the result for square matrices for ease of notation.

\begin{proposition}
Let $\bb{E}_{ij} = \bb{e}_i\bb{e}_j^\intercal$ denote the $ij$-th basis vector for the matrix space $\mathbb{R}^{n\times n}$.  Then, we have that 
\[\bb{K} = \sum_{ij} \bb{E}_{ji}^\intercal\otimes \bb{E}_{ji},\]
satisfies $\vect{\XX}^\intercal = \bb{K}\vect{\XX}$.
\end{proposition}

\begin{proof}
Given $\bb{X}\in\mathbb{R}^{n\times n}$, it follows by the vec trick that 
\begin{align*}
\bb{K}\vect\XX &= \lr{\sum \bb{E}_{ji}^\intercal\otimes \bb{E}_{ji}}\vect{\XX} = \vect\lr{\sum \bb{E}_{ji}\bb{X}\bb{E}_{ji}} \\
&= \vect\lr{\sum \bb{e}_jx_{ij}\bb{e}_i^\intercal} = \vect\lr{\sum x_{ij}\bb{E}_{ji}} = \vect\XX^\intercal.
\end{align*}
\end{proof}

More practically, the following pseudocode is used to generate a sparse matrix representing $\bb{K}$.

\begin{algorithm}[htb]
\caption{Building the commutation matrix $\bb{K}$}\label{alg:transposition}
\begin{algorithmic}[1]
\Require Integers $m,n>0$.
\Ensure Sparse matrix $\bb{K}\in\mathbb{Z}^{mn\times mn}$ satisfying $\vect{\XX}^\intercal = \bb{K}\vect{\XX}$ for all $\XX\in\mathbb{R}^{m\times n}$.
\State Let $row = \{1,2,...,mn\} \in\mathbb{Z}^{mn}$ be the vector of row indices.
\State Let $row'\in\mathbb{Z}^{m\times n}$ be defined by reshaping $row$ column-wise. 
\State Let $col\in\mathbb{Z}^{mn}$, the list of column indices, be the row-wise flattening of $row'$.\\
\Return Sparse matrix $\bb{K}$ with indices $(row,col)$ and entries $\{1,...,1\}\in\mathbb{Z}^{mn}$.
\end{algorithmic}
\end{algorithm}

\subsection{Appendix B: Proofs of Results}\label{app:proofs}

Here we provide omitted proofs for the results in the body.  Note that Einstein summation is assumed throughout, so that any tensor index appearing both ``up'' and ``down'' in an expression is implicitly summed over its range.

\begin{proof}[Proof of Proposition~\ref{prop:trunc}]
This is a straightforward consequence of the fact that OpInf of size $n$ decouples into $n^2$ scalar minimization problems.  To see this, notice that if $\Dh$ solves the OpInf problem of size $n$ and $1\leq i',j'\leq n'<K\leq n$, $1
\leq k\leq n$, then
\begin{align*}
\lr{\Xh_t\Xh^\intercal}^{i'}_{j'} &= \hat{D}^{i'}_k\lr{\Xh\Xh^\intercal}^k_{j'} = \hat{D}^{i'}_{k'}\lr{\Xh\Xh^\intercal}^{k'}_{j'} + \hat{D}^{i'}_{K}\lr{\Xh\Xh^\intercal}^{K}_{j'} \\
&= \hat{D}^{i'}_{k'}\,\delta^{k'}_{j'}\sigma_{j'}^2 + \hat{D}^{i'}_{K}\,\delta^K_{j'}\sigma^2_{j'} = \hat{D}^{i'}_{j'}\,\sigma_{j'}^2,
\end{align*}
where $\delta$ denotes the Kronecker delta tensor and the first equality of the second line follows from the fact that, for all $1\leq i,j\leq n$,
\[\lr{\Xh\Xh^\intercal}^i_{j} = \IP{\XX^\intercal\bb{u}_i}{\XX^\intercal\bb{u}_{j}} =\IP{\bb{e}_i}{\uu^\intercal\XX\XX^\intercal\uu\bb{e}_{j}} = \IP{\bb{e}_{i}}{\bm{\Sigma}^2\bb{e}_{j}} = \delta^i_{j} \sigma_j^2.\]
Therefore, the minimization problem for each component $\hat{D}^i_j$ has the solution (note the sum on $k$),
\[\argmin_{\hat{D}^i_j\in\mathbb{R}} \nn{\bb{u}_i^\intercal\bb{X}_t - \sum_k\hat{D}^i_k\bb{u}_k^\intercal\bb{X}}^2 = \frac{\bb{u}_i^\intercal\bb{X}_t\bb{X}^\intercal\bb{u}_j}{\sigma_j^2},\]
showing that each entry of $\Dh$ depends only on the indices $i,j$.  Therefore, the solution $\Dh'$ to the OpInf problem of size $n'<n$ can be extracted from $\Dh$ by extracting the top-left $n'\times n'$ submatrix, as desired.
\end{proof}

\begin{proof}[Proof of Proposition~\ref{prop:newtrunc}]
Suppose $\Dh$ is the solution hypothesized \rev{in the statement of the Proposition}.  Then, it follows that
\rev{\begin{align*}
    \Ah^\intercal\Ch\Bh^\intercal\pm\Bh\Ch^\intercal\Ah &= \Ah^\intercal\Ah\Dh\Bh\Bh^\intercal\pm\Bh\Bh^\intercal\Dh^\intercal\Ah^\intercal\Ah \\
    &= \Ah^\intercal\Ah\Dh\Bh\Bh^\intercal + \Bh\Bh^\intercal\Dh\Ah^\intercal\Ah
\end{align*}}
Now, for any $1\leq i,j\leq n$, notice that 
\[\lr{\Ah^\intercal\Ch\Bh^\intercal\pm\Bh\Ch^\intercal\Ah}^i_j = \hat{A}^i_i\bb{u}_i^\intercal\bb{CB}^\intercal\bb{u}_j \pm \bb{u}_i^\intercal\bb{BC}^\intercal\bb{u}_jA^j_j.\]
Therefore, the $ij$-th entry of the left-hand side of the optimality condition for $\Dh$ depends only on the basis vectors $\bb{u}_i,\bb{u}_j$, and we have 
\[\Ah'^\intercal\hat{\bb{C}}'\hat{\bb{B}}'^\intercal \pm \hat{\bb{B}}'\hat{\bb{C}}'^\intercal\Ah' = \lr{\Ah^\intercal\Ch\Bh^\intercal\pm\Bh\Ch^\intercal\Ah}'.\]
Similarly, letting $1\leq i',j',k',l'\leq n'$ and $1\leq k,l\leq n$, it follows that 
\rev{\begin{align*}
    \lr{\Ah^\intercal\Ch\Bh^\intercal\pm\Bh\Ch^\intercal\Ah}^{i'}_{j'} &= \lr{\Ah^\intercal\Ah}^{i'}_k \hat{D}^k_{l}\lr{\Bh\Bh^\intercal}^{l}_{j'} + \lr{\Bh\Bh^\intercal}^{i'}_k \hat{D}^k_l\lr{\Ah^\intercal\Ah}^l_{j'} \\
    &= \lr{\hat{A}^{i'}_{k'}}^2 D^{k'}_{l'}\lr{\hat{\bb{B}}\hat{\bb{B}}^\intercal}^{l'}_{j'} + \lr{\hat{\bb{B}}\hat{\bb{B}}^\intercal}^{i'}_{k'} D^{k'}_{l'}\lr{\hat{A}^{l'}_{j'}}^2,
\end{align*}}
where the second line uses the fact that $\lr{\hat{\bb{A}}^\intercal\hat{\bb{A}}}^{i'}_{K}=\lr{\hat{\bb{B}}\hat{\bb{B}}^\intercal}^L_{j'}=\lr{\hat{\bb{B}}\hat{\bb{B}}^\intercal}^{i'}_K=\lr{\hat{\bb{A}}^\intercal\hat{\bb{A}}}^L_{j'}$ for all $n'<K,L\leq n$.  Putting these computations together, this shows that the truncation $\Dh'$ of $\Dh$ satisfies
\rev{\begin{align*}
   \Ah'^\intercal\hat{\bb{C}}'\hat{\bb{B}}'^\intercal \pm \hat{\bb{B}}'\hat{\bb{C}}'^\intercal\Ah' &= \Ah'^\intercal\Ah'\Dh'\bar{\bb{B}}\bar{\bb{B}}^\intercal + \bar{\bb{B}}\bar{\bb{B}}^\intercal\Dh'\Ah'^\intercal\Ah',
\end{align*}}
showing that $\Dh'$ is the desired minimizer.
\end{proof}

\subsection{Appendix C: Second Hamiltonian Fomulation of KdV}\label{app:kdv}

Another Hamiltonian formulation of the KdV equation is given by the data
\[H(x) = \frac{1}{2}\int_0^l x^2\,ds, \qquad L(x) = \frac{\alpha}{3}\lr{x\partial_s + \partial_s(x\cdot)} + \rho\partial_s + \nu\partial_{sss}, \]
where $\partial_s(x\cdot)y = \partial_s(xy)$.  Choosing $\bb{A}$ to be the central difference discretization of $\partial_s$ (this was $\LL$ in the first formulation) leads to the skew-symmetric discrete operator
\[\LL(\xx) = \frac{\alpha}{3}\lr{\mathrm{Diag}\lr{\xx}\bb{A} + \bb{A}\,\mathrm{Diag}\lr{\xx}} + \rho\bb{A} + \nu\bb{E},\]
where $\bb{E}$ is the pentadiagonal circulant matrix representing the central difference discretization of $\partial_{sss}$, i.e.
\[\bb{E} = \frac{1}{2\lr{\Delta x}^3}\begin{pmatrix} 0 & -2 & 1 & & -1 & 2 \\ 2 & \ddots & \ddots & \ddots & & -1 \\ -1 & \ddots & \ddots & \ddots & \ddots & \\ & \ddots & \ddots & \ddots & \ddots & 1 \\ 1 & & \ddots & \ddots & \ddots & -2 \\ -2 & 1 & & -1 & 2 & 0\end{pmatrix}.\]
This implies the second Hamiltonian formulation of the KdV system,
\[\dot{\xx} = \LL(\xx)\nabla H(\xx) = \frac{\alpha}{3}\lr{\mathrm{Diag}\lr{\xx}\bb{A}\xx + \bb{A}\,\mathrm{Diag}\lr{\xx}\xx} + \rho\bb{A}\xx + \nu\bb{E}\xx,\]
which is integrated via AVF to yield the discrete system,
\[\frac{\xx^{k+1}-\xx^k}{\Delta t} = \LL\lr{\xx^{k+\frac{1}{2}}}\xx^{k+\frac{1}{2}}.\]
This leads to the $\Delta t$-normalized residual and Jacobian functions,
\begin{align*}
R^k\lr{\xx^{k+1}} &= \xx^{k+1} - \xx^k - \Delta t\,\LL\lr{\xx^{k+\frac{1}{2}}}\xx^{k+\frac{1}{2}}, \\
J^k\lr{\xx^{k+1}} &= \bb{I} -\frac{\Delta t}{2}\left[\frac{2\alpha}{3}\lr{\diag\lr{\xx^{k+\frac{1}{2}}}\bb{A}+\bb{A}\,\diag\lr{\xx^{k+\frac{1}{2}}}} + \rho\bb{A} + \nu\bb{E}\right],
\end{align*}
which are solvable with Newton iterations.  

With this, intrusive Galerkin and Hamiltonian ROMs can then be constructed as before.  On the other hand, notice that AVF evaluates $\LL$ at the midpoint of the discrete trajectory, meaning that Galerkin projection and discretization with AVF no longer commute, since $\ut\LL$ is not a Poisson matrix.  However, letting $\xt = \xx_0+\uu\xh$, Galerkin projection after AVF yields
\begin{align*}
    \frac{\xh^{k+1}-\xh^k}{\Delta t} &=\ut\LL\lr{\xt^{k+\frac{1}{2}}}\xt^{k+\frac{1}{2}} \\
    &= \lr{\ut\LL\lr{\xx_0} + \frac{\alpha}{3}\lr{\ut\diag\lr{\uu\xh^{k+\frac{1}{2}}}\bb{A}+\ut\bb{A}\diag\lr{\uu\xh^{k+\frac{1}{2}}}}}\lr{\xx_0+\uu\xh^{k+\frac{1}{2}}} \\
    &:= \ut\LL\lr{\xx_0}\xx_0 + \lr{\bb{T}\lr{\xh^{k+\frac{1}{2}}}\xx_0 + \ut\LL\lr{\xx_0}\uu\xh^{k+\frac{1}{2}}} + \hat{\bb{T}}\lr{\xh^{k+\frac{1}{2}}}\xh^{k+\frac{1}{2}} \\
    &:= \hat{\bb{c}} + \hat{\bb{C}}\xh^{k+\frac{1}{2}} + \hat{\bb{T}}\lr{\xh^{k+\frac{1}{2}}}\xh^{k+\frac{1}{2}},
\end{align*}
where $\bb{T},\hat{\bb{T}}$ are precomputable order 3 tensors given component-wise by $T^a_{jc}=(\alpha/3)U^a_i\lr{U^i_c+U_{jc}}A^i_j$ and $\hat{T}^a_{bc} = (\alpha/3) U^a_i\lr{U^i_c+U_{jc}}A^i_jU^j_b$.  Now, a Hamiltonian ROM can be computed in the same way:
applying AVF before Hamiltonian projection, it follows that 
\begin{align*}
    \frac{\xh^{k+1}-\xh^k}{\Delta t} &=\ut\LL\lr{\xt^{k+\frac{1}{2}}}\uu\ut\xt^{k+\frac{1}{2}} \\
    &= \lr{\ut\LL\lr{\xx_0}\uu + \frac{\alpha}{3}\lr{\ut\diag\lr{\uu\xh^{k+\frac{1}{2}}}\bb{A}\uu+\ut\bb{A}\diag\lr{\uu\xh^{k+\frac{1}{2}}}\uu}}\lr{\ut\xx_0+\xh^{k+\frac{1}{2}}} \\
    &= \ut\LL\lr{\xx_0}\uu\ut\xx_0 + \lr{\hat{\bb{T}}\lr{\xh^{k+\frac{1}{2}}}\ut\xx_0 + \ut\LL\lr{\xx_0}\uu\xh^{k+\frac{1}{2}}} + \hat{\bb{T}}\lr{\xh^{k+\frac{1}{2}}}\xh^{k+\frac{1}{2}} \\
    &:= \hat{\bb{c}} + \hat{\bb{C}}\xh^{k+\frac{1}{2}} + \hat{\bb{T}}\lr{\xh^{k+\frac{1}{2}}}\xh^{k+\frac{1}{2}},
\end{align*}
where the tensor $\hat{T}$ is identical to before.  In either case, these equations are easily solved with Newton iterations, as explained in Section~\ref{subsec:kdv}.  

\begin{remark}
It is interesting to note that $\nabla H(\xx) = \xx$ in this formulation, so that its matrix representation $\bb{A}=\bb{I}$.  This has the effect of equalizing the (non mean-centered) H-ROM and G-ROM, since $\Lh\Ah = \widehat{\LL\bb{A}}$.
\end{remark}

\begin{figure}[htb]
    \centering
    \begin{minipage}{0.5\textwidth}
         \includegraphics[width=\textwidth]{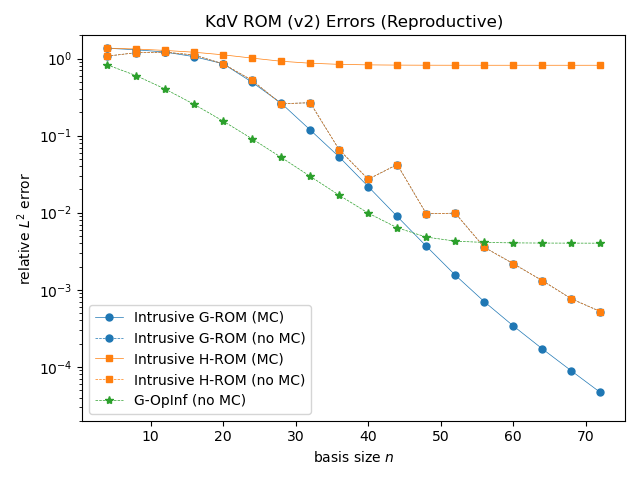}
    \end{minipage}%
    \begin{minipage}{0.5\textwidth}
         \includegraphics[width=\textwidth]{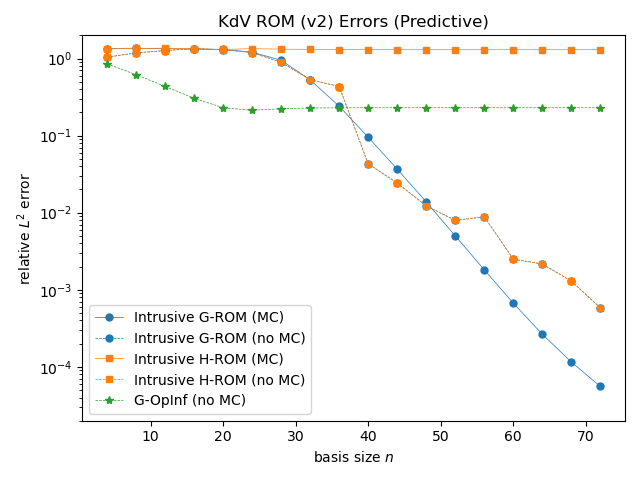}
    \end{minipage}
    \caption{Relative state errors as a function of basis modes for the ROMs in the KdV equation (V2) example.  Left: reproductive case ($T=20$).  Right: predictive case ($T=100$). ``MC'' indicates the use of a mean-centered reconstruction.}
    \label{fig:KdVerrorsV2}
\end{figure}

Figures~\ref{fig:KdVerrorsV2} and \ref{fig:KdVConservedT20v2} show the results of this procedure, alongside a linear G-OpInf ROM for comparison (c.f. Section~\ref{subsec:genopinf}).  The experimental parameters are identical to those in Section~\ref{subsec:kdv}.  It is remarkable that the mean-centered H-ROM does not perform well in this case, despite conserving the first three invariant quantities as well as the mean-centered G-ROM.  Note that the naming convention in Figure~\ref{fig:KdVConservedT20v2} follows that of Figure~\ref{fig:KdVConservedT100}, despite the fact that $P$ is now the Hamiltonian functional.

\begin{figure}[htb]
    \centering
    \includegraphics[width=\textwidth]{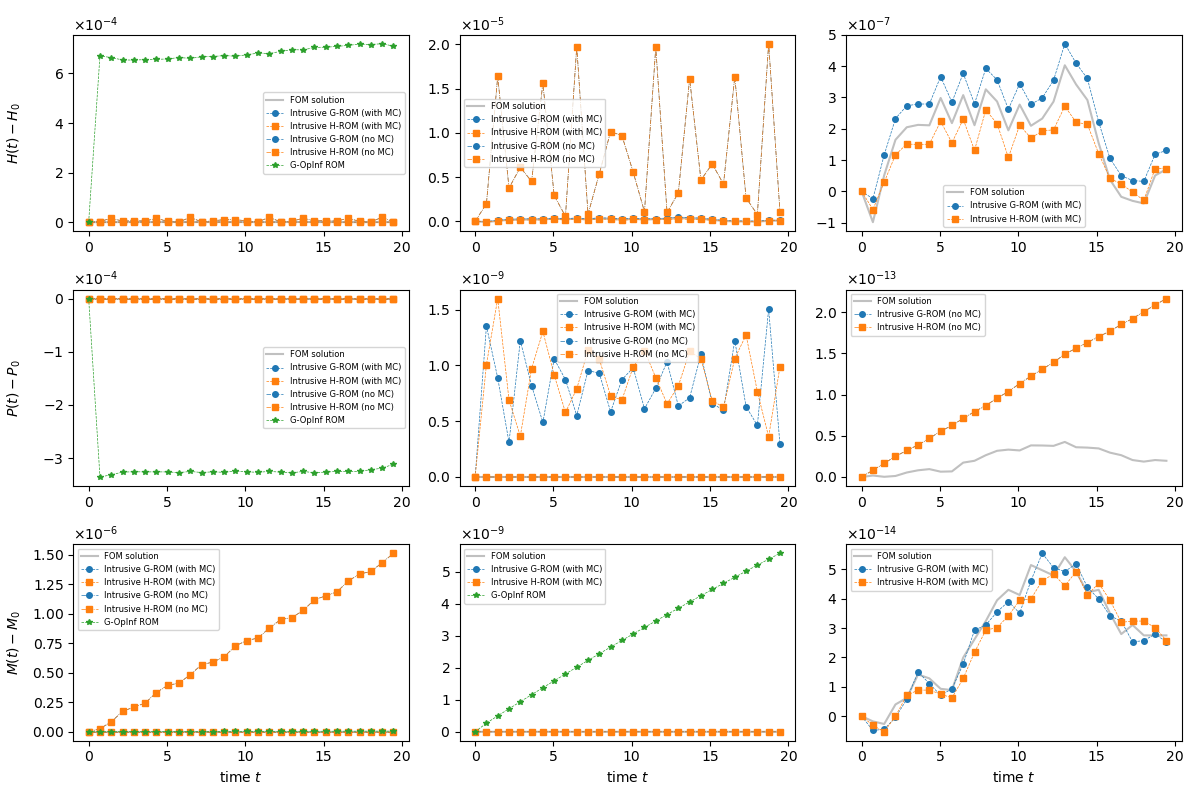}
    \caption{Errors in conserved quantities for the (mean-centered) ROMs in the KdV (v2) equation example in the predictive case ($T=20$) when using a POD basis with $n=72$ modes.}
    \label{fig:KdVConservedT20v2}
\end{figure}

% \begin{align*}
% \dot{\xh} &= \Lh\lr{\xh}\nabla\hat{H}\lr{\xh} = \ut\LL\lr{\xt}\uu\ut\lr{\xx_0+\uu\xh} \\
% &= \ut\left[\frac{\alpha}{3}\lr{\diag\lr{\xx_0}\bb{A}+\bb{A}\,\diag\lr{\xx_0}} + \lr{\rho\bb{A}+\nu\bb{E}}\right]\uu\ut\lr{\xx_0 + \uu\xh} + \Th\lr{\xh}\ut\lr{\xx_0+\uu\xh} \\
% &:= \hat{\bb{c}} + \hat{\bb{C}}\xh + \Th\lr{\xh}\xh,
% \end{align*}
% where $\Th$ is a precomputable degree three tensor with components $T^a_{bc} = (\alpha/3) U^a_i\lr{U^i_c+U_{jc}}A^i_jU^j_b$.  Applying AVF integration yields the fully discrete system 
% \[\frac{\xh^{k+1}-\xh^k}{\Delta t} = \hat{\bb{c}} + \hat{\bb{C}}\xh^{k+\frac{1}{2}} + \frac{1}{3}\lr{\Th\lr{\xh^{k}}\xh^k + \Th\lr{\xh^{k+1}}\xh^{k+1}} + \frac{1}{6}\lr{\Th\lr{\xh^k}\xh^{k+1} + \Th\lr{\xh^{k+1}}\xh^{k}}.\]

% \red{finish this, include other time integration}

\end{document}